\documentclass{article}

\usepackage[table, svgnames, dvipsnames]{xcolor}

\usepackage{microtype}
\usepackage{graphicx}
\usepackage{subfig}
\usepackage{booktabs} 

\usepackage{hyperref}

\usepackage[accepted]{icml2024} 

\usepackage{amsmath}
\usepackage{amssymb}
\usepackage{mathtools}
\usepackage{amsthm}

\usepackage[capitalize,noabbrev]{cleveref}

\theoremstyle{plain}
\newtheorem{theorem}{Theorem}[section]

\newtheorem{lemma}[theorem]{Lemma}

\theoremstyle{definition}
\newtheorem{definition}[theorem]{Definition}
\newtheorem{assumption}[theorem]{Assumption}
\theoremstyle{remark}

\usepackage[textsize=tiny]{todonotes}

\usepackage{amsmath,amsfonts,bm}

\def\eqref#1{equation~\ref{#1}}

\def\1{\bm{1}}

\def\vmu{{\bm{\mu}}}

\def\vs{{\bm{s}}}

\def\vv{{\bm{v}}}

\def\vx{{\bm{x}}}

\DeclareMathAlphabet{\mathsfit}{\encodingdefault}{\sfdefault}{m}{sl}
\SetMathAlphabet{\mathsfit}{bold}{\encodingdefault}{\sfdefault}{bx}{n}

\def\sR{{\mathbb{R}}}

\usepackage[utf8]{inputenc}
\usepackage{multirow}
\usepackage{tablefootnote}
\usepackage{makecell}
\usepackage{enumitem}
\usepackage{soul}
\usepackage[export]{adjustbox}
\usepackage{bbm}
\usepackage{bbding}
\usepackage[toc,page,header]{appendix}
\usepackage{minitoc}

\renewcommand{\algorithmicrequire}{\textbf{Input:}}  
\renewcommand{\algorithmicensure}{\textbf{Output:}} 

\newcommand\norm[1]{\left\lVert#1\right\rVert}

\icmltitlerunning{ProtoGate: Prototype-based Neural Networks with Global-to-local Feature Selection for Tabular Biomedical Data}

\begin{document}
\doparttoc
\faketableofcontents

\twocolumn[
\icmltitle{ProtoGate: Prototype-based Neural Networks with \\Global-to-local Feature Selection for Tabular Biomedical Data}



\icmlsetsymbol{equal}{*}

\begin{icmlauthorlist}
\icmlauthor{Xiangjian Jiang}{cam}
\icmlauthor{Andrei Margeloiu}{cam}
\icmlauthor{Nikola Simidjievski}{cam,onco}
\icmlauthor{Mateja Jamnik}{cam}
\end{icmlauthorlist}

\icmlaffiliation{cam}{Department of Computer Science and Technology, University of Cambridge, UK}
\icmlaffiliation{onco}{Department of Oncology, University of Cambridge, UK}

\icmlcorrespondingauthor{Xiangjian Jiang}{xj265@cam.ac.uk}

\icmlkeywords{Machine Learning, ICML, Interpretable Machine Learning, Deep Tabular Learning, Biomedical Data}

\vskip 0.3in
]

\printAffiliationsAndNotice{}  

\begin{abstract}
\looseness-1
Tabular biomedical data poses challenges in machine learning because it is often high-dimensional and typically low-sample-size (HDLSS). Previous research has attempted to address these challenges via local feature selection, but existing approaches often fail to achieve optimal performance due to their limitation in identifying globally important features and their susceptibility to the co-adaptation problem. In this paper, we propose ProtoGate, a prototype-based neural model for feature selection on HDLSS data. ProtoGate first selects instance-wise features via adaptively balancing global and local feature selection. Furthermore, ProtoGate employs a non-parametric prototype-based prediction mechanism to tackle the co-adaptation problem, ensuring the feature selection results and predictions are consistent with underlying data clusters. We conduct comprehensive experiments to evaluate the performance and interpretability of ProtoGate on synthetic and real-world datasets. The results show that ProtoGate generally outperforms state-of-the-art methods in prediction accuracy by a clear margin while providing high-fidelity feature selection and explainable predictions. Code is available at~\url{https://github.com/SilenceX12138/ProtoGate}.

\end{abstract}

\section{Introduction}
In biomedical research, tabular data is frequently collected \cite{baxevanis2020bioinformatics, lesk2019introduction} for a wide range of applications such as detecting marker genes~\cite{hsu2003unsupervised} and performing survival analysis~\cite{fan2022survival}. Clinical trials, whilst collecting large amounts of high-dimensional data using modern high-throughput sequencing technologies, often consider a small number of patients due to practical reasons~\cite{levin2022transfer}. The resulting tabular datasets are thus often high-dimensional and typically low-sample-size (HDLSS). Moreover, given the inherent heterogeneity of biomedical data, important features often vary from sample to sample -- even in the same dataset~\cite{yang2022locally, yoon2018invase}. Such scenarios have proven challenging for current machine learning approaches, including deep tabular models~\citep{shwartz2022tabular, yang2022locally, margeloiu2022weight}.

\begin{figure}[!t]
    \centering
    \includegraphics[width=0.97\columnwidth]{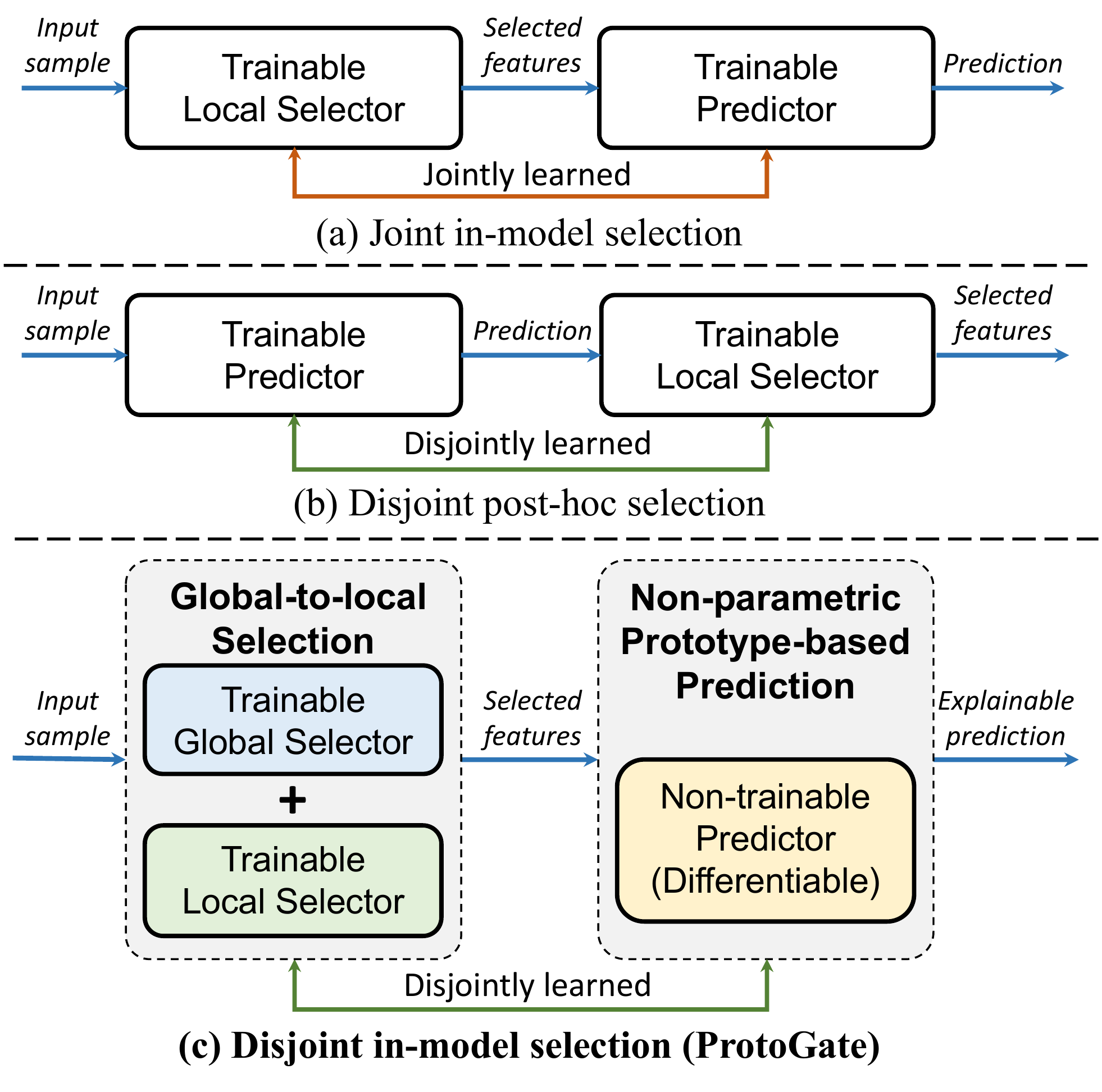}
    \vspace{-6mm}
    \caption{\textbf{An overview of the proposed model.} ProtoGate introduces a novel disjoint in-model selection method. It balances global and local feature selection, and makes explainable prototypical predictions. In contrast to~(a), ProtoGate integrates a trainable feature selector with a non-trainable predictor (i.e., no trainable parameters in the predictor), which allows for disjointly learned feature selector and predictor, thus mitigating the co-adaptation problem. In contrast to~(b), ProtoGate makes predictions with the selected features, preserving their in-model explainability.}  
    \label{fig:overview}
    \vspace{-4mm}
\end{figure}

\begin{table*}[!t]
    \centering
    \caption{\textbf{Model design comparison between ProtoGate and prior local feature selection methods.} ProtoGate has different design rationales to the benchmark methods: a novel learning paradigm with global-to-local feature selector and prototype-based predictor.}
    \label{tab:related_work_mini}
    \resizebox{0.99\textwidth}{!}{
        \begin{tabular}{lc|c|c|c|c|c}
        \midrule
        \multirow{2}{*}[-3mm]{Methods} & \multirow{2}{*}[-3mm]{Learning paradigms} & \multicolumn{2}{c|}{Feature Selection} & \multicolumn{2}{c|}{Prediction}   & \multirow{2}{*}{\parbox{0.1\linewidth}{\vspace{3mm} \thead{Co-adaptation\\avoidance}}} \\
        
        \cmidrule{3-6}
        &                                             & Global   Selector             & Local Selector                                 & Explainability                & \thead{Clustering\\assumption}         \\
        
        \midrule
        TabNet~\cite{arik2021tabnet}                                       & \multirow{5}{*}{\parbox{0.2\linewidth}{\vspace{2mm} Joint in-model selection}}                                                     & {\color[HTML]{E6341C} \XSolidBrush}  & {\color[HTML]{3A7B21} \CheckmarkBold} & {\color[HTML]{E6341C} \XSolidBrush}   & {\color[HTML]{E6341C} \XSolidBrush}  & {\color[HTML]{E6341C} \XSolidBrush} \\
        
        L2X~\cite{chen2018learning}                                          &                                                                                                             & {\color[HTML]{E6341C} \XSolidBrush}   &              {\color[HTML]{3A7B21} \CheckmarkBold}                                  & {\color[HTML]{E6341C} \XSolidBrush}   & {\color[HTML]{E6341C} \XSolidBrush}  & {\color[HTML]{E6341C} \XSolidBrush} \\
        
        INVASE~\cite{yoon2018invase}                                       &                                                                                              & {\color[HTML]{E6341C} \XSolidBrush}   &                {\color[HTML]{3A7B21} \CheckmarkBold}                                & {\color[HTML]{E6341C} \XSolidBrush}   & {\color[HTML]{E6341C} \XSolidBrush}  & {\color[HTML]{E6341C} \XSolidBrush} \\
        
        LSPIN~\cite{yang2022locally}                                        &                                                                                                                    & {\color[HTML]{E6341C} \XSolidBrush}   &                   {\color[HTML]{3A7B21} \CheckmarkBold}                             & {\color[HTML]{E6341C} \XSolidBrush}   & {\color[HTML]{E6341C} \XSolidBrush}  & {\color[HTML]{E6341C} \XSolidBrush} \\
        
        LLSPIN~\cite{yang2022locally}                                       &                                                                                                                  & {\color[HTML]{E6341C} \XSolidBrush}   &                    {\color[HTML]{3A7B21} \CheckmarkBold}                            & {\color[HTML]{3A7B21} \CheckmarkBold} & {\color[HTML]{E6341C} \XSolidBrush}  & {\color[HTML]{E6341C} \XSolidBrush} \\
        
        \midrule
        REAL-X~\cite{jethani2021have}                                       & Disjoint post-hoc selection                                                                     & {\color[HTML]{E6341C} \XSolidBrush}   &                 {\color[HTML]{3A7B21} \CheckmarkBold}                               & {\color[HTML]{E6341C} \XSolidBrush}   & {\color[HTML]{E6341C} \XSolidBrush} & {\color[HTML]{3A7B21} \CheckmarkBold}  \\
        
        \midrule
        \rowcolor{Gainsboro!60}
        \textbf{ProtoGate (Ours)}                  & Disjoint in-model selection           & {\color[HTML]{3A7B21} \CheckmarkBold} &                      {\color[HTML]{3A7B21} \CheckmarkBold}                          & {\color[HTML]{3A7B21} \CheckmarkBold} & {\color[HTML]{3A7B21} \CheckmarkBold} & {\color[HTML]{3A7B21} \CheckmarkBold} \\
        \bottomrule
        \end{tabular}
    }
    \vspace{-4mm}
\end{table*}

Prior work~\cite{remeseiro2019review, arik2021tabnet, chen2018learning, yoon2018invase, yang2022locally, yoshikawa2022neural} has attempted to address such challenges with local feature selection: rather than selecting a general subset of features across all samples, local methods select specific subsets of features for each sample and these subsets may vary from sample to sample.

\vspace{-0.25mm}
However, existing methods still have three main limitations, which are summarised in \cref{tab:related_work_mini}:
\textit{(i)~Neglecting globally important features.}
In many real-world HDLSS tasks, even simple models -- such as MLP and Lasso -- can outperform the advanced methods (see \cref{tab:acc} for more details). We hypothesise that this is because the mainstream methods (\cref{fig:overview}(a) and \cref{fig:overview}(b)) only perform local feature selection without explicitly capturing globally important features, provided the competitive global methods support their existence. 
\textit{(ii)~Low-fidelity feature selection.}
The selected features by these methods can be uninformative even when prediction accuracy is high. For instance, L2X~\cite{chen2018learning} achieves 96\% accuracy in digit classification on MNIST by using only one pixel as input~\cite{jethani2021have}. Although REAL-X mitigates this problem via disjointly learning a predictor to prevent it from overfitting on the selected features (\cref{fig:overview}(b)), it sacrifices the in-model explainability of selected features. 
\textit{(iii)~Insufficient explainability and inappropriate inductive bias.}
In the biomedical domain, the clustering assumption, which states similar samples likely belong to the same class~\cite{chapelle2006semi}, has been shown effective~\cite{kolodner1992introduction, li2018deep, bichindaritz2006case, bichindaritz2008case, lu2021ace}. However, prior work may inappropriately incorporate such inductive bias. For instance, LSPIN~\cite{yang2022locally} selects similar features for similar samples, but it cannot guarantee that samples of the same class cluster together after selection (see \cref{sec:exp_cls} for more details). In this paper, we aim to overcome these challenges by introducing a novel global-to-local method that selects accurate and high-fidelity features with explainable predictions.

We propose ProtoGate (\cref{fig:overview}(c)), a simple yet effective feature selection method for tabular biomedical data. Our approach is distinguished by three core concepts. Firstly, ProtoGate selects features in a global-to-local manner to adaptively balance global and local feature selection. Thus ProtoGate can effectively capture both global and local information across samples. Secondly, ProtoGate elegantly mitigates the co-adaptation problem through a unique learning paradigm for local feature selection: ProtoGate selects features with a non-parametric predictor. The predictor is non-trainable, and it can only evaluate whether the selected features lead to accurate predictions, rather than jointly learn to overfit the selected features for high classification accuracy. With the resistance to co-adaptation problem, ProtoGate can safely make predictions with the selected features. Therefore, ProtoGate can improve the fidelity of selected features while still preserving the in-model explainability of selected features. Thirdly, ProtoGate encodes the clustering assumption into feature selection via prototype-based prediction, which promotes feature selection that clusters samples of the same class together. This design not only confers ProtoGate an inductive bias aligned with the clustering assumption, but also the capability to make explainable prototypical predictions. In addition, we propose a hybrid sorting strategy to expedite the prototype-based prediction.

\vspace{1.8mm}
\looseness-1
\textbf{Contributions.}
More broadly, our contributions are: 
(i)~We present ProtoGate, a novel feature selection method for tabular biomedical data (\cref{sec:method}). The proposed method is characterised by three novelties: a global-to-local feature selection approach, a unique learning paradigm to tackle co-adaptation, and a prototype-based prediction mechanism that aligns with the clustering assumption. 
(ii)~We show that ProtoGate generally outperforms 14 benchmark methods in seven biomedical classification tasks with fewer selected features and shorter runtimes (\cref{sec:exp_cls}). 
(iii)~We further show that ProtoGate can generalise well in non-biomedical domains (\cref{sec:exp_cls}).
(iv)~We conduct comprehensive ablation studies to show that each component of ProtoGate makes complementary contributions to the model performance (\cref{sec:exp_ablation}).
(v)~We quantitatively and qualitatively demonstrate that ProtoGate can provide robust interpretability for the selected features and predictions (\cref{sec:exp_syn}).
(vi) We provide an open-source implementation of ProtoGate to facilitate future research and applications.

\section{Related Work}
\label{sec:related_work}

\textbf{Local Feature Selection.}
Recent work attends to the heterogeneity across samples by selecting instance-wise features for making predictions~\cite{chen2018learning, arik2021tabnet, jethani2021have, yoon2018invase, yang2022locally, yoshikawa2022neural}. Its main focus is on designing differentiable optimisation to generate sparse masks for feature selection. L2X uses mutual information with Concrete distribution~\cite{chen2018learning}, and INVASE models each feature's mask value with independent Bernoulli distributions~\cite{yoon2018invase}.
\citet{yamada2020feature} propose a continuous relaxation of discrete random variables, which was further extended to the exact formulation by LSPIN~\cite{yang2022locally}. 
In this work, we employ this exact formulation for computing differentiable $\ell_0$-regularisation
\cref{tab:related_work_mini} and \cref{appendix:model_cmp} show that there have been few mechanisms to explicitly promote globally important features in local feature selection.  Instead, ProtoGate proposes to address the gap by adaptively balancing global and local feature selection, leading to flexible feature selection behaviours across different datasets.

\vspace{1.5mm}
\looseness-1
\textbf{Co-adaptation Problem.}
In local feature selection, co-adaptation refers to the situation where the model achieves high prediction accuracy with low-fidelity features~\cite{jethani2021have, adebayo2018sanity, hooker2019benchmark, samek2016evaluating}. \citet{jethani2021have} prove that co-adaptation stems from jointly learning a trainable selector and a trainable predictor, where the predictor learns to overfit on poorly selected features for high accuracy. Post-hoc explanation methods (\cref{fig:overview}(b)) address the problem by training a local selector to explain a pre-trained predictor. Instead, we propose to mitigate the problem via a non-trainable predictor. In ProtoGate, the predictor has no trainable parameters and cannot adapt to the feature selector for high accuracy, eliminating the possibility of co-adaptation while preserving the in-model explainability.

\vspace{1.5mm}
\looseness-1
\textbf{Prototype-based Machine Learning.}
Prototype-based models~\cite{biehl2016prototype} in machine learning are closely related to metric learning~\cite{goldberger2004neighbourhood} and case-based reasoning~\cite{kolodner1992introduction}. They are built upon the clustering assumption and aim to represent data through prototypical exemplars (e.g., KNN~\cite{fix1985discriminatory}) or a set of prototypical centroids (e.g., $k$-means~\cite{ball1965isodata}) that capture the fundamental characteristics of the data. These core ideas parallel similar concepts from cognitive psychology and neurosciences, and it is a pervasive behaviour in everyday human problem-solving~\cite{kolodner1992introduction, li2018deep, bichindaritz2006case, bichindaritz2008case, lu2021ace}. Thus, in ProtoGate we leverage the efficacy of prototype-based prediction for feature selection, leading to more generally applicable inductive biases and improved explainability in predictions.

\begin{figure*}[!t]
    \centering
    \includegraphics[width=2\columnwidth]{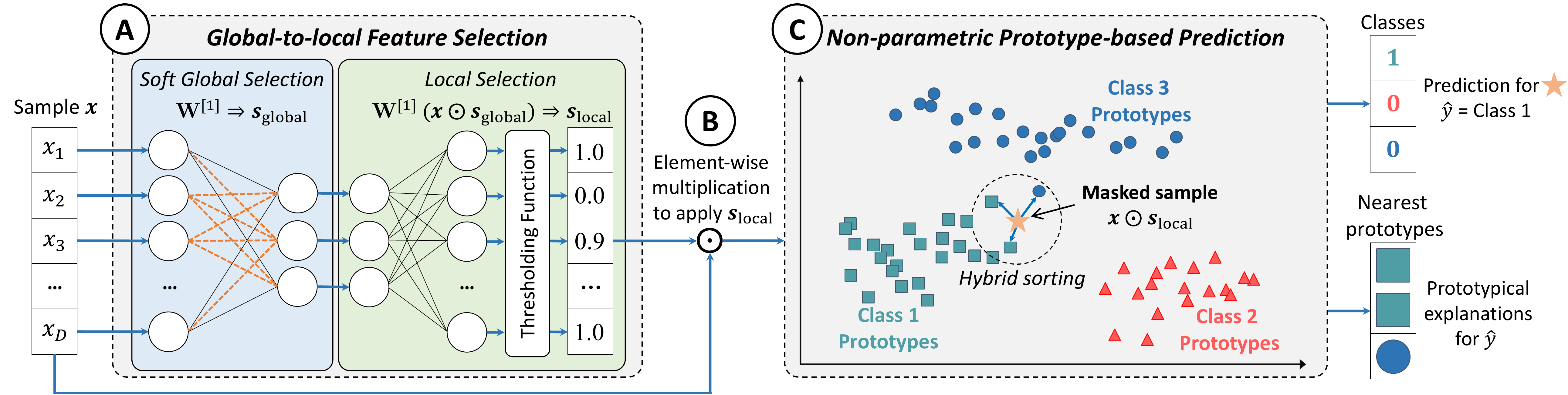}
    \vspace{-1mm}
    \caption{\textbf{The architecture of ProtoGate.} \textbf{(A)} Given a sample $\vx \in \sR^{D}$, the global-to-local feature selection performs soft global feature selection in the first layer of the gating network. The {\color[HTML]{ED7D31}{orange}} dashed lines denote sparsified weights (i.e., reduced to zero) in $\mathbf{W}^{[1]}$ under $\ell_1$-regularisation. The neural network then computes the instance-wise mask $\{s_d\}_{d=1}^{D} \in [0,1]^{D}$ with a thresholding function for local feature selection. \textbf{(B)} The local mask $\vs_{\text{local}}$ is applied to the sample for local feature selection by element-wise multiplication. \textbf{(C)} The non-parametric prototype-based prediction further classifies $\vx \odot \vs_{\text{local}}$ by retrieving the $K$ nearest prototypes in base $\mathcal{B}$ via hybrid sorting. The majority class is used as the predicted label $\hat{y}$, and the exemplars (i.e., the nearest prototypes) provide prototypical explanations.}
    \label{fig:architecture}
    \vspace{-4mm}
\end{figure*}

\section{Method}
\label{sec:method}

\cref{fig:architecture} illustrates the architecture of ProtoGate. We first describe our problem setup (\cref{sec:problem_setup}). Then we present the core components of ProtoGate: a global-to-local feature selection approach (\cref{sec:g2l_fs}), and a non-parametric prototype-based prediction mechanism (\cref{sec:proto_pred}). Furthermore, we define a disjoint objective function to optimise ProtoGate end-to-end (\cref{sec:objective_func}). We present the pseudocode for training and inference in \cref{appendix:algo}. 

\subsection{Problem Setup}
\label{sec:problem_setup}
We consider the classification task on tabular biomedical data with $\mathcal{Y}$ classes. Let $X \coloneqq \left[\vx^{(1)}, \dots, \vx^{(N)}\right]^\top \in \sR^{N \times D}$ be the data matrix consisting of $N$ samples $\vx^{(i)} \in \sR^D$ with $D$ features, and let $Y  \coloneqq \left[y^{(1)}, \dots, y^{(N)}\right]^\top \in \sR^{N \times 1}$ be the corresponding labels. We denote $x^{(i)}_{d}$ as the $d$-th feature of the $i$-th sample. To simplify the notation, we assume all samples in $X$ are used for training. We mainly consider the HDLSS datasets, where the number of features is much greater than the number of samples (i.e., $D	\gg N$). 

\subsection{Global-to-local Feature Selection}
\label{sec:g2l_fs}
Our proposed global-to-local feature selection is implemented via a gating network $S_\mathbf{W}\!\!:\! \sR^{D} \!\rightarrow \![0, 1]^{D}$ that sequen\-tially performs soft global selection and local selection (\cref{fig:architecture}A). It takes as input a sample $\vx^{(i)}$ and generates a mask (also referred to as ``gate'') $\vs_{\text{local}}^{(i)} \coloneqq [s_1^{(i)}, \dots, s_D^{(i)}] \in [0,1]^D$ for local feature selection. The $d$-th feature is selected if and only if the mask value is positive ($s_d^{(i)}\! >\! 0$) (\cref{fig:architecture}B).

\vspace{1.5mm}
\looseness-1
\textbf{Soft Global Selection.}
ProtoGate first captures the globally important features via applying $\ell_1$-regularisation on $\mathbf{W}^{[1]}$, the weights of the first layer in the gating network. The regularisation promotes sparsity in $\mathbf{W}^{[1]}$ via reducing some weights to zero. When all weights connected to the same input neuron are reduced to zero, the corresponding input feature is dropped (e.g., $x_2$ and $x_3$ in \cref{fig:architecture}A). Mathematically, the sparsity in $\mathbf{W}^{[1]}$ transforms the first layer's computation as $\mathbf{W}^{[1]}\vx^{(i)} = \mathbf{W}^{[1]}(\vx^{(i)} \odot \vs_\text{global}^{(i)})$, where $\odot$ represents element-wise multiplication and $\vs_\text{global}^{(i)} \in \{0, 1\}^{D}$ is a binary mask, signifying dropped features as zeros. The first layer is shared across samples, and thus the mask is consistent for all samples (i.e., $\vs_\text{global}=\vs_\text{global}^{(i)}=\vs_\text{global}^{(j)}$). Furthermore, the global selection by $\mathbf{W}^{[1]}$ is ``soft'' because ProtoGate uses $\vs_\text{global}$ to explicitly emphasise the existence of globally important features, serving as a foundation rather than replacing subsequent local feature selection. Specifically, the initially dropped features can be recovered in the following local selection process (e.g., $x_3$ in \cref{fig:architecture}A). We further illustrate the interplay between soft global and local selection with a real-world example in \cref{appendix:fs_illustration}.

\textbf{Local Selection.}
After the initial selection of globally important features, ProtoGate further refines $\vs_\text{global}$ to generate instance-wise masks $\vs_\text{local}^{(i)}$ for local feature selection. This stage involves processing the embedded output $\mathbf{W}^{[1]}(\vx^{(i)} \odot \vs_\text{global})$ through subsequent network layers, yielding the vector $\vmu^{(i)}$. Given the range of $\vmu^{(i)}$ is not necessarily $[0,1]^D$, we threshold it to compute the instance-wise masks and then apply $\ell_0$-regularisation to promote sparse masks. However, the non-differentiability of $||\vs_\text{local}^{(i)}||_0$ poses a challenge for efficient optimisation via backpropagation. Building on prior work~\cite{louizos2018learning, rolfe2016discrete, yamada2020feature, yang2022locally}, we re-formalise the local mask as a vector of random variables, defined by $\vs_\text{local}^{(i)} = \max(0, \min(1, \vmu^{(i)} + \bm{\epsilon}^{(i)}))$, where $\bm{\epsilon}^{(i)}$ is Gaussian noise sampled from $\mathcal{N}(\bm{0}, \bm{\Sigma})$. This formulation allows for the differentiable estimation of $\ell_0$-regularisation through its expectation (i.e., $||\vs_\text{local}^{(i)}||_0 \propto \mathbb{E} [ ||\vs_\text{local}^{(i)}||_0 ]$). In addition, the thresholded mask values are continuous (i.e., $\vs_{\text{local}}^{(i)} \in [0,1]^D$), which are considered to be more expressive than binary masks~\cite{gros2021softseg}.

\vspace{-1mm}
The integration of these two phases is encapsulated in ProtoGate's weighted sparsity regularisation for feature selection
\begin{equation}
\label{eq:reg_sparsity}
\begin{aligned}
    \!\!\!\!\!\mathcal{R}_{\text{select}} &\coloneqq \mathcal{R_{\text{global}}}(\vs_{\text{global}}) + \mathcal{R_{\text{local}}}(\vs_{\text{local}}^{(i)}) \\
    &\propto \lambda_{\text{global}}\norm{\mathbf{W}^{[1]}}_1 + \lambda_{\text{local}}\sum_{d=1}^{D}\int_{\hat{\mu}_d^{(i)}}^{\infty}\exp(-\frac{t^2}{2})dt
\end{aligned}
\end{equation}
where $(\lambda_{\text{global}}, \lambda_{\text{local}})$ is a pair of hyperparameters to balance the regularisation strength between global and local feature selection, and $\hat{\mu}_d^{(i)}$ is the standardised output of $S_{\mathbf{W}}$. Comprehensive theoretical justifications are in \cref{appendix:reg}.

\looseness-1
\textbf{Complementary regularisations.}
ProtoGate employs different regularisations for its two-stage selection process to address distinct objectives.
Firstly, the soft global selection focuses on efficiently identifying a globally important lower-dimension feature set (i.e. $\textbf{s}_\text{global}$). Therefore, ProtoGate promotes sparsity in $\textbf{s}_\text{global}$ via the faster $\ell_1$-regularisation because it does not require the computation of complex surrogates like $\ell_0$-regularisation.
Secondly, the local selection aims to identify locally important features for accurate predictions and high interpretability, which prioritises selecting fewer features (i.e., $\textbf{s}_\text{local}^{(i)}$ contains more ``exact zeros''). Note that small mask values (e.g., 0.1) still signify a selected feature. Therefore, we choose $\ell_0$-regularisation for local selection because it generally leads to more exact zeros~\cite{louizos2018learning}.
Note that, while $\ell_1$-regularisation can also promote exact zeros with greater regularisation strength, we find that it can detrimentally impact model performance, aligning with findings from previous study~\cite{louizos2018learning}.

\vspace{-0.4mm}
\subsection{Non-parametric Prototype-based Prediction}
\label{sec:proto_pred}
The non-parametric prototype-based prediction (\cref{fig:architecture}C) is implemented via a differentiable $k$-Nearest Neighbours (KNN), denoted as \(F_\theta\!:\! \mathbb{R}^D\! \rightarrow\! \mathcal{Y}\), which takes as input the masked sample $\vx^{(i)} \odot \vs_\text{local}^{(i)}$ and predicts its label $\hat{y}^{(i)}$. Note that $\theta$ here refers to hyperparameters for a non-parametric model, rather than trainable parameters as $\textbf{W}$ of $S_\mathbf{W}$.

\vspace{-1mm}
During the training period, ProtoGate first constructs a prototype base $\mathcal{B}$ with the masked training samples. Recall that we assume all $x^{(i)}$ are used as training samples. Thus ProtoGate retains the pairs of the masked sample and its label $p^{(i)} \coloneqq \left(\vx^{(i)} \odot \vs_\text{local}^{(i)}, y^{(i)}\right)$ as prototypes in the base. With the acquired prototypes, ProtoGate can classify a randomly picked query sample $\vx_{\text{query}} \in X$ by retrieving similar prototypes in the base $\mathcal{B}$. Specifically, ProtoGate sorts the prototypes by their similarities (i.e., Euclidean distance) to the masked query sample and makes predictions according to the majority class of the $K$ nearest prototypes. 
Technically, during training, we use the masked samples from the same batch as prototypes; during inference, we use the masked samples from the whole training set as prototypes.

\vspace{-1mm}
\looseness-1
\textbf{Hybrid Sorting.}
The standard sorting operation (e.g., QuickSort~\cite{hoare1961algorithm}) is non-differentiable and the model cannot be optimised via backpropagation. In contrast, the differentiable sorting operation (e.g., NeuralSort~\cite{grover2018stochastic}) allows backpropagation, but the high time complexity can be a bottleneck. With these in mind, we take a step further and propose \textit{hybrid} sorting. We integrate two kinds of sorting operations to amalgamate their strengths: ProtoGate first trains with NeuralSort for backpropagating gradients from the sorted prototypes, and then substitutes NeuralSort for QuickSort during inference. During the training phase, given a random training sample as query sample $\vx_{\text{query}} \coloneqq \vx^{(i)}\in X$, we sort the base $\mathcal{B}$ via computing a row-stochastic matrix $\widehat{\mathbf{P}} \in \sR^{N \times N}$, and $\widehat{\mathbf{P}}[n,m]$ denotes the probability that the $m$-th prototype in $\mathcal{B}$ is the $n$-th nearest to query sample. Given $n \sim \mathcal{U}\{1,K\}$, we can improve the accuracy via $\max\ \mathbb{E}_{\widehat{\mathbf{P}}[n,:]}[\mathbbm{1}(y^{(m)}=y_{\text{query}})]$ where $\widehat{\mathbf{P}}[n,:]$ is the $n$-th row of $\widehat{\mathbf{P}}$, and $\mathbbm{1}(\cdot)$ is the indicator function (i.e., maximising the number of prototypes that have the same label as $\vx_{\text{query}}$). Therefore, the prediction loss per query sample is given by:

\vspace{-3mm}
\begin{equation}
\begin{aligned}
\!\mathcal{L}_{\text{pred}}
    \coloneqq K - \sum_{n=1}^K \mathbb{E}_{\widehat{\mathbf{P}}[n,:]} \left[\mathbbm{1}(y^{(m)} = y_{\text{query}})\right]
    \label{eq:dknn_loss}
\end{aligned}
\end{equation}
\vspace{-2mm}
\looseness-1

where $y_{\text{query}}=y^{(i)}$. Note that we set $m \neq i$ to avoid using $x_{\text{query}}$ itself as a prototype. \cref{eq:dknn_loss} estimates the number of prototypes that have different labels to $\vx_{\text{query}}$ among the $K$ nearest prototypes (we provide theoretical analysis in \cref{appendix:pred_loss}). By optimising the gating network, ProtoGate learns to generate local masks that cluster samples of the same class together. During inference, the prototypes are fixed and query samples are from unseen test data. We find that the hybrid strategy decreases the inference time by almost half while preserving the identical predictive performance as only using differentiable sorting (\cref{appendix:ablation_sort}).

\textbf{Differentiability vs. Non-trainability.}
The two concepts are not contradictory. For hybrid sorting, the differentiability allows backpropagating gradients from sorted prototypes. On the other hand, its non-parametric nature leads to a non-trainable predictor $F_\theta$ that consistently evaluates whether the masked samples align with the clustering assumption. And $F_\theta$ cannot overfit on $\vs_\text{local}^{(i)}$ from $S_{\mathbf{W}}$ through training, leading to disjointly learned $S_{\mathbf{W}}$ and $F_\theta$. Therefore, ProtoGate can be resistant to the co-adaptation problem.

\vspace{-2mm}
\subsection{Disjoint Training Loss}
\label{sec:objective_func}
We further incorporate the core components into the estimator of ProtoGate's training loss:

\vspace{-5mm}
\begin{equation}
\begin{aligned}
    \mathcal{L_{\text{total}}} (\textbf{W}) \coloneqq \mathbb{E}_{(X,Y)} \left[ \mathcal{L}_{\text{pred}} + \mathcal{R}_{\text{select}} \right].
    \label{eq:training_loss}
\end{aligned}
\end{equation}
We can see that the gating network $S_{\mathbf{W}}$ is the only trainable component, denoting disjointly learned feature selector and predictor (i.e., ProtoGate's training objective only needs to optimise $\mathbf{W}$). Moreover, the training loss is fully differentiable, allowing ProtoGate to be optimised end-to-end in a single stage with standard gradient-based approaches.

\begin{table*}[!t]
    \centering
    \caption{\textbf{Classification accuracy (\%) on seven HDLSS real-world tabular datasets.} We report the mean $\pm$ std balanced accuracy and average accuracy rank across datasets. A higher rank implies higher accuracy. Note that ``$-$'' denotes failed convergence, and the rank is computed with the balanced accuracy of other methods. We highlight the {\color[HTML]{008080} \textbf{First}}, {\color[HTML]{7030A0} \textbf{Second}} and {\color[HTML]{C65911} \textbf{Third}} ranking accuracy for each dataset. ProtoGate consistently ranks Top-3 across datasets and achieves the best overall performance.}
    \label{tab:acc}
    \resizebox{\textwidth}{!}{
    \begin{tabular}{m{0.1cm}lrrrrrrr|r}
    \toprule
    \multicolumn{2}{l} {Methods} & colon & lung & meta-dr & meta-pam & prostate & tcga-2y & toxicity & \textbf{Rank} \\
    \midrule
    
    \multirow{4}{*}{\rotatebox{90}{\begin{tabular}{l} \textit{Baseline} \end{tabular}}}
    & Ridge & 77.50$_{\pm\text{9.43}}$ & 92.77$_{\pm\text{6.45}}$ & {\color[HTML]{C65911} \textbf{59.19}}$_{\pm\text{9.78}}$ & 92.16$_{\pm\text{7.69}}$ & 87.22$_{\pm\text{8.56}}$ & 57.73$_{\pm\text{5.83}}$ & {\color[HTML]{C65911} \textbf{91.88}}$_{\pm\text{5.49}}$ & 6.71$_{\pm\text{2.75}}$ \\
    
    & SVM & 70.75$_{\pm\text{13.93}}$ & 72.77$_{\pm\text{8.33}}$ & 50.64$_{\pm\text{2.24}}$ & 61.46$_{\pm\text{7.65}}$ & 85.75$_{\pm\text{6.63}}$ & 52.63$_{\pm\text{4.02}}$ & 66.75$_{\pm\text{7.86}}$ & 15.14$_{\pm\text{1.68}}$ \\
    
    & KNN & 71.65$_{\pm\text{12.03}}$ & 91.06$_{\pm\text{5.41}}$ & 54.64$_{\pm\text{7.92}}$ & 82.79$_{\pm\text{7.95}}$ & 78.78$_{\pm\text{9.20}}$ & 58.83$_{\pm\text{6.71}}$ & 83.86$_{\pm\text{7.07}}$ & 10.86$_{\pm\text{3.98}}$ \\
    
    & MLP & 80.00$_{\pm\text{8.70}}$ & {\color[HTML]{008080} \textbf{96.47}}$_{\pm\text{2.69}}$ & 57.89$_{\pm\text{9.63}}$ & {\color[HTML]{7030A0} \textbf{96.17}}$_{\pm\text{2.59}}$ & 87.22$_{\pm\text{7.41}}$ & {\color[HTML]{7030A0} \textbf{60.18}}$_{\pm\text{7.24}}$ & {\color[HTML]{008080} \textbf{94.48}}$_{\pm\text{4.28}}$ & {\color[HTML]{7030A0} \textbf{3.86}}$_{\pm\text{3.34}}$ \\
    \midrule
    
    \multirow{6}{*}{\rotatebox{90}{\begin{tabular}{l} \textit{Global} \end{tabular}}} 
    & Lasso & 79.40$_{\pm\text{10.18}}$ & {\color[HTML]{7030A0} \textbf{94.47}}$_{\pm\text{4.39}}$ & 58.58$_{\pm\text{9.17}}$ & 95.15$_{\pm\text{2.83}}$ & {\color[HTML]{7030A0} \textbf{91.18}}$_{\pm\text{6.39}}$ & 56.99$_{\pm\text{5.21}}$ & 91.86$_{\pm\text{6.03}}$ & {\color[HTML]{C65911} \textbf{5.00}}$_{\pm\text{2.89}}$ \\
    
    & RF & {\color[HTML]{C65911} \textbf{80.05}}$_{\pm\text{10.37}}$ & 91.73$_{\pm\text{6.61}}$ & 51.48$_{\pm\text{3.41}}$ & 88.73$_{\pm\text{6.24}}$ & 90.38$_{\pm\text{7.31}}$ & 58.70$_{\pm\text{6.84}}$ & 79.78$_{\pm\text{7.10}}$ & 8.71$_{\pm\text{5.02}}$ \\

    & XGBoost & 72.60$_{\pm\text{12.59}}$ & 86.61$_{\pm\text{8.72}}$ & 58.09$_{\pm\text{8.65}}$ & 92.65$_{\pm\text{5.40}}$ & 82.55$_{\pm\text{10.22}}$ & 55.91$_{\pm\text{6.97}}$ & 70.13$_{\pm\text{7.85}}$ & 11.71$_{\pm\text{2.69}}$ \\
    
    & CatBoost & 72.65$_{\pm\text{10.12}}$ & 91.57$_{\pm\text{5.74}}$ & 57.08$_{\pm\text{5.74}}$ & 95.93$_{\pm\text{4.39}}$ & 90.24$_{\pm\text{6.87}}$ & 55.09$_{\pm\text{7.52}}$ & 81.95$_{\pm\text{7.47}}$ & 9.00$_{\pm\text{3.56}}$ \\
    
    & LightGBM & 76.60$_{\pm\text{11.67}}$ & 93.42$_{\pm\text{5.91}}$ & 58.23$_{\pm\text{8.56}}$ & 94.98$_{\pm\text{5.19}}$ & {\color[HTML]{008080} \textbf{91.38}}$_{\pm\text{5.71}}$ & 57.09$_{\pm\text{7.87}}$ & 81.98$_{\pm\text{6.25}}$ & 6.43$_{\pm\text{3.21}}$ \\
    
    & STG & 79.55$_{\pm\text{10.53}}$ & 93.30$_{\pm\text{6.28}}$ & 58.15$_{\pm\text{8.67}}$ & 76.13$_{\pm\text{8.19}}$ & 89.38$_{\pm\text{5.85}}$ & 57.04$_{\pm\text{5.76}}$ & 87.95$_{\pm\text{5.01}}$ & 7.43$_{\pm\text{3.60}}$ \\
    \midrule
    
    \multirow{7}{*}{\rotatebox{90}{\begin{tabular}{l} \textit{Local} \end{tabular}}} 
    & TabNet & 56.75$_{\pm\text{15.20}}$ & 80.14$_{\pm\text{12.23}}$ & 53.87$_{\pm\text{7.31}}$ & 82.66$_{\pm\text{11.56}}$ & 66.55$_{\pm\text{15.33}}$ & 52.16$_{\pm\text{8.20}}$ & 41.68$_{\pm\text{9.03}}$ & 15.29$_{\pm\text{1.60}}$ \\
    
    & L2X & 57.60$_{\pm\text{13.48}}$ & 50.02$_{\pm\text{14.26}}$ & 52.54$_{\pm\text{8.30}}$ & 62.64$_{\pm\text{13.75}}$ & 61.78$_{\pm\text{13.69}}$ & 52.30$_{\pm\text{6.29}}$ & 31.72$_{\pm\text{9.11}}$ & 14.29$_{\pm\text{0.70}}$ \\
    
    & INVASE & $-$ & 91.22$_{\pm\text{6.16}}$ & $-$ & 91.70$_{\pm\text{6.84}}$ & $-$ & 55.98$_{\pm\text{6.45}}$ & 79.94$_{\pm\text{6.60}}$ & 11.29$_{\pm\text{0.95}}$ \\

    & REAL-X & 76.75$_{\pm\text{12.21}}$ & 93.27$_{\pm\text{4.32}}$ & {\color[HTML]{7030A0} \textbf{60.01}}$_{\pm\text{7.12}}$ & 95.59$_{\pm\text{3.04}}$ & 86.75$_{\pm\text{6.68}}$ & {\color[HTML]{C65911} \textbf{59.30}}$_{\pm\text{7.49}}$ & 90.79$_{\pm\text{4.75}}$ & 5.86$_{\pm\text{3.18}}$ \\

    & LSPIN & {\color[HTML]{7030A0} \textbf{81.30}}$_{\pm\text{7.97}}$ & 76.92$_{\pm\text{9.38}}$ & 53.98$_{\pm\text{8.00}}$ & {\color[HTML]{008080} \textbf{97.18}}$_{\pm\text{3.16}}$ & 87.75$_{\pm\text{6.74}}$ & 55.95$_{\pm\text{7.45}}$ & 83.47$_{\pm\text{8.59}}$ & 8.29$_{\pm\text{5.19}}$ \\

    & LLSPIN & 79.35$_{\pm\text{7.74}}$ & 70.10$_{\pm\text{12.31}}$ & 56.77$_{\pm\text{9.65}}$ & 95.50$_{\pm\text{3.60}}$ & 88.71$_{\pm\text{5.98}}$ & 57.88$_{\pm\text{6.02}}$ & 81.67$_{\pm\text{9.01}}$ & 9.00$_{\pm\text{3.65}}$ \\

    \cmidrule{2-10}
    
    & \textbf{ProtoGate (Ours)} & {\color[HTML]{008080} \textbf{83.95}}$_{\pm\text{9.82}}$ & {\color[HTML]{C65911} \textbf{93.44}}$_{\pm\text{6.37}}$ & {\color[HTML]{008080} \textbf{60.43}}$_{\pm\text{7.62}}$ & {\color[HTML]{C65911} \textbf{95.96}}$_{\pm\text{3.93}}$ & {\color[HTML]{C65911} \textbf{90.58}}$_{\pm\text{5.72}}$ & {\color[HTML]{008080} \textbf{61.18}}$_{\pm\text{6.47}}$ & {\color[HTML]{7030A0} \textbf{92.34}}$_{\pm\text{5.67}}$ & {\color[HTML]{008080} \textbf{2.00}}$_{\pm\text{1.00}}$ \\
    \bottomrule
    \end{tabular}
    }   
    \vspace{-4mm}
\end{table*}

\vspace{-2.3mm}
\section{Experiments}
\label{sec:exp}
\looseness-1
We evaluate ProtoGate by answering three questions:

\vspace{-1mm}
\textbf{Performance (Q1):} Can ProtoGate achieve competitive accuracy in classification and sparsity in feature selection against those of the benchmark methods?

\vspace{-1mm}
\textbf{Ablation Impacts (Q2):} What contributions to model performance do individual components of ProtoGate make?

\vspace{-1mm}
\textbf{Interpretability (Q3):} Can ProtoGate improve the interpretability of selected features and predictions?

\vspace{-1mm}
To answer these questions, we first compare ProtoGate against 16 benchmark methods on real-world classification tasks (\cref{sec:exp_cls}). Next, we investigate the effectiveness of global-to-local feature selection and non-parametric prototype-based prediction with ablation studies (\cref{sec:ablation_sel}). Finally, we quantitatively and qualitatively evaluate ProtoGate's interpretability by analysing the selected features and predictions (\cref{sec:exp_interpretability}).

\vspace{-1mm}
\looseness-1
\textbf{Real-world datasets.}
We use seven open-source HDLSS tabular biomedical datasets~\cite{margeloiu2022weight}, which contain 2,000-5,966 features with 62-197 samples of 2-4 different classes. We also select four challenging non-HDLSS and non-biomedical tabular datasets from the TabZilla benchmark~\cite{mcelfresh2023neural}, which contain 19-857 features with 846-2,000 samples of 4-100 classes. Full descriptions are in \cref{appendix:data_real}.

\vspace{-1mm}
\textbf{Synthetic datasets.}
We generate three challenging synthetic datasets (Syn1, Syn2 and Syn3) by adapting those in prior studies \cite{yang2022locally, yoon2018invase, jethani2021have}. In our settings, each dataset has 200 samples of 100 features, which is only 10\% of the samples and 10 times more features compared to~\citet{yang2022locally}. Each dataset has two classes, and we introduce an imbalance by generating 50 and 150 samples, respectively. The exact synthetic datasets are described in \cref{appendix:data_syn}.

\vspace{-1mm}
\textbf{Benchmark methods.}
We compare ProtoGate against 16 benchmark methods, including four standard baselines: Ridge Regression (Ridge)~\cite{cox1958regression}, SVM~\cite{cortes1995support}, KNN~\cite{fix1985discriminatory} and MLP~\cite{gorishniy2021revisiting}; six global feature selection methods (also referred to as ``global methods''): Lasso~\cite{tibshirani1996regression}, Random Forest (RF)~\cite{breiman2001random}, XGBoost~\cite{chen2016xgboost}, CatBoost~\cite{prokhorenkova2018catboost}, LightGBM~\cite{ke2017lightgbm} and STG~\cite{yamada2020feature}; six local feature selection methods (also referred to as ``local methods''): TabNet~\cite{arik2021tabnet}, L2X~\cite{chen2018learning}, INVASE~\cite{yoon2018invase}, REAL-X~\cite{jethani2021have} and LSPIN/LLSPIN~\cite{yang2022locally}.

\vspace{-1mm}
\looseness-1
\textbf{Experimental setup.}
For each dataset, we use 5-fold cross-validation and repeat it for five times, summing up to 25 runs. For the training fold, we randomly select 10\% of samples as the validation set. For each model and each dataset, we tune its hyperparameters according to the aggregated performance on validation sets across 25 runs. Our setup with validation sets leads to less training data, but it is more realistic than those in prior work~\cite{yang2022locally}, where \textit{no} validation sets are used for model selection on ``colon'' and ``toxicity'' datasets. Note that the reported results are averaged over 25 runs on test sets by default. When aggregating results across datasets, we use the average distance to the minimum (ADTM) metric via affine renormalisation between the top-performing and worse-performing models~\cite{grinsztajn2022tree}. Full experimental details are in \cref{appendix:reprod}.

\begin{figure*}[!t]
    \centering
    \subfloat{\includegraphics[width=0.7\columnwidth]{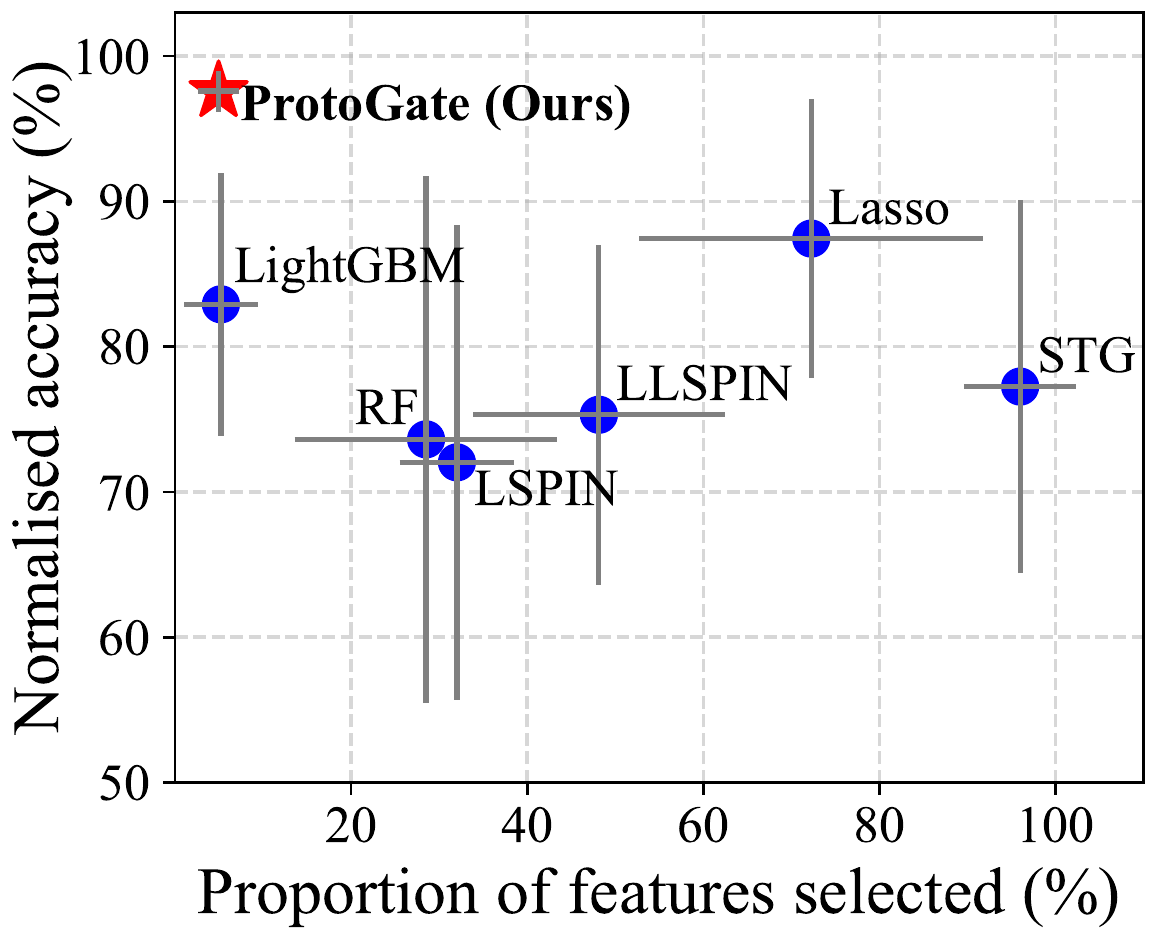}}
    \subfloat{\includegraphics[width=0.664\columnwidth]{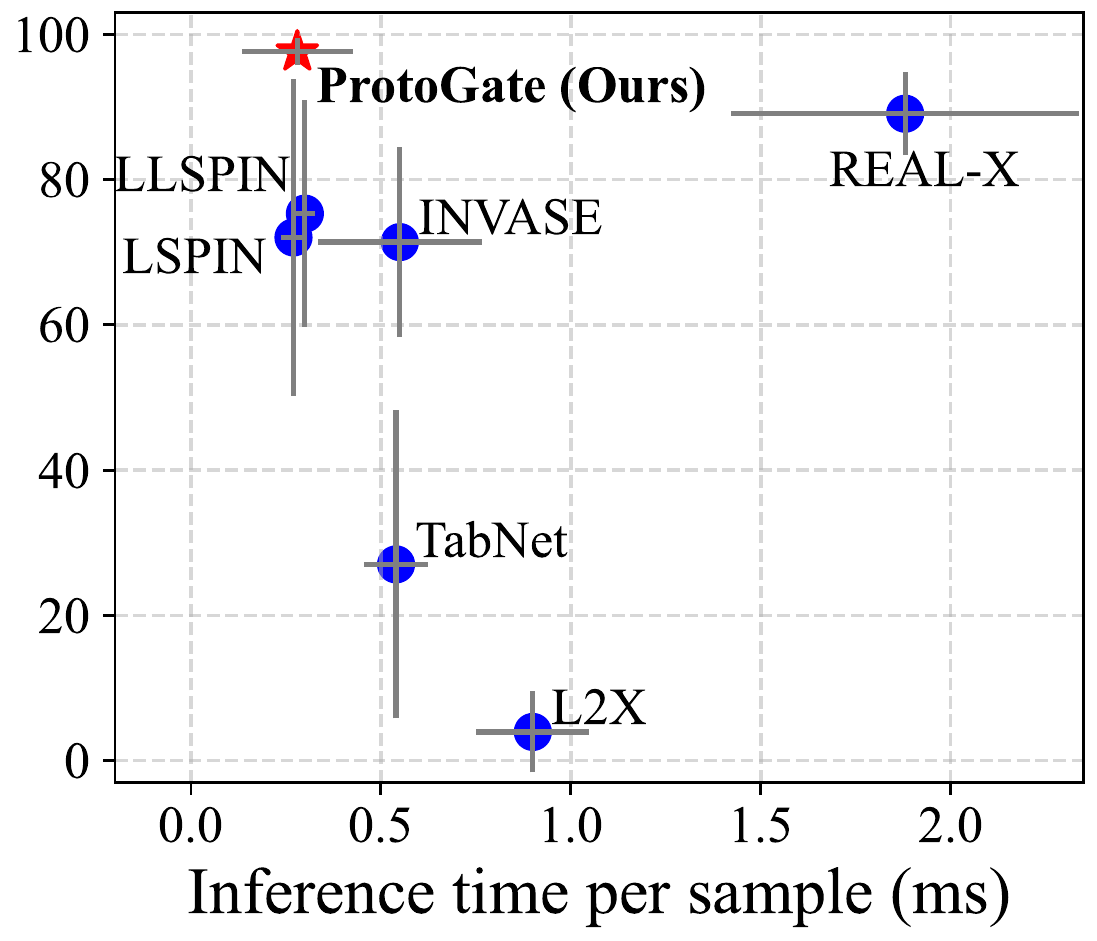}}
    \subfloat{\includegraphics[width=0.664\columnwidth]{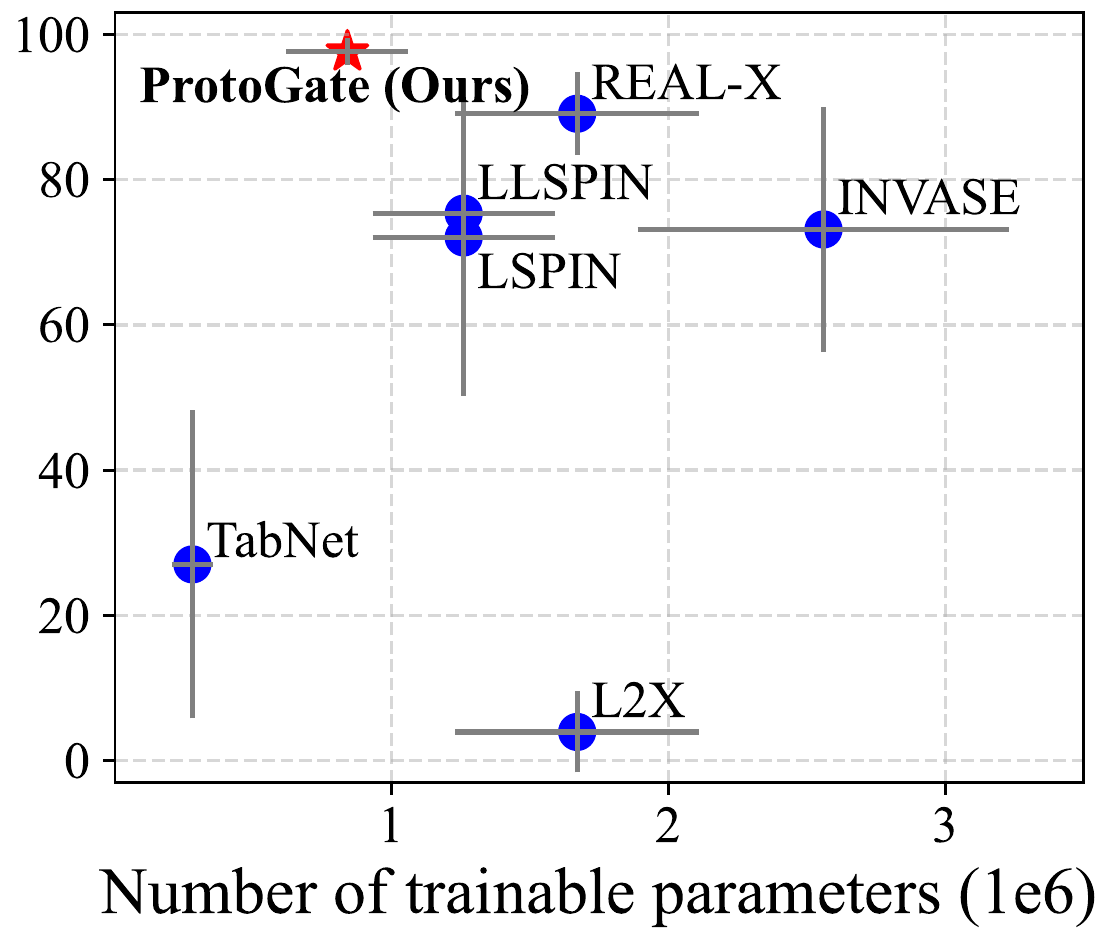}}
    \vspace{-2mm}
    \caption{\textbf{Left:} Mean normalised feature selection sparsity vs. mean normalised balanced accuracy. We exclude the outliers (TabNet, L2X and INVASE) due to their suboptimal results or failed convergence. \textbf{Middle:} Median runtime vs.\ mean normalised balanced accuracy. \textbf{Right:} Median model size vs.\ mean normalised balanced accuracy. ProtoGate generally achieves higher accuracy and fewer selected features with higher computation efficiency than other local methods.}
    \label{fig:trade_offs}
    \vspace{-4mm}
\end{figure*}

\subsection{Prediction on Real-world Datasets (Q1)}
\label{sec:exp_cls}
We compare ProtoGate against benchmark methods in real-world scenarios. We evaluate the models from four aspects: (i) classification accuracy, (ii) feature selection sparsity, (iii) computation efficiency and (iv) model generalisability.

\looseness-1
\textbf{Classification accuracy.}
We evaluate the performance in real-world classification tasks via \textit{balanced accuracy} and the corresponding \textit{rank}. Typically, higher classification accuracy denotes that the selected features are more informative.

\cref{tab:acc} shows that ProtoGate exhibits robust performance in real-world HDLSS tasks,  with notably high results on ``colon'', ``meta-dr'' and ``tcga-2y'' datasets. It consistently ranks in the top three for balanced accuracy across datasets, outperforming the benchmark methods in average rank. In particular, ProtoGate surpasses the best-performing benchmark local method, REAL-X, by a clear margin: from 5.86$\rightarrow$2.00 in average rank. In addition, ProtoGate provides \textit{in-model} explainability of selected features, while REAL-X only provides \textit{post-hoc} explainability. Although other local methods can achieve high classification accuracy (e.g., LSPIN on ``meta-pam'' dataset), they are susceptible to co-adaptation problem and hence the interpretability of their selected features cannot be fairly compared to ProtoGate's (see  \cref{sec:exp_interpretability}). Another important finding is that simple Lasso and MLP can outperform local methods on HDLSS datasets, but not ProtoGate. 

\vspace{0.7mm}
The results suggest that ProtoGate can be a robust feature selection method for HDLSS tasks. We attribute ProtoGate's competitive performance to three main reasons. Firstly, ProtoGate takes advantage of both global and local selection. ProtoGate selects instance-wise features \textit{after} explicitly capturing the globally important features, which adaptively balances global and local feature selection. Secondly, ProtoGate meticulously attends to the similarity in high-dimensional space. For instance, LSPIN has constraints to generate similar masks for similar samples in the original input space. However, the poor accuracy of KNN shows that the similarity between samples \textit{before} feature selection can be misleading for predictions, thus harming the model performance. In contrast, ProtoGate inherently leverages the similarity between samples \textit{after} feature selection. Furthermore, ProtoGate is better aligned with the clustering assumption because it encourages samples of the same class to have similar representations, \textit{not} similar masks. Thirdly, ProtoGate is parameter-efficient. In contrast to other local methods that can easily be over-parameterised due to complex network-based architectures,  ProtoGate has a sparsified gating network and a non-parametric predictor, which notably reduces the number of trainable parameters.

\vspace{0.65mm}
\looseness-1
\textbf{Feature selection sparsity.}
We evaluate the sparsity via \textit{proportion of features selected per sample} for in-model selection methods. In \cref{fig:trade_offs} (Left), we plot the ``sparsity vs. accuracy'' by aggregating the results across datasets. The numerical results and visualisation of selected features are in \cref{appendix:complete_sparsity}. Overall, ProtoGate stably achieves higher accuracy with a lower proportion of features selected per sample than benchmark methods, suggesting that the features selected by ProtoGate are easier to interpret by humans.

\vspace{0.65mm}
\looseness-1
\textbf{Feature selection behaviour.}
We further analyse the behaviour of ProtoGate via the \textit{composition of the selected features} (i.e., the proportions of ``both selected'' and ``locally recovered'' features in $\textbf{s}_\text{local}^{(i)}$). We find that ProtoGate exhibits adaptive feature selection behaviour. On most datasets (e.g., the ``colon'' dataset), the majority proportion of $\textbf{s}_\text{local}^{(i)}$ is generally included in $\textbf{s}_\text{global}$  (i.e., the proportion of ``both selected'' features generally exceeds that of ``locally recovered'' features), suggesting ProtoGate can behave more globally in feature selection. In contrast, on the other datasets (e.g., the ``meta-dr'' dataset), ProtoGate tends to recover more features locally than reusing the globally selected features, highlighting a more pronounced local selection strategy. Full results and discussion are in \cref{appendix:feat_composition} and \cref{appendix:lambda_t_test}.

\vspace{0.65mm}
\textbf{Computation efficiency.}
We evaluate the speed and model size via \textit{inference time} and \textit{number of trainable parameters}. In \cref{fig:trade_offs} (Middle), the trade-off between accuracy and inference speed shows ProtoGate excels benchmark local methods in maintaining high accuracy while keeping the inference time short. In \cref{fig:trade_offs} (Right), the trade-off between accuracy and model size further shows ProtoGate is more parameter-efficient than benchmark local methods. More results on computation efficiency are in \cref{appendix:efficiency}. The results suggest that ProtoGate can be a highly efficient local method for HDLSS tasks.

\begin{table*}[!t]
\centering
\caption{\textbf{Classification accuracy (\%) of ProtoGate and its two variants.}  We \textbf{bold} the highest accuracy for each dataset. ProtoGate consistently achieves the best accuracy across all datasets, showing the complementary effects of global and local selection.}
\label{tab:ablation_reg_acc}
\resizebox{\textwidth}{!}{
\begin{tabular}{lrrrrrrr}
\toprule
Methods & colon & lung & meta-dr & meta-pam & prostate & tcga-2y & toxicity \\
\midrule
ProtoGate-global & 75.95$_{\pm\text{11.63}}$ & 90.20$_{\pm\text{6.67}}$ & 50.66$_{\pm\text{4.08}}$ & 92.80$_{\pm\text{6.00}}$ & 89.02$_{\pm\text{5.38}}$ & 56.46$_{\pm\text{6.15}}$ & 70.27$_{\pm\text{10.31}}$ \\

ProtoGate-local & 81.45$_{\pm\text{10.87}}$ & 90.51$_{\pm\text{8.36}}$ & 58.88$_{\pm\text{8.88}}$ & 92.60$_{\pm\text{6.07}}$ & 89.58$_{\pm\text{5.96}}$ & 59.26$_{\pm\text{6.34}}$ & 87.44$_{\pm\text{6.84}}$ \\

\textbf{ProtoGate (Global-to-local)} & \textbf{83.95}$_{\pm\text{9.82}}$ & \textbf{93.56}$_{\pm\text{6.29}}$ & \textbf{60.43}$_{\pm\text{7.62}}$ & \textbf{95.96}$_{\pm\text{3.93}}$ & \textbf{90.58}$_{\pm\text{5.72}}$ & \textbf{61.18}$_{\pm\text{6.47}}$ & \textbf{92.34}$_{\pm\text{5.67}}$ \\
\bottomrule
\end{tabular}
}
\vspace{-6mm}
\end{table*}

\begin{figure*}[!t]
    \centering
    \subfloat{\includegraphics[width=0.875\columnwidth]{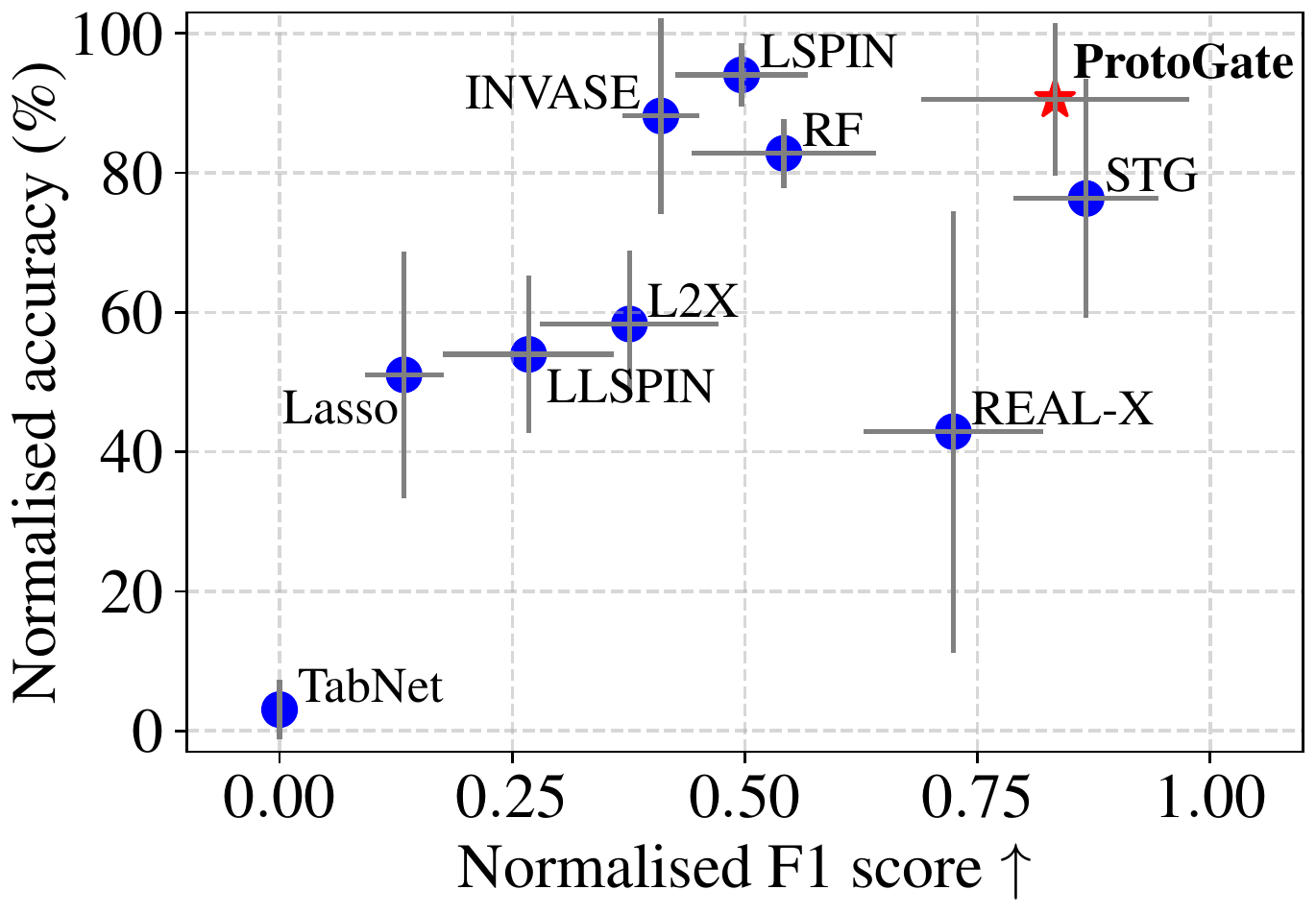}}
    \subfloat{\includegraphics[width=0.85\columnwidth]{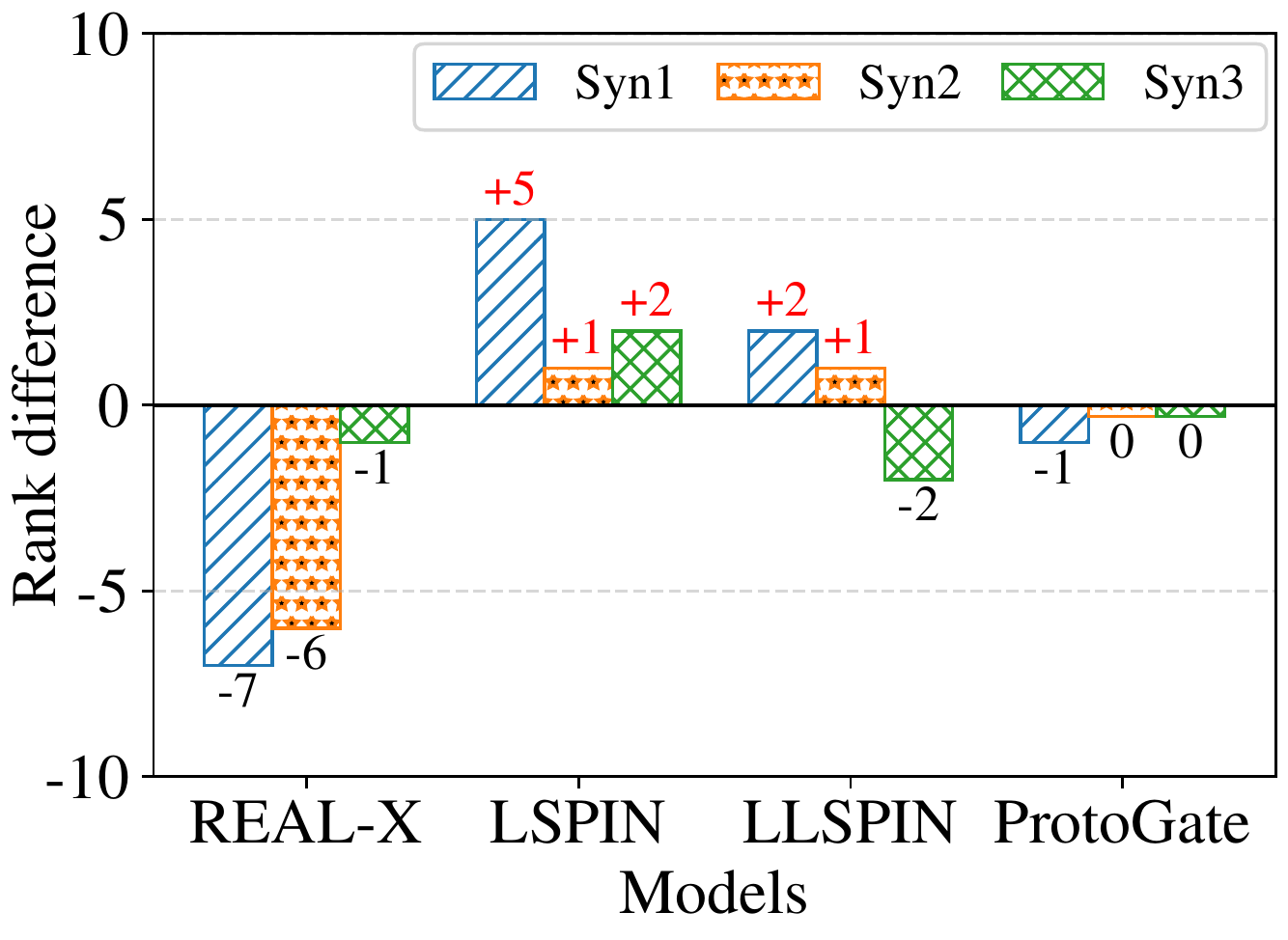}}
    \vspace{-5mm}
    \caption{\textbf{Fidelity evaluation of selected features on three synthetic datasets.} \textbf{Left:} Mean normalised F1 score of selected features (F1$_\text{select}$) vs. mean normalised balanced accuracy (ACC$_{\text{pred}}$). \textbf{Right:} Rank difference between F1$_\text{select}$ and ACC$_{\text{pred}}$. A positive value (highlighted in red) indicates low-fidelity feature selection. Note that we plot short bars at zeros for visual clearance. ProtoGate has competitive trade-off between F1$_\text{select}$ and ACC$_{\text{pred}}$, and consistently non-positive rank differences, showing robustness to co-adaptation.}
    \vspace{-5mm}
    \label{fig:interpretability_fidelity}
\end{figure*}

\looseness-1
\textbf{Model generalisability.}
We further evaluate classification performance on four non-HDLSS and non-biomedical datasets. Although ProtoGate is designed for the HDLSS regime, the evaluation results show that ProtoGate generalises well on non-HDLSS real-world datasets, which contain a much larger number of samples and classes than the considered HDLSS datasets. Specifically, ProtoGate can achieve competitive accuracy against benchmark methods, including the gradient boosting trees. Detailed results and discussion are in \cref{appendix:non_HDLSS}.

\subsection{Ablation Studies (Q2)}
\label{sec:exp_ablation}
We investigate the efficacy of ProtoGate's two main components: (i) Global-to-local Feature Selection and (ii) Non-parametric Prototype-based Prediction.

\looseness-1
\textbf{Soft
global selection and local selection are complementary, rather than interchangeable.}
\label{sec:ablation_sel}
We illustrate how ProtoGate balances global and local feature selection by performing ablation experiments on the interplay between $\lambda_{\text{global}}$ and $\lambda_{\text{local}}$. We compare ProtoGate ($\lambda_{\text{global}} \neq 0, \lambda_{\text{local}} \neq 0$) against two variants: (i) remove $\mathcal{R_{\text{local}}}(\vs_{\text{local}}^{(i)})$: ProtoGate-global ($\lambda_{\text{global}} \neq 0, \lambda_{\text{local}} = 0$) and (ii) remove $\mathcal{R_{\text{global}}}(\vs_{\text{global}})$: ProtoGate-local ($\lambda_{\text{global}} = 0, \lambda_{\text{local}} \neq 0$). Detailed experimental setup is in \cref{appendix:ablation_fs}.

\vspace{-0.5mm}
\cref{tab:ablation_reg_acc} shows that ProtoGate consistently outperforms its variants in accuracy, suggesting that either removing global selection (ProtoGate-local) or removing local selection (ProtoGate-global) leads to performance degradation. Therefore, the two stages are complementary for more accurate selection results.

\looseness-1
\cref{tab:ablation_reg_acc} further shows the effects of explicitly leveraging global information in feature selection. Note that we do not claim that the purely local methods, such as ProtoGate-local, omit global information completely. Indeed, they can promote global selection by increasing the regularisation strength for local selection to render smaller feature sets per sample. However, the ablation studies on feature selection sparsity (\cref{appendix:ablation_fs}) show that even when ProtoGate and ProtoGate-local achieve similar sparsity, the accuracy gap remains. This demonstrates that ProtoGate selects different features to ProtoGate-local (we will also see this in \cref{fig:interpretability_generalisability_normal}). On the other hand, they can also reduce the regularisation strength to render larger feature sets per sample (i.e., increasing the likelihood of selecting the same features for different samples). However, larger feature sets do \textit{not} guarantee that the additionally selected ``common'' features are informative, leading to lower accuracy. In contrast, ProtoGate explicitly promotes global information by applying $\ell_1$-regularisation on $\mathbf{W}^{[1]}$ and thus generates smaller feature sets per sample. Therefore, the soft global selection and local selection are complementary to help ProtoGate exhibit unique feature selection behaviours.

\looseness-1
\textbf{Non-parametric prototype-based prediction is appropriate for biomedical tasks.}
We investigate how the prototype-based predictor impacts classification accuracy. Specifically, we compare ProtoGate against two variants with different predictors and investigate different numbers of different nearest neighbours. The further studies (\cref{appendix:ablation_pred})
show the efficacy and robustness of non-parametric prototype-based prediction.

\subsection{Interpretability Evaluation (Q3)}
\label{sec:exp_interpretability}
We evaluate ProtoGate's interpretability by focusing on three aspects of the selected features and predictions: (i) fidelity of selected features, (ii) transferability of selected features and (iii) explainability of predictions.

\looseness-1
Note that our criteria are more realistic and comprehensive than prior work~\cite{alvarez2018towards, jethani2021have, yang2022locally}. In addition to fidelity and transferability of selected features, we also evaluate the explainability of predictions for local feature selection methods. More discussion on the metrics for interpretability evaluation is in \cref{appendix:metric_interpretability}.

\vspace{2mm}
\textbf{Fidelity of selected features.}
\label{sec:exp_syn}
This criterion focuses on \textit{``Can the classification accuracy denote feature selection performance, i.e., does the model achieve high accuracy with poorly selected features?''}, and we quantitatively evaluate it on synthetic datasets via the \textit{rank difference} between feature selection and prediction (\cref{fig:interpretability_fidelity}). Specifically, we first compute the F1 score of selected features (F1$_\text{select}$) and the balanced accuracy of predictions (ACC$_{\text{pred}}$). Then we compute the rank difference between F1$_\text{select}$ and ACC$_{\text{pred}}$. A \textit{positive} difference demonstrates the model makes accurate predictions with relatively low-fidelity features. The complete numerical results are in \cref{appendix:co_adapt}.

\looseness-1
\cref{fig:interpretability_fidelity} (Left) shows that ProtoGate achieves competitive trade-off between F1$_\text{select}$ and ACC$_{\text{pred}}$ against the benchmark methods. In contrast, some existing methods, such as LSPIN, can achieve high accuracy in prediction, while the feature selection is relatively inaccurate, indicating poor fidelity of selected features.
\cref{fig:interpretability_fidelity} (Right) further demonstrates that ProtoGate is the only in-model local feature selection method that consistently has non-positive rank differences between F1$_\text{select}$ and ACC$_{\text{pred}}$. Although REAL-X also has non-positive values across datasets, its post-hoc explainability of selected features cannot improve the prediction performance, leading to poor classification accuracy. In contrast, the in-model explainability helps ProtoGate to achieve high fidelity of selected features while guaranteeing high classification accuracy.

\looseness-1
Furthermore, the results raise concerns over the reliability of high accuracy achieved by other benchmark local methods in \cref{sec:exp_cls}. For instance,  LSPIN achieves the best classification accuracy on Syn1, but the correctness of selected features is much worse, with a rank of six out of ten methods. Consequently, LSPIN's high accuracy may not correspond to selecting the real informative features, and this is also true for other local methods that are susceptible to co-adaption. In contrast, ProtoGate's well-aligned performance between feature selection and classification guarantees the fidelity of selected features.

\begin{figure}[!t]
    \centering
    \includegraphics[width=1\columnwidth]{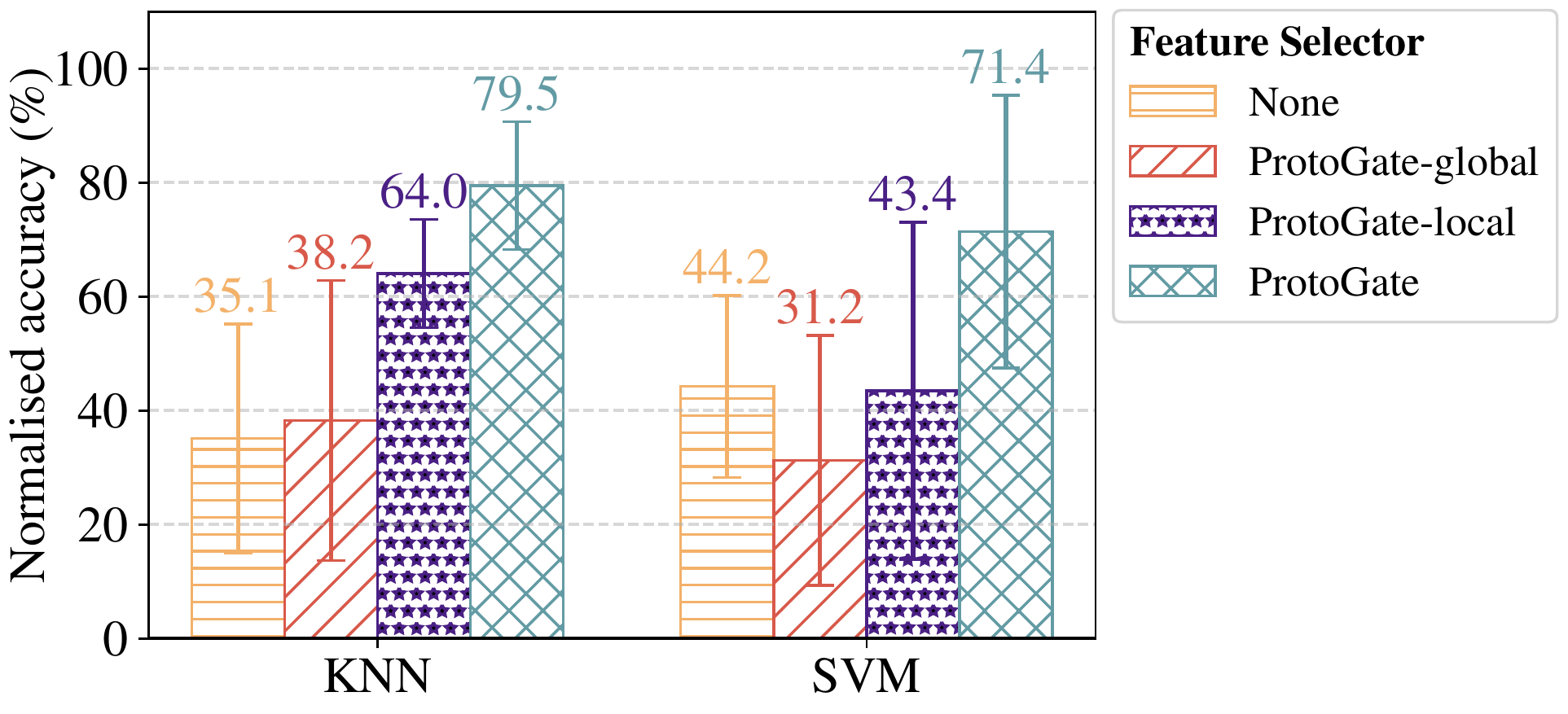}
    \vspace{-8mm}
    \caption{\textbf{Normalised balanced  accuracy (\%) of simple models with different feature selectors.} ProtoGate, with global-to-local selection, selects highly transferable features that generally improve the classification performance of KNN and SVM.}
    \vspace{-4mm}
    \label{fig:interpretability_generalisability_normal}
\end{figure}

\textbf{Transferability of selected features.}
This criterion focuses on \textit{``Can the selected features improve the performance of other simpler models?''}. We evaluate the balanced accuracy of downstream classifiers trained on the features selected by ProtoGate. Specifically, we first process the HDLSS datasets by applying $\vx \odot \vs_{\text{local}}$ -- where the mask $\vs_{\text{local}}$ is generated by ProtoGate, ProtoGate-global or ProtoGate-local -- and then train a KNN and SVM on the masked datasets. For each simple model, we aggregate the balanced accuracy across datasets with the ADTM metric.

\looseness-1
In \cref{fig:interpretability_generalisability_normal}, the features selected by ProtoGate generally improve the performance of KNN and SVM across datasets. In contrast, the benefits conferred by ProtoGate-global and ProtoGate-local are limited, even leading to lower accuracy than vanilla SVM. The fine-grained results on the accuracy improvements per dataset (see \cref{appendix:transferability} for more details) reveal that features selected by ProtoGate consistently avoid performance degradation for simple models, while ProtoGate-global and ProtoGate-local can cause a considerable drop in accuracy. These findings also support that ProtoGate selects different features to its variants that rely solely on either global or local selection, showing the efficacy of global-to-local feature selection. 

\textbf{Explainability of predictions.}
This criterion focuses on \textit{``Can the predictions be easily understood by practitioners?''}, and we qualitatively evaluate it by providing specific examples of predictions. \cref{fig:interpretability_pred} presents UMAP~\cite{mcinnes2018umap} on the ``colon'' dataset to show ProtoGate's inference process. In this split, ProtoGate classifies the new sample by explicitly pointing out the three ($K=3$) nearest prototypes as evidence. ProtoGate's prototypical predictions are easy to interpret because they resemble human behaviour~\cite{kolodner1992introduction, lu2021ace}.

\begin{figure}[!t]
    \centering
    \includegraphics[width=0.85\columnwidth]{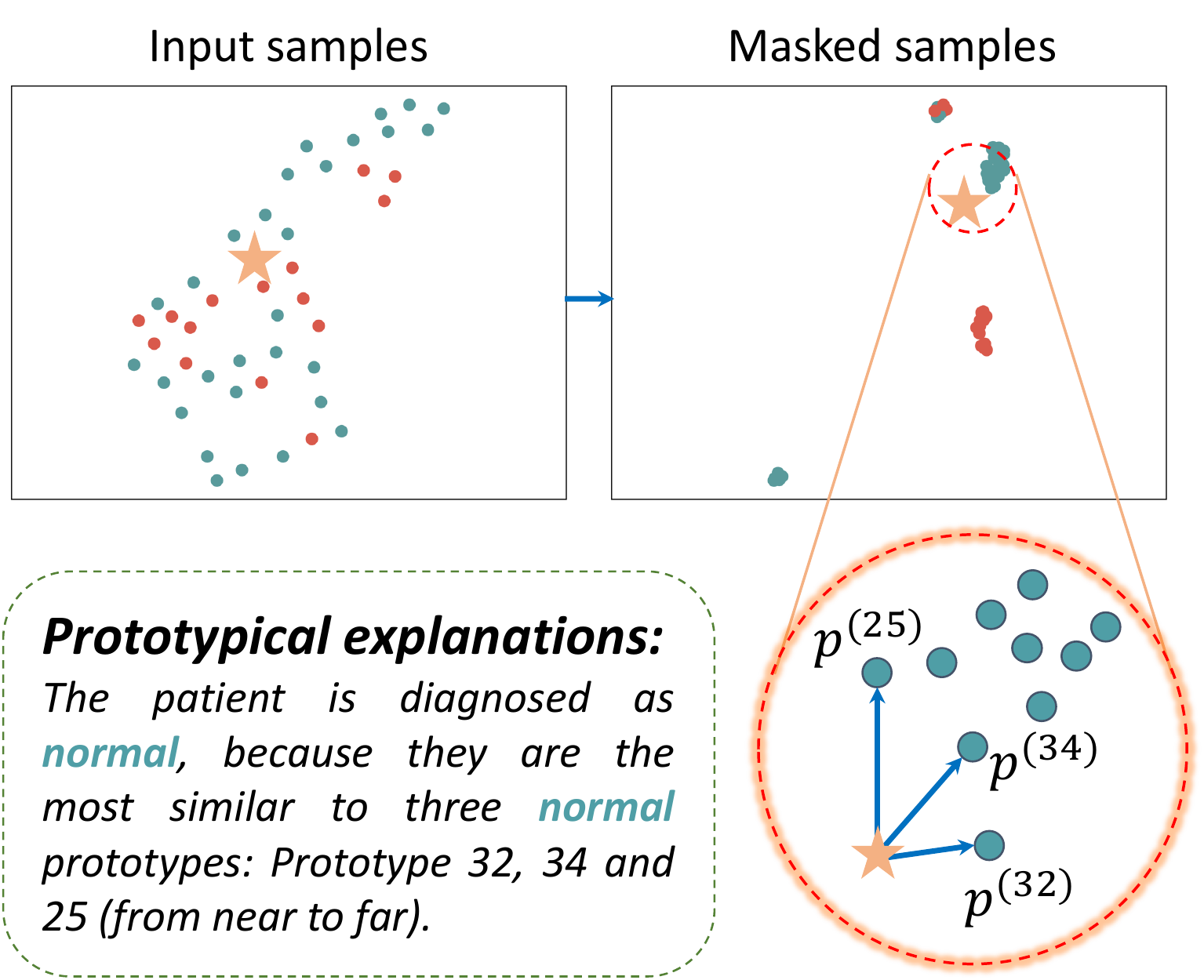}
    \vspace{-3mm}
    \caption{\textbf{Visualisation of ProtoGate's inference on the ``colon'' dataset.} The circles in ``Masked samples'' denote prototype base $\mathcal{B}$ learned by ProtoGate, containing {\color[HTML]{008080} \textbf{normal}} and {\color[HTML]{D9594B} \textbf{cancer}} prototypes $p^{(i)}$. The star is a new patient (i.e., a sample in the test set). ProtoGate provides an explainable diagnosis for the new patient.}
    \label{fig:interpretability_pred}
    \vspace{-5mm}
\end{figure}

\section{Conclusion}
\vspace{-0.19mm}

We introduce ProtoGate, a prototype-based neural model for feature selection in the high-dimensional and low-sample-size regime. ProtoGate proposes to select features in a global-to-local manner and perform non-parametric prototype-based prediction via hybrid sorting. The experimental results on real-world datasets demonstrate that ProtoGate generally improves the accuracy and feature selection sparsity of local methods on tabular datasets without sacrificing computation efficiency. The interpretability evaluation further validates that ProtoGate can effectively guarantee the fidelity and transferability of selected features by elegantly mitigating the co-adaptation problem. ProtoGate's robust performance and the ability to generate human-understandable explanations for its prototypical predictions has the potential to enhance practitioners' experience in decision-making processes beyond the biomedical domain.

\clearpage

\section*{Impact Statement}
This paper presents a novel feature selection method that aims to advance the field of machine learning by addressing challenges in the high-dimensional and low-sample-size regime. Furthermore, ProtoGate offers an elegant solution to the prevalent co-adaptation problem in machine learning models, which helps to pave the path to more robust and applicable feature selection methods. These characteristics can be particularly useful in high-stakes and data-scarce domains like healthcare, such as preclinical drug evaluation in early-stage clinical trials~\cite{bespalov2016failed, morford2011preclinical}. Moreover, ProtoGate can provide patient-wise insights with informative features and similar prototypes as explanations, facilitating automatic diagnosis systems~\cite{fansi2022towards, nazari2022explainable}.

ProtoGate's impact further extends to enabling broader machine learning applications in low-resource regions, such as low-income countries where medical resources are limited. Indeed, improving machine learning models' accuracy and interpretability in low-sample-size regimes can help elevate healthcare quality and foster medical equity~\cite{alami2020artificial, ciecierski2022artificial, mollura2020artificial}. ProtoGate can further facilitate research and enhance machine learning accessibility in various communities. Without foreseeing harmful applications, ProtoGate is designed to enhance the availability and fairness of AI across various societal and scientific domains.

\section*{Acknowledgements}
\looseness-1
AM acknowledges the support from the Cambridge ESRC Doctoral Training Partnership. NS acknowledges the support of the U.S. Army Medical Research and Development Command of the Department of Defense; through the FY22 Breast Cancer Research Program of the Congressionally Directed Medical Research Programs, Clinical Research Extension Award GRANT13769713. Opinions, interpretations, conclusions, and recommendations are those of the authors and are not necessarily endorsed by the Department of Defense.

\nocite{blondel2020fast}
\bibliography{icml2024}
\bibliographystyle{icml2024}

\clearpage

\appendix
\onecolumn

\addcontentsline{toc}{section}{Appendix}
\part{Appendix: ProtoGate: Prototype-based Neural Networks with Global-to-local Feature Selection for Tabular Biomedical Data}
\mtcsetdepth{parttoc}{3} 
\parttoc
\newpage
\section{Summary of Related Work}
\label{appendix:model_cmp}

As a supplement to \cref{sec:related_work}, we further provide a detailed summary of the related work on local feature selection and highlight the differences between ProtoGate and the prior studies. 

\vspace{-1mm}
\textbf{Joint in-model selection (\cref{fig:overview}(a)).} This line of work attempts to predict instance-wise masks for local feature selection via jointly learning a local selector and a predictor~\cite{chen2018learning, yoon2018invase, arik2021tabnet, yang2022locally, yoshikawa2022neural}. TabNet uses sequential attention to select instance-wise features for different samples~\cite{arik2021tabnet}, but it can suffer from overfitting on HDLSS datasets with the complex transformer architecture~\cite{margeloiu2022weight}. L2X uses mutual information with Concrete distribution, but it requires specifying the number of selected features~\cite{chen2018learning}. INVASE addresses such limitation by modelling each feature's mask/gate value with independent Bernoulli distributions~\cite{yoon2018invase}. However, both L2X and INVASE utilise computationally expensive gradient estimators like REINFORCE~\cite{williams1992simple} or REBAR~\cite{tucker2017rebar} to generate sparse masks. Based on STG~\cite{yamada2020feature} and Localized Lasso~\cite{yamada2017localized}, LSPIN~\cite{yang2022locally} further extends the estimation of $\ell_0$-regularisation to the exact formulation. To date, few contributions have been made to investigate the efficacy of globally important features in local feature selection. The recent Contextual Lasso~\cite{thompson2024contextual} attempts to introduce context information for local feature selection, but it requires domain knowledge to partition the features into contextual features and explanatory features, which is impractical for HDLSS regimes. Although LassoNet~\cite{lemhadri2021lassonet} features penalisation on the residual layer parameters and constraints on the first layer's parameters, it can only select features globally and has an unstable training process, even failed convergence, on HDLSS datasets~\cite{margeloiu2022weight, margeloiu2023gcondnet}. In addition, the paradigm to jointly learn a feature selector and a predictor is susceptible to the co-adaptation problem, leading to a considerable loss in the fidelity of selected features~\cite{jethani2021have, adebayo2018sanity, hooker2019benchmark}. Moreover, these methods typically make predictions without a tailored inductive bias for biomedical data, further highlighting their limitations. 

\vspace{-1mm}
These challenges in this relatively underexplored field suggest the complexity and non-trivial nature of proposing seemingly straightforward ideas. Indeed, our work bridges these gaps by (i) exploring the interplay between global and local selection, (ii) tackling the co-adaptation problem via incorporating a trainable feature selector with a non-trainable predictor and (iii) encoding the clustering assumption as an inductive bias by performing prototype-based prediction.

\vspace{-1mm}
\textbf{Disjoint post-hoc selection (\cref{fig:overview}(b)).}
Another promising line of work addresses the heterogeneity across samples by designing models to explain a trained predictor via local feature selection~\cite{jethani2021have, ribeiro2016should, lundberg2017unified, shrikumar2017learning, simonyan2014deep, lundberg2018consistent, bach2015pixel}. However, these methods can have limited applicability because they can only provide post-hoc analysis of the important features. Furthermore, the features they identify do not enhance prediction performance, given that the predictor has already been trained using all features. In contrast, ProtoGate directly utilises the selected features for predictions, thereby not only boosting the non-parametric predictor's accuracy but also preserving the in-model explainability of selected features. Moreover, ProtoGate surpasses post-hoc selection approaches such as REAL-X in computation efficiency, as it mitigates the need to enumerate various mask possibilities for an accurate approximation of the predictor's decision boundary.

\begin{table}[!htbp]
\centering
\vspace{-2mm}
\caption{\textbf{Model design comparison between ProtoGate and prior local feature selection methods.} ProtoGate has novel design rationales and leverages different feature selectors and predictors to the benchmark methods.}
\label{tab:related_work}
\resizebox{\textwidth}{!}{
    \begin{tabular}{l|c|c|c|c|c|c|c}
    \Xhline{3\arrayrulewidth}
    \multirow{2}{*}{\parbox{0.1\linewidth}{\vspace{5mm} Methods}} & \multirow{2}{*}{\parbox{0.08\linewidth}{\vspace{5mm} Category}}                & \multirow{2}{*}{\parbox{0.08\linewidth}{\vspace{5mm} Key ideas}}                                                                                                    & \multicolumn{2}{c|}{Feature Selection}                                          & \multicolumn{2}{c|}{Prediction}   & \multirow{2}{*}{\parbox{0.1\linewidth}{\vspace{3mm} \thead{Co-adaptation\\avoidance}}}                             \\
    
    \cline{4-7}
                                                 &                                             &                                                                                                                                & Global   Selector             & Local Selector                                 & Explainability                & \thead{Clustering\\assumption}         \\
    
    \hline
    TabNet                                       & \multirow{5}{*}{\parbox{0.2\linewidth}{\vspace{22mm} Joint in-model selection}} & \thead{Uses sequential attention \\to select important features}                                                      & {\color[HTML]{E6341C} \XSolidBrush}  & {\color[HTML]{3A7B21} \CheckmarkBold} & {\color[HTML]{E6341C} \XSolidBrush}   & {\color[HTML]{E6341C} \XSolidBrush}  & {\color[HTML]{E6341C} \XSolidBrush} \\
    
    L2X                                          &                                             & \thead{Mutual Information maximization \\with Concrete distribution}                                                                   & {\color[HTML]{E6341C} \XSolidBrush}   &              {\color[HTML]{3A7B21} \CheckmarkBold}                                  & {\color[HTML]{E6341C} \XSolidBrush}   & {\color[HTML]{E6341C} \XSolidBrush}  & {\color[HTML]{E6341C} \XSolidBrush} \\
    
    INVASE                                       &                                             & \thead{Modelling instance-wise mask value \\with independent Bernoulli distribution}                                                   & {\color[HTML]{E6341C} \XSolidBrush}   &                {\color[HTML]{3A7B21} \CheckmarkBold}                                & {\color[HTML]{E6341C} \XSolidBrush}   & {\color[HTML]{E6341C} \XSolidBrush}  & {\color[HTML]{E6341C} \XSolidBrush} \\
    
    LSPIN                                        &                                             & \thead{Formalising the mask value \\with Normal distribution}                                                                          & {\color[HTML]{E6341C} \XSolidBrush}   &                   {\color[HTML]{3A7B21} \CheckmarkBold}                             & {\color[HTML]{E6341C} \XSolidBrush}   & {\color[HTML]{E6341C} \XSolidBrush}  & {\color[HTML]{E6341C} \XSolidBrush} \\
    
    LLSPIN                                       &                                             & \thead{Replace LSPIN's predictor \\with a linear prediction head}                                                                      & {\color[HTML]{E6341C} \XSolidBrush}   &                    {\color[HTML]{3A7B21} \CheckmarkBold}                            & {\color[HTML]{3A7B21} \CheckmarkBold} & {\color[HTML]{E6341C} \XSolidBrush}  & {\color[HTML]{E6341C} \XSolidBrush} \\
    
    \cline{1-3}
    REAL-X                                       & Disjoint post-hoc selection                 & \thead{Decoupling the training objectives \\of the feature selector and predictor}                                                     & {\color[HTML]{E6341C} \XSolidBrush}   &                 {\color[HTML]{3A7B21} \CheckmarkBold}                               & {\color[HTML]{E6341C} \XSolidBrush}   & {\color[HTML]{E6341C} \XSolidBrush} & {\color[HTML]{3A7B21} \CheckmarkBold}  \\
    
    \cline{1-3}
    \textbf{ProtoGate (Ours)}                  & Disjoint in-model selection                 & \thead{Balances global and local selection \\and avoiding co-adaptation with\\non-parametetric prototype-based prediction} & {\color[HTML]{3A7B21} \CheckmarkBold} &                      {\color[HTML]{3A7B21} \CheckmarkBold}                          & {\color[HTML]{3A7B21} \CheckmarkBold} & {\color[HTML]{3A7B21} \CheckmarkBold} & {\color[HTML]{3A7B21} \CheckmarkBold} \\
    \Xhline{3\arrayrulewidth}
    \end{tabular}
}
\vspace{-3mm}
\end{table}

\clearpage
\section{Pseudocode for ProtoGate Training and Inference}
\label{appendix:algo}

\begin{algorithm*}[!hbpt]  
  \caption{Training procedure of ProtoGate}  
  \label{alg:model_train}
  \begin{algorithmic}  
    \REQUIRE  
    training samples $X \in \sR^{N \times D}$, labels $Y \in \sR^{N \times 1}$, global-to-local feature selector~$S_\mathbf{W}$, prototype-based predictor~$F_{\theta}$, sparsity hyperparameters ($\lambda_{\text{global}}$, $\lambda_{\text{local}}$), number of nearest neighbours~$K$,  total training epochs~$E$, learning rate~$\alpha$
    \ENSURE
    trained feature selector $S_\mathbf{W}$, prototype base $\mathcal{B}$
    \STATE{$\mathbf{W}$ $\leftarrow$ GaussianInitialisation()} 
    \COMMENT{Initialise the weights of feature selector}

    \FOR{$e \gets 1$ to $E$}
        \STATE{$\mathcal{B} \gets \{\}$}
        \COMMENT{Initialise the prototype base as an empty set}
        \STATE{$\vs_{\text{global}} \gets \mathbf{W}^{[1]}$}
        \COMMENT{Compute the global mask from $\mathbf{W}^{[1]}$}

        \FOR{$i \gets 1$ to $N$}
            \STATE{$\vs_{\text{local}}^{(i)} \gets S_\mathbf{W}(\vx^{(i)} \odot \vs_{global})$}
            \COMMENT{Predict the local mask for a specific training sample}
            \STATE{$p^{(i)} \gets (\vx^{(i)} \odot \vs_{\text{local}}^{(i)}, y^{(i)})$}
            \COMMENT{Construct a prototype with the masked training sample and its label}
            \STATE{$\mathcal{B} \gets \mathcal{B} \cup \{p^{(i)}\}$}
            \COMMENT{Retain the prototype in the base}
        \ENDFOR

        \FOR{$j \gets 1$ to $N$}
            \STATE{$\vx_{\text{query}}^{(j)} \gets \text{RandomSampling}(X)$}
            \COMMENT{Randomly pick a training sample as the query sample}
            \STATE{$\vx_{\text{masked}}^{(j)} \gets \vx_{\text{query}}^{(j)} \odot S_\mathbf{W}(\vx_{\text{query}}^{(j)} \odot \vs_{\text{global}})$}
            \COMMENT{Perform local feature selection for the query sample}
            \STATE{$\widehat{\mathbf{P}}_{\text{query}}^{(j)} \gets \text{HybridSort}(X_\mathcal{B}, \vx_{\text{masked}}^{(j)})$}
            \COMMENT{Compute the permutation matrix based on the similarity to the query sample}
            \STATE{$\hat{y}_{\text{query}}^{(j)} \gets \text{MajorityClass} \left( (\widehat{\mathbf{P}}_{\text{query}}^{(j)}Y_\mathcal{B})[1:K] \right) $}
            \COMMENT{Predict with the majority class of the $K$ nearest prototypes}
        \ENDFOR

        \STATE{$\mathcal{L_{\text{total}}} \gets \mathbb{E}_{(X,Y)} \left[ \mathcal{L}_{\text{pred}} + \mathcal{R}_{\text{select}} \right]$}
        \COMMENT{Compute the training loss according to \cref{eq:reg_sparsity} and \cref{eq:dknn_loss}}
        \STATE{$\mathbf{W} \gets \mathbf{W} - \alpha \nabla_\mathbf{W} \mathcal{L}_{\text{total}}$}
        \COMMENT{Update the weights of feature selector}
    \ENDFOR
    
    \STATE {return $S_\mathbf{W}$, $\mathcal{B}$}
  \end{algorithmic}  
\end{algorithm*}

\begin{algorithm*}[!hbpt]  
  \caption{Inference procedure of ProtoGate}  
  \label{alg:model_inference}
  \begin{algorithmic}  
    \REQUIRE  
    test sample $\vx_{\text{test}} \in \sR^{D}$
    \ENSURE
    predicted label $\hat{y}_{\text{test}} \in \sR$, prototypical explanations $C_{\text{explanation}}$

    \STATE{$\vs_{\text{global}} \gets \mathbf{W}^{[1]}$}
    \COMMENT{Compute the global mask from $\mathbf{W}^{[1]}$}
    \STATE{$\vs_{\text{local}} \gets S_\mathbf{W}(\vx_{\text{test}} \odot \vs_{\text{global}})$}
    \COMMENT{Predict the local mask for the test sample}
    \STATE{$\vx_{\text{masked}} \gets \vx_{\text{query}} \odot S_\mathbf{W}(\vx_{\text{query}} \odot \vs_{\text{global}})$}
    \COMMENT{Perform local feature selection on the test sample}
    \STATE{$\widehat{\mathbf{P}}_{\text{test}} \gets \text{HybridSort}(X_\mathcal{B}, \vx_{\text{masked}})$}
    \COMMENT{Compute the permutation matrix based on the similarity to the test sample}
    \STATE{$\hat{y}_{\text{test}} \gets \text{MajorityClass} \left( (\widehat{\mathbf{P}}_{\text{test}}Y_\mathcal{B})[1:K] \right) $}
    \COMMENT{Predict with the majority class of the $K$ nearest prototypes}
    \STATE{$C_{\text{explanation}} \gets \left( (\widehat{\mathbf{P}}_{\text{test}}X_\mathcal{B})[1:K], (\widehat{\mathbf{P}}_{\text{test}}Y_\mathcal{B})[1:K] \right)$}
    \COMMENT{Use the $K$ nearest prototypes as prototypical explanations}
    
    \STATE {return $\hat{y}_{\text{test}}$, $C_{\text{explanation}}$}
  \end{algorithmic}  
\end{algorithm*}

\clearpage
\section{Illustration of Global-to-local Feature Selection}
\label{appendix:fs_illustration}
In \cref{fig:fs_illustration}, we illustrate the interplay between soft global selection and local selection with a real-world example from the ``colon'' dataset. Specifically, we randomly pick a data split and visualise the mask values of the first training sample (i.e., $\vs_{\text{global}}$ and $\vs_{\text{local}}^{(1)}$). We can see the four different feature selection behaviours: \textit{(i) Both selected:} the 61st feature is selected by both global and local selection; \textit{(ii) Locally dropped:} the 828th feature is globally selected, but it is locally dropped for this sample; \textit{(iii) Locally recovered:} the 1355th feature is globally dropped, but it is locally recovered for this sample; \textit{(iv) Both dropped:} the 1759th feature is dropped globally and locally. These results show ProtoGate's flexible feature selection process, thus allowing the model to adapt to various datasets via balancing global and local selection. 

\begin{figure*}[!htbp]
    \centering
    \includegraphics[width=0.6\columnwidth]{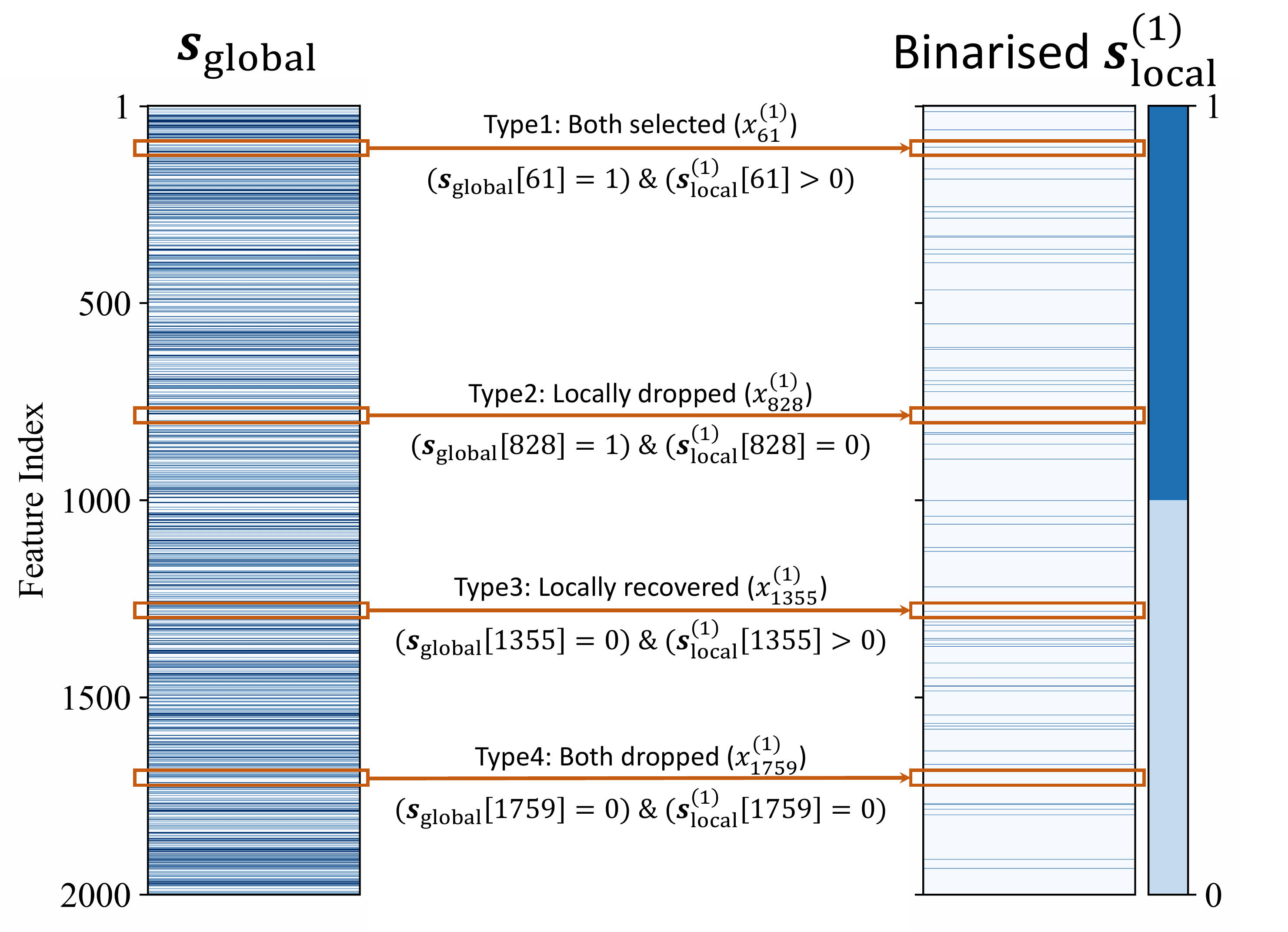}
    \caption{\textbf{Illustration of the global-to-local feature selection.} We visualise the mask values of the first training sample (i.e., $\vx^{(1)}$) from the ``colon'' dataset. Note that we binarise the local mask via $\mathbbm{1}(\vs_{\text{local}}^{(1)} > 0)$. The heatmap presents four different feature selection behaviours, showing ProtoGate's capability to adaptatively balance global and local selection.}
    \label{fig:fs_illustration}
\end{figure*}

\clearpage
\section{Theoretical Analysis}
We now provide the proof sketch for readers to understand the key ideas of ProtoGate, which revolves around computing gradients for the approximation of non-differentiable operations. Therefore, $\mathcal{L}_{\text{total}}$ boils down to deriving the empirical estimations of $\ell_0$ norm of local masks $\norm{\vs_{\text{local}}^{(i)}}_0$ and permutation matrix $\mathbf{P}$, which we show below respectively.

\paragraph{Estimation of $\norm{\vs_{\text{local}}^{(i)}}_0$.}
\label{appendix:reg}
The derivation of differentiable $\ell_0$-regularisation follows prior work~\cite{louizos2018learning, rolfe2016discrete, yamada2020feature, yang2022locally}. Recall the local mask is computed by $\vs_\text{local}^{(i)} = \max(0, \min(1, \vmu^{(i)} + \bm{\epsilon}^{(i)}))$, where $\bm{\epsilon}^{(i)}$ is Gaussian noise sampled from $\mathcal{N}(\bm{0}, \bm{\Sigma})$. Given that \(\bm{\epsilon}^{(i)}\) is sampled from a Gaussian distribution with mean vector \(\bm{0}\) and covariance matrix \(\bm{\Sigma}\), and noting the independence and identical distribution (i.i.d.) of the noise across dimensions, \(\bm{\Sigma}\) simplifies to \(\sigma \mathbf{I}\), where $\mathbf{I} \in \sR^{D \times D}$ is an identity matrix. The standard deviation $\sigma=0.5$ is fixed during training, and it is removed during the inference time for deterministic mask values. Subsequently, the relaxed sparsity regularisation for the local mask can be re-formalised as an expectation over the probabilistic distribution induced by the Gaussian noise:
\begin{equation}
\begin{split}
    \mathcal{R_{\text{local}}}(\vs_{\text{local}}^{(i)}) &= \lambda_{\text{local}}\norm{\vs_{\text{local}}^{(i)}}_0 \propto \lambda_{\text{local}}\mathbb{E}\left[\norm{\vs_{\text{local}}^{(i)}}_0 \right] \\
    &\propto \lambda_{\text{local}}\mathbb{E}_{\bm{\epsilon}^{(i)} \sim \mathcal{N}(\bm{0}, \sigma \mathbf{I})}\left[ \norm{\max\left( 0, \min \left( 1, \vmu^{(i)} + \bm{\epsilon}^{(i)} \right) \right)}_0 \right] \\
    &\propto \lambda_{\text{local}}\sum_{d=1}^{D} P \left( \mu_d^{(i)} + \epsilon_d^{(i)}>0 \right)\\
    &\propto \lambda_{\text{local}}\sum_{d=1}^{D} P \left( \frac{\epsilon_d^{(i)}}{\sigma}>- \frac{\mu_d^{(i)}}{\sigma} \right) \\
    &\propto \lambda_{\text{local}}\sum_{d=1}^{D} Q\left( - \frac{\mu_d^{(i)}}{\sigma} \right) \\
    &\propto \lambda_{\text{local}}\sum_{d=1}^{D}\int_{\hat{\mu}_d^{(i)}}^{\infty}\exp(-\frac{t^2}{2})dt
\end{split}
\end{equation}
where $\hat{\mu}_d^{(i)}=-\frac{\mu_d^{(i)}}{\sigma}$ is the standardised output of $S_{\mathbf{W}}$. Hence, the complete sparsity regularisation can be approximated via:
\begin{equation}
\begin{split}
    \mathcal{R}_{\text{select}} (\vs_{\text{global}}, \vs_{\text{local}}^{(i)}) &\coloneqq \mathcal{R_{\text{global}}}(\vs_{\text{global}}) + \mathcal{R_{\text{local}}}(\vs_{\text{local}}^{(i)}) \\
    &\propto \lambda_{\text{global}}\norm{\mathbf{W}^{[1]}}_1 + \lambda_{\text{local}}\sum_{d=1}^{D}\int_{\hat{\mu}_d^{(i)}}^{\infty}\exp(-\frac{t^2}{2})dt.
\end{split}
\end{equation}

\paragraph{Estimation of $\mathbf{P}$.}
\label{appendix:pred_loss}
The essence of prototype-based prediction is to sort the prototypes based on their similarities to the query sample. Consider the matrix of the prototype features $X_\mathcal{B} \coloneqq [\vx^{(1)} \odot \vs_{\text{local}}^{(1)},\dots,\vx^{(N)} \odot \vs_{\text{local}}^{(N)}]^{\top} \in \sR^{N \times D}$ Sorting the prototypes is equivalent to performing row operations on $X_\mathcal{B}$ via left-multiplying by a permutation matrix $\mathbf{P} \in \sR^{N \times N}$.
\begin{definition}
    Let $\vv \in \sR^{N}$ represent the vector of similarity, where each element $\vv[n]$ signifies the reciprocal of the distance between the query sample and the $n$-th prototype (i.e., $\frac{1}{\text{dist}(x_{\text{query}}, \vx^{(n)} \odot \vs_{\text{local}}^{(n)})}$). Let $A \in \sR^{N \times N}$ denote the matrix of absolute pairwise differences of the distances such that $A[n, m]=|v_n - v_m|$, and $\mathbf{1} \in \sR^{N \times 1}$ denote a column vector of all ones. The permutation matrix can be defined as follows:
    \begin{equation}
        \mathbf{P}[n, m]=
            \left\{\begin{array}{l}
            1, \ \ \  \text{if } m=\arg \max \left[(N+1-2 n)\vv - A \mathbf{1}\right] \\
            0, \ \ \  \text{otherwise}
            \end{array}\right.
    \end{equation}
    To circumvent the non-differentiability of $\arg \max$, we use its continuous approximation $\operatorname{soft} \max$ to compute the relaxed permutation matrix:
    \begin{equation}
        \widehat{\mathbf{P}}[n,:](\tau)=\operatorname{soft} \max \left[ \frac{\left((N+1-2 n)\vv - A \mathbf{1}\right)}{\tau}\right]
        \label{eq:relaxed_sort}
    \end{equation}
    where $\tau > 0$ is a temperature parameter. 
\end{definition}

\begin{lemma}
\label{lemma:lim}
     For the relaxed permutation matrix $\widehat{\mathbf{P}}$, assuming the entries of $\vv$ are independently drawn from a distribution continuously relative to the Lebesgue measure on $\sR$, then the convergence holds almost surely:
    \begin{equation}
        \lim _{\tau \rightarrow 0^{+}} \widehat{\mathbf{P}}[n,:](\tau)=\mathbf{P}[n,:], \quad \forall n \sim \mathcal{U}\{1, N\}
    \end{equation}
    where $\mathcal{U}$ denotes a discrete uniform distribution. This convergence is substantiated by ``Theorem 4'' in~\citet{grover2018stochastic}.
\end{lemma}
\begin{theorem}
\label{theo:pred_loss}
    In the context outlined above, the prediction loss of ProtoGate (\cref{eq:dknn_loss}) estimates the count of prototypes whose labels are different to that of the query sample.
\end{theorem}
\begin{proof}
    Invoking \cref{lemma:lim}, the prediction loss limit is computed as follows:
    \begin{equation}
        \begin{split}
            \lim_{\tau \to 0^{+}} \mathcal{L}_{\text{pred}} 
            &= K - \lim_{\tau \to 0^{+}} \sum_{n=1}^K \mathbb{E}_{\widehat{\mathbf{P}}[n,:](\tau)} \left[\mathbbm{1}(y^{(m)} = y_{\text{query}})\right] \\
            &= K - \sum_{n=1}^K \sum_{m=1}^{N} \mathbf{P}[n, m] \mathbbm{1}(y^{(m)} = y_{\text{query}})
        \end{split} 
    \end{equation}
    Since $\mathbf{P}$ is a binary, row-stochastic matrix, it holds that
    \begin{equation}
    \label{eq:nth_prototype}
        \sum_{m=1}^{N} \mathbf{P}[n, m] = 
        \begin{cases}
            1, & \text{if }(\mathbf{P}Y_{\mathcal{B}})[n] = y_{\text{query}} \\
            0, & \text{otherwise}
        \end{cases}
    \end{equation}
    Here, $(\mathbf{P}Y_{\mathcal{B}})[n]$ represents the label of the $n$-th nearest prototype. \cref{eq:nth_prototype} indicates whether the $n$-th nearest prototype shares the same label as the query sample. Summing over the top $K$ rows of $\mathbf{P}$ then calculates the count of prototypes that belong to the same class as the query sample among the top $K$ nearest prototypes.

Therefore, minimising \cref{eq:dknn_loss} is equivalent to reducing the number of nearest prototypes with labels differing from the query sample, thereby potentially enhancing classification accuracy.
\end{proof}

\clearpage
\section{Reproducibility}
\label{appendix:reprod}

\subsection{Real-word Datasets}
\label{appendix:data_real}

\paragraph{HDLSS datasets.}
\looseness-1
All seven HDLSS datasets are publicly available, and the details are listed in \cref{tab:hdlss_dataset}. Four of them are available online (\url{https://jundongl.github.io/scikit-feature/datasets}): \textbf{lung}~\cite{bhattacharjee2001classification}, \textbf{Prostate-GE} (referred to as ``prostate'')~\cite{singh2002gene}, \textbf{TOX-171} (referred to as ``toxicity'')~\cite{bajwa2016cutting} and \textbf{colon}~\cite{ding2005minimum}. The methods to build the other three datasets are detailed below.

In accordance with the methodology presented in~\cite{margeloiu2022weight}, we derived two datasets from the \textbf{METABRIC} dataset~\cite{curtis2012genomic}. We combined the molecular data with the clinical label ``DR'' to create the \textbf{``meta-dr''} dataset, and we combined the molecular data with the clinical label ``Pam50Subtype'' to create the \textbf{``meta-pam''} dataset. Because the label ``Pam50Subtype'' was very imbalanced, we transformed the task into a binary task of basal vs non-basal by combining the classes ``LumA'', ``LumB'', ``Her2'', ``Normal'' into one class and using the remaining class ``Basal'' as the second class. For both ``meta-dr'' and ``meta-pam'', we selected the Hallmark gene set~\cite{liberzon2015molecular} associated with breast cancer, and the new datasets contain $4160$ expressions (features) for each patient. We randomly sampled \textbf{200} patients while maintaining stratification to create the final datasets, as our focus is on the HDLSS regime.

Following the procedure by~\citet{margeloiu2022weight}, we also derived \textbf{``tcga-2y''} dataset from the \textbf{TCGA} dataset~\cite{tomczak2015cancer}. We combined the molecular data and the label ``X2yr.RF.Surv'' to create the \textbf{``tcga-2ysurvival''} dataset. Similar to the previous datasets, we selected the Hallmark gene set \citep{liberzon2015molecular} associated with breast cancer, resulting in $4381$ expressions (features). We randomly sampled \textbf{200} patients while maintaining stratification to create the final datasets, as our primary focus is on the HDLSS regime.

\begin{table*}[!htbp]
\centering
\caption{Details of seven HDLSS real-world tabular datasets.}
\label{tab:hdlss_dataset}
\begin{tabular}{lrrrr}
\toprule
Dataset         & \# Samples & \# Features & \# Classes & \# Samples per class  \\ \midrule
colon           & 62         & 2,000        & 2          & {[}40, 22{]}          \\
lung            & 197        & 3,312        & 4          & {[}139, 21, 20, 17{]} \\
meta-dr     & 200        & 4,160        & 2          & {[}139, 61{]}         \\
meta-pam  & 200        & 4,160        & 2          & {[}167, 33{]}         \\
prostate        & 102        & 5,966        & 2          & {[}52, 50{]}          \\
tcga-2y & 200        & 4,381        & 2          & {[}122, 78{]}         \\
toxicity        & 171        & 5,748        & 4          & {[}45, 45, 42, 39{]}  \\
\bottomrule
\end{tabular}
\end{table*}

\paragraph{non-HDLSS datasets.}
To show the generalisability of ProtoGate, we select four representative and challenging non-HDLSS datasets from the open-source TabZilla benchmark~\cite{mcelfresh2023neural}. Firstly, their dimensionalities are smaller than the number of samples (i.e., $N>D$). Secondly, they are from various domains beyond the biomedical scenarios. For instance, ``cnae-9'' is relevant to business descriptions. Thirdly, they have many more classes than the HDLSS datasets we mentioned. For instance, ``100-plant-texture'' has up to 100 classes.

\begin{table*}[!htbp]
\centering
\caption{Details of four non-HDLSS real-world tabular datasets.}
\label{tab:non_hdlss_dataset}
\begin{tabular}{lrrrr}
\toprule
Dataset         & \# Samples & \# Features & \# Classes & \# Samples per class  \\ \midrule
100-plants-texture           & 1,599         & 65        & 100          & 16 per class except 15 of ``Class 100''          \\
cnae-9            & 1,080        & 857        & 9          & 120 per class \\
mfeat-fourier  & 2,000        & 77        & 10          & 200 per class         \\
vehicle        & 846        & 19        & 4          & {[}218, 217, 212, 199{]}          \\
\bottomrule
\end{tabular}
\end{table*}

\subsection{Synthetic Datasets}
\label{appendix:data_syn}
The synthetic datasets are adapted from the nonlinear datasets proposed by \citet{yoon2018invase}. Specifically, we generate three synthetic datasets: Syn1 (also referred to as ``Syn1$_{(+)}$''), Syn2 (also referred to as ``Syn2$_{(+)}$''), and Syn3 (also referred to as ``Syn3$_{(-)}$''), which are designed for the classification task. Each sample is characterised by 100 features, where the feature values are independently sampled from a Gaussian distribution $\mathcal{N}(\mathbf{0}, \mathbf{I})$, with $\mathbf{I}$ representing a $100 \times 100$ identity matrix. The ground truth label (target) $y$ for each sample is computed by:
\begin{equation}
    y = \mathbbm{1}(\frac{1}{1 + \text{logit}(\vx)} > 0.5)
\end{equation}
where $\mathbbm{1}(\cdot)$ is the indicator function. For each samples, the $\text{logit}(\vx)$ is computed with a small proportion of its features:
\begin{equation}
\label{eq:syn1}
    \textbf{Syn1}_{(+)}: \ \text{logit} = \begin{cases}        \exp(x_1 x_2 - x_3) & \text{if} \ x_{11} < 0 \\         \exp(x_3^2 + x_4^2 + x_5^2 + x_6^2 - 4) & \text{otherwise}\\    \end{cases}
\end{equation}

\begin{equation}
\label{eq:syn2}
    \textbf{Syn2}_{(+)}:\ \text{logit} = \begin{cases}        \exp(x_3^2 + x_4^2 + x_5^2 + x_6^2 + x_7^2 - 4) & \text{if}\ x_{11} < 0 \\         \exp(-10 \sin(0.2 x_7) + |x_8| + x_9^2 + \exp(-x_{10}) - 2.4) & \text{otherwise} \\    \end{cases}
\end{equation}

\begin{equation}
\label{eq:syn3}
    \textbf{Syn3}_{(-)}:\ \text{logit} = \begin{cases}        \exp(x_1 x_2 + |x_9|) & \text{if}\ x_{11} < 0 \\         \exp(-10 \sin(0.2 x_7) + |x_8| + x_9^2 + \exp(-x_{10}) - 2.4) & \text{otherwise} \\    \end{cases}
\end{equation}
Within each dataset, the two classes have a minimum of two informative features in common. For example, in Syn1$_{(+)}$, both class one and class two share $(x_3, x_{11})$ as the informative features. To introduce class imbalance, we intentionally generate 150 samples for class one and 50 samples for class two.

\looseness-1
Note that we purposely design Syn3$_{(-)}$ to examine the clustering assumption in ProtoGate by adding even function $|x_9|$ to Class One. The absolute value function is an even function. Two samples with opposite values of the same feature are likely to have equal logit values, and then they belong to the same class. However, the opposite values mean a long distance between them, and they should not belong to the same class according to the clustering assumption. Therefore, prototype-based models are expected to perform poorly in this regime. We implement it by adding absolute value function $|x_9|$ in the first class of Syn3$_{(-)}$ to observe the performance degradation in ProtoGate. Because of the evenness of absolute value, two samples with opposite values of $x_9$ are likely to be of Class One, which is against the clustering assumption. Although Syn1$_{(+)}$ and Syn2$_{(+)}$ also contain even functions like square and absolute value, they also have many other informative features that do not utilise the even functions to compute logit value. Therefore, the side effect of even functions is diluted in Syn1$_{(+)}$ and Syn2$_{(+)}$.

\paragraph{Differences to prior work.} Compared to previous studies~\cite{yang2022locally, yoon2018invase}, we focus on more realistic and challenging synthetic datasets by considering four key aspects. Firstly, we only generate 200 samples for each dataset, which is only 10\% of the samples in prior work~\cite{yang2022locally}. Secondly, each sample has 100 features, which is ten times more than that in prior work~\cite{yang2022locally}. Thirdly, our synthetic datasets are imbalanced, while those in prior work have balanced class distribution~\cite{yang2022locally}. Lastly, we incorporate a greater number of overlapping informative features between the two classes, while those in prior work may have no overlapping features~\cite{yang2022locally}.

\subsection{Data Preprocessing}
Following the methodology presented by~\citet{margeloiu2022weight}, we perform Z-score normalisation on each dataset before training the models. This normalisation process involves two steps. First, we compute the mean and standard deviation of each feature in the training data. Using these statistics, we transform the training samples to have a mean of zero and a variance of one for each feature. Subsequently, we apply the same transformation to the validation and test data before conducting evaluations.

\subsection{Computing Resources}
\label{appendix:compute}
We trained over 15,000 models (including over 3,000 of ProtoGate) for evaluations. All the experiments were conducted on a machine equipped with an NVIDIA A100 GPU with 40GB memory and an Intel(R) Xeon(R) CPU (at 2.20GHz) with six cores. The operating system used was Ubuntu 20.04.5 LTS.

\subsection{Training Details and hyperparameter Tuning}
\label{appendix:training}

\paragraph{Software implementation.}
We implemented ProtoGate with Pytorch Lightning~\cite{Falcon_PyTorch_Lightning_2019}: the global-to-local feature selector is implemented from scratch, and the relaxed sorting predictor is adapted from the official implementation of NeuralSort (\url{https://github.com/ermongroup/neuralsort}). Note that we further optimised the speed of the official implementation with matrix operators in PyTorch~\cite{paszke2019pytorch}. We re-implemented LSPIN/LLSPIN because the official implementation (\url{https://github.com/jcyang34/lspin}) used a different evaluation setup from ours: we report the mean $\pm$ std number of selected features, while they report the median number of selected features. We implemented XGBoost (\url{https://xgboost.readthedocs.io/en/stable/}), CatBoost (\url{https://catboost.ai/}) and LightGBM (\url{https://github.com/microsoft/LightGBM}) using their open-source implementations. With scikit-learn~\cite{pedregosa2011scikit}, we implemented Random Forest (\url{https://scikit-learn.org/stable/modules/generated/sklearn.ensemble.RandomForestClassifier}), KNN (\url{https://scikit-learn.org/stable/modules/generated/sklearn.neighbors.KNeighborsClassifier}) and Lasso (\url{https://scikit-learn.org/stable/modules/generated/sklearn.linear\_model.Lasso}). For other benchmark methods, we used their open-source implementations: STG (\url{https://github.com/runopti/stg}), TabNet (\url{https://github.com/dreamquark-ai/tabnet}), L2X (\url{https://github.com/Jianbo-Lab/L2X}), INVASE (\url{https://github.com/vanderschaarlab/mlforhealthlabpub/tree/main/alg/invase}) and REAL-X (\url{https://github.com/rajesh-lab/realx}).

We implemented a uniform pipeline using PyTorch Lightning to ensure consistency and reproducibility. We further fixed the random seeds for data loading and evaluation throughout the training and evaluation process. This ensured that ProtoGate and all benchmark models were trained and evaluated on the same set of samples. The experimental environment settings, including library dependencies, are specified in the associated code for reference and replication purposes.

Note that all the libraries utilised in this study adhere to open-source licenses. Specifically, the scikit-learn and the INVASE implementation follow the BSD-3-Clause license, Pytorch Lightning follows the Apache-2.0 license, and the others follow the MIT license.

\paragraph{Training procedures.}
In this section, we outline the key training settings for ProtoGate and all benchmark methods. The complete experimental settings can be found in the accompanying code. We made diligent efforts to ensure a fair comparison among the benchmark methods whenever possible. For example, we employed the same predictor architecture in LSPIN, MLP and STG, as these models share similar design principles.

\begin{itemize}[topsep=0pt, leftmargin=18pt]
    \item \textbf{ProtoGate} has a three-layer feature selector. The number of neurons in the hidden layer is 200 for real-world datasets and 100 for synthetic datasets. And the activation function is $tanh$ in all layers for a fair comparison against prior work~\cite{yang2022locally}. The model is trained for 10,000 iterations using early stopping with patience 500 on the validation loss. We used the suggested temperature parameter $\tau=16.0$ in NeuralSort~\cite{grover2018stochastic}. We train the models with a batch size of 64 and utilise an SGD optimiser with a weight decay of $1e-4$. 
    
    \item \textbf{TabNet} has a width of eight for the decision prediction layer and the attention embedding for each mask and 1.5 for the coefficient for feature reusage in the masks. The model is trained with Adam optimiser with momentum of 0.3 and gradient clipping at 2.
    
    \item \textbf{L2X, INVASE and REAL-X} have the default architecture as published~\cite{chen2018learning, yoon2018invase, jethani2021have}. The feature selector network has two hidden layers of $[100, 100]$, and the predictor network has two hidden layers of $[200, 200]$. They all use the $relu$ activation after layers. For convergence and computation efficiency, L2X is trained for 7,000 iterations, INVASE is trained for 5,000 epochs and REAL-X is trained for 1,000 iterations.

    \item \textbf{STG, LSPIN and LLSPIN} have a feature selector with the same architecture as that in ProtoGate. For LSPIN/STG, the predictor is a feed-forward neural network with hidden layers of $[100, 100, 10]$ with $tanh$ activation function. And we used the same architecture of predictor for \textbf{MLP}. For LLSPIN, the architecture of the predictor is the same, but the activation functions are removed. In other words, LLSPIN has multiple linear layers with no activations. This is proposed in LLSPIN’s original paper~\cite{yang2022locally} and the official open-source implementation (\url{https://github.com/jcyang34/lspin/blob/dev/Demo0.ipynb?short_path=69bb17b#L214}). That being said, we weren’t able to find discussion or reasons behind the implementation choices of the LLSPIN’s predictor with multiple linear layers. Nevertheless, we empirically tested and found that LLSPIN with multiple linear layers generally outperforms the variant with only one linear layer in the predictor (see \cref{tab:llspin_multi_linear_layers}). The standard deviation $\sigma$ for injected noise is $0.5$. The model is trained for 7,000 iterations using early stopping with patience 500 on the validation loss.

    \item \textbf{Ridge} is trained for 10,000 iterations to minimise the multinomial loss with Limited-memory BFGS solver~\cite{liu1989limited}, and the tolerance for early stopping is set as $1e-4$.

    \item \textbf{SVM} is trained for 10,000 iterations with the RBF kernel, and the tolerance for early stopping is set as $1e-3$.

    \item \textbf{KNN} measured the distance across samples with Euclidean distance and used uniform weights to compute the majority class in the neighbourhood.

    \item \textbf{Lasso} is trained for 10,000 iterations to minimise the weighted loss with SAGA solver~\cite{defazio2014saga}, and the tolerance for early stopping is set as $1e-4$.

    \item \textbf{Random Forest} has 500 estimators, feature bagging with the square root of the number of features, and used balanced weights from class distribution.

    \item \textbf{XGBoost} has 100 estimators. It is trained with a learning rate of $1.0$ to minimise the cross-entropy loss. And the $\ell_2$-regularization term on weights is set as $1e-5$.

    \item \textbf{CatBoost} has a maximum depth of 6, and it is trained with a learning rate of $0.03$ to minimise the cross-entropy loss. The $\ell_2$-regularization term of the cost function is $3.0$.
    
    \item \textbf{LightGBM} has 200 estimators, feature bagging with 30\% of the features, a minimum of two instances in a leaf. It is trained for 10,000 iterations to minimise the weighted cross-entropy loss using early stopping with patience 100 on validation loss.
    
\end{itemize}

\begin{table}[htbp]
\centering
\caption{\textbf{Classification accuracy (\%) of two variants of LLSPIN.} We \textbf{bold} the highest accuracy for each dataset. LLSPIN with multiple linear layers consistently achieves higher accuracy across all datasets.}
\label{tab:llspin_multi_linear_layers}
\begin{tabular}{lrrrrrrr}
\toprule
Methods                         & colon & lung  & meta-dr & meta-pam & prostate & tcga-2y & toxicity \\

\midrule

LLSPIN (single linear layer)    & 73.75          & 62.50         & 55.95            & 88.31             & 85.45             & 55.21            & 74.65             \\
LLSPIN (multiple linear layers) & \textbf{79.35} & \textbf{70.10} & \textbf{56.77} & \textbf{95.50} & \textbf{88.71} & \textbf{57.88} & \textbf{81.67}  \\

\bottomrule

\end{tabular}
\end{table}

\paragraph{Hyperparameter tuning.}
To ensure optimal performance, we initially identified a suitable range of hyperparameters for each model to facilitate convergence. Subsequently, we conducted a grid search within this predefined range to determine the optimal hyperparameter settings. The selection of models was based on their balanced accuracy on the validation sets averaged over 25 runs. It is worth noting that tuning hyperparameters in LSPIN can be challenging, particularly for real-world datasets. Therefore, we followed the recommendations in the original paper~\cite{yang2022locally} and employed Optuna~\cite{optuna_2019} to fine-tune the hyperparameters for LSPIN. 

Note that the LSPIN/LLSPIN performance on ``toxicity'' and ``colon'' datasets in \cref{tab:acc} is different from its original paper. The mismatched performance stems from the experimental settings. Specifically: the LSPIN paper~\cite{yang2022locally} mentions in paragraph 3 of its Appendix B.5.2 that they did \textbf{not} perform cross-validation on the ``colon'' and ``toxicity'' datasets. Namely, they did \textbf{not} use validation sets for model selection. In contrast, we have more realistic evaluation settings. Specifically, we \textbf{do} use validation sets for all datasets, leading to smaller train sets. Therefore, the classification accuracy of LSPIN/LLSPIN can be lower than those reported in the original paper.

\cref{tab:param_protogate} lists the searching ranges of hyperparameters in ProtoGate, and \cref{tab:param_fs} lists the searching ranges of hyperparameters in network-based feature selection benchmark methods. 

\begin{table*}[htbp]
\centering
\caption{Searching ranges of hyperparameters in ProtoGate.}
\label{tab:param_protogate}
\resizebox{\textwidth}{!}{
\begin{tabular}{lrrrr}
\toprule
Datasets  & Global Sparsity $\lambda_{\text{global}}$              & Local Sparsity $\lambda_{\text{local}}$                        & $K$                 & Learning Rate $\alpha$   \\ \midrule
Real-word & $\{1e-4, 2e-4, 3e-4, 4e-4, 6e-4\}$               & $\{1e-3\}$ & $\{1, 2, 3, 4, 5\}$ & $\{5e-2, 7.5e-2, 1e-1\}$ \\
Synthetic & $\{1e-2, 1.5e-2, 2e-2\}$ & $\{0, 1e-4, 3e-4\}$                & $\{3\}$             & $\{1e-1\}$               \\
\bottomrule
\end{tabular}
}
\end{table*}

\begin{table*}[htbp]
\centering
\caption{\textbf{Searching ranges of hyperparameters in network-based feature selection benchmark methods.} Note that the range for LSPIN/LLSPIN on real-world datasets is an interval instead of a set because we used Optuna to search for the optimal settings.}
\label{tab:param_fs}
\begin{tabular}{llrr}
\toprule
Datasets                       & Methods   & Sparsity regularisation strength ($\lambda$)       & Learning Rate          \\ \midrule
\multirow{6}{*}{Real-world} & STG          & $\{35, 40, 45, 50, 55\}$     & $\{3e-3\}$             \\
                            & TabNet       & $\{1e-4, 1e-3, 1e-2, 1e-1\}$ & $\{1e-2, 2e-2, 3e-2\}$ \\
                            & L2X          & $\{1, 5, 10\}$               & $\{1e-4\}$             \\
                            & INVASE       & $\{1, 1.5, 2\}$              & $\{1e-4\}$             \\
                            & REAL-X       & $\{1, 5, 10, 30, 50\}$       & $\{1e-4\}$             \\
                            & LSPIN/LLSPIN & $[5e-4, 1.5e-3]$             & $[5e-2, 1e-1]$         \\ \midrule
\multirow{6}{*}{Synthetic}  & STG          & $\{1, 3, 5\}$                & $\{1e-1\}$             \\
                            & TabNet       & $\{1e-2, 1e-1, 5e-1\}$       & $\{1e-2\}$             \\
                            & L2X          & $\{1, 5, 10\}$               & $\{1e-4\}$             \\
                            & INVASE       & $\{1, 1.5, 2\}$              & $\{1e-4\}$             \\
                            & REAL-X       & $\{1, 5, 10, 30, 50\}$       & $\{1e-4\}$             \\
                            & LSPIN/LLSPIN & $\{1e-2, 5e-2, 1e-1\}$       & $\{1e-1\}$             \\ 
\bottomrule
\end{tabular}
\end{table*}

In line with prior studies~\cite{margeloiu2022weight, yoon2018invase, yang2022locally}, we performed hyperparameter searching for other methods within the same ranges for real-world and synthetic datasets. For \textbf{Ridge}, we performed a grid search of the regularisation strength in $\{1, 1e1, 1e2, 1e3\}$. For \textbf{SVM}, we performed a grid search of the regularisation strength in $\{1e-3, 1e-2, 1e-1, 1, 1e1\}$. For \textbf{KNN}, we performed a grid search of the number of nearest neighbours in $\{1, 3, 5\}$. For \textbf{MLP}, we used Optuna to find the optimal learning rate within $[1e-3, 1e-1]$. For \textbf{Lasso}, we performed a grid search of the regularisation strength in $\{1, 1e1, 1e2, 1e3\}$. For \textbf{Random Forest}, we performed a grid search for the maximum depth in $\{3, 5, 7\}$ and the minimum number of instances in a leaf in $\{2, 3\}$. For \textbf{LightGBM}, we performed a grid search for the learning rate in $\{1e-2, 1e-1\}$ and maximum depth in $\{1, 2\}$. Please refer to the associated code for the full details of the hyperparameter settings and their corresponding ranges.

\clearpage
\section{Additional Results on Real-world Prediction Tasks}

\subsection{Results on Feature Selection}
\label{appendix:complete_sparsity}

\subsubsection{Numerical Results for Feature Selection Sparsity}
As a supplement to \cref{fig:trade_offs} (Left), we provide detailed numerical results of the feature selection sparsity here. In \cref{tab:sparisity_num}, we show the number of selected features per sample. For global methods, the number of selected features is the same across samples, and thus the standard deviation is zero. In \cref{fig:prop}, we further visualise the proportion of selected features among all features. ProtoGate generally selects fewer features than benchmark feature selection methods, indicating improved interpretability in the selected features.

\begin{table*}[!htbp]
\centering
\caption{\textbf{Feature selection sparsity on real-world datasets.} We report the mean $\pm$ std of the number of selected features on test samples, averaged over 25 runs. ProtoGate generally selects fewer features than other feature selection methods.}
\label{tab:sparisity_num}
\resizebox{\textwidth}{!}{
\begin{tabular}{lrrrrrrr}
\toprule
Methods      & colon                        & lung                         & meta-dr                       & meta-pam                      & prostate                      & tcga-2y                       & toxicity                      \\
\midrule
Lasso        & 371.40$_{\pm\text{0.00}}$    & 1618.08$_{\pm\text{0.00}}$   & 4159.92$_{\pm\text{0.00}}$    & 4159.40$_{\pm\text{0.00}}$    & 5434.68$_{\pm\text{0.00}}$    & 4214.56$_{\pm\text{0.00}}$    & 2951.28$_{\pm\text{0.00}}$    \\
RF           & 1629.72$_{\pm\text{0.00}}$   & 504.76$_{\pm\text{0.00}}$    & 577.60$_{\pm\text{0.00}}$     & 1439.20$_{\pm\text{0.00}}$    & 510.72$_{\pm\text{0.00}}$     & 887.12$_{\pm\text{0.00}}$     & 1507.44$_{\pm\text{0.00}}$    \\
LightGBM     & 70.12$_{\pm\text{0.00}}$     & 82.44$_{\pm\text{0.00}}$     & 29.64$_{\pm\text{0.00}}$      & 336.25$_{\pm\text{0.00}}$     & 117.58$_{\pm\text{0.00}}$     & 31.88$_{\pm\text{0.00}}$      & 1150.48$_{\pm\text{0.00}}$    \\
STG          & 2000.00$_{\pm\text{0.00}}$   & 3312.00$_{\pm\text{0.00}}$   & 4157.96$_{\pm\text{0.00}}$    & 2992.00$_{\pm\text{0.00}}$    & 5966.00$_{\pm\text{0.00}}$    & 4381.00$_{\pm\text{0.00}}$    & 5748.00$_{\pm\text{0.00}}$    \\
L2X          & 5.00$_{\pm\text{0.00}}$      & 1.00$_{\pm\text{0.00}}$      & 5.00$_{\pm\text{0.00}}$       & 10.00$_{\pm\text{0.00}}$      & 10.00$_{\pm\text{0.00}}$      & 5.00$_{\pm\text{0.00}}$       & 5.00$_{\pm\text{0.00}}$       \\
LSPIN        & 1044.32$_{\pm\text{293.67}}$ & 564.83$_{\pm\text{1236.21}}$ & 1138.51$_{\pm\text{1545.96}}$ & 1073.04$_{\pm\text{1661.89}}$ & 2120.00$_{\pm\text{1968.86}}$ & 1418.35$_{\pm\text{1936.41}}$ & 1979.29$_{\pm\text{2387.03}}$ \\
LLSPIN       & 1311.86$_{\pm\text{209.80}}$ & 673.27$_{\pm\text{1212.20}}$ & 3026.77$_{\pm\text{642.02}}$  & 1180.08$_{\pm\text{1769.59}}$ & 2151.05$_{\pm\text{1954.80}}$ & 3486.77$_{\pm\text{696.29}}$  & 1999.12$_{\pm\text{2398.65}}$ \\
\midrule
ProtoGate    & 110.70$_{\pm\text{58.62}}$   & 177.24$_{\pm\text{173.13}}$  & 337.21$_{\pm\text{738.81}}$   & 469.47$_{\pm\text{46.90}}$    & 91.29$_{\pm\text{7.20}}$      & 348.79$_{\pm\text{869.23}}$   & 76.39$_{\pm\text{17.42}}$     \\ 
\bottomrule
\end{tabular}
}
\end{table*}

\begin{figure*}[!htbp]
    \centering
    \includegraphics[width=0.95\columnwidth]{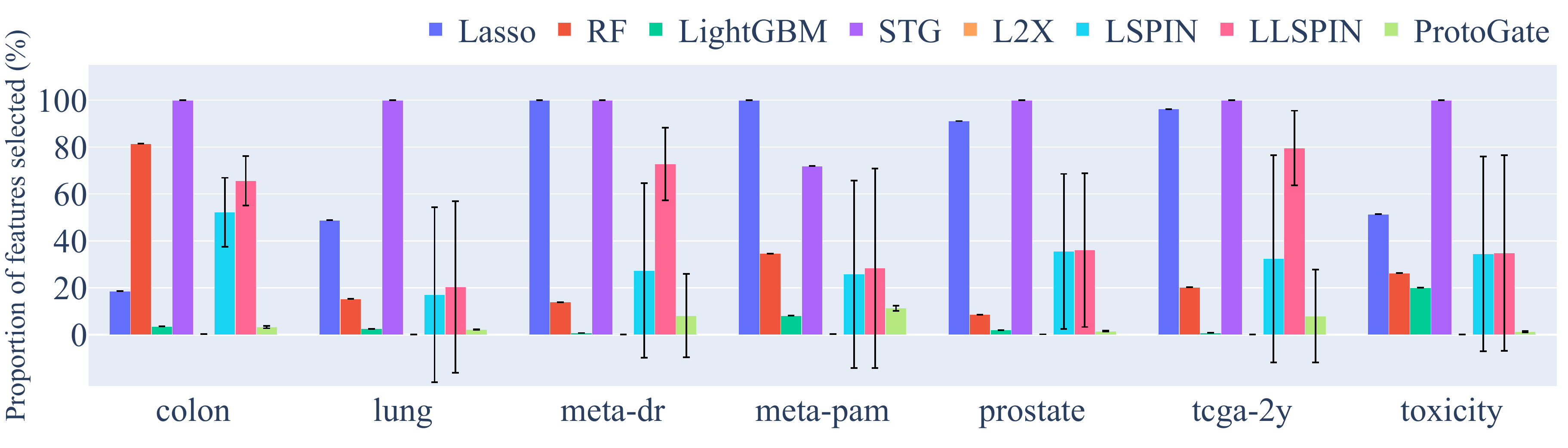}
    \caption{\textbf{Visualisation of the feature selection sparsity on real-world datasets.} We report the mean $\pm$ std of the proportion of selected features on test samples, averaged over 25 runs. ProtoGate generally selects fewer features than other local feature selection methods.}
    \label{fig:prop}
\end{figure*}

\subsubsection{Visualisation of Selected Features}
\label{appendix:vis_heatmap}
\cref{fig:gate_val} qualitatively shows that ProtoGate generally has smaller mask values than LSPIN and LLSPIN, denoting ProtoGate selects fewer features. Furthermore, we can see different feature selection results (i.e., $\vs_{\text{local}}$) in ProtoGate's heatmaps (e.g., the ``meta-dr'' dataset), showing that ProtoGate can select features locally. Additionally, ProtoGate can behave more globally on other datasets (e.g., the ``lung'' dataset). The results suggest that ProtoGate can adaptively balance global and local feature selection across different datasets, showing the effect of global-to-local feature selection.

\begin{figure*}[!htbp]
    \centering
    \subfloat[Heatmaps of mask values from LSPIN]
    {\includegraphics[width=0.95\columnwidth]{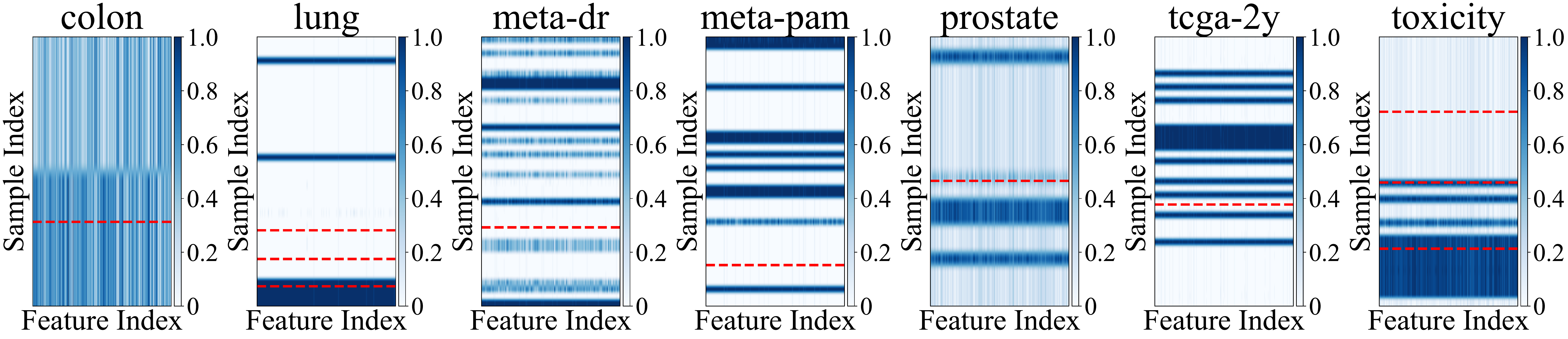}}\\
    \subfloat[Heatmaps of mask values from LLSPIN]
    {\includegraphics[width=0.95\columnwidth]{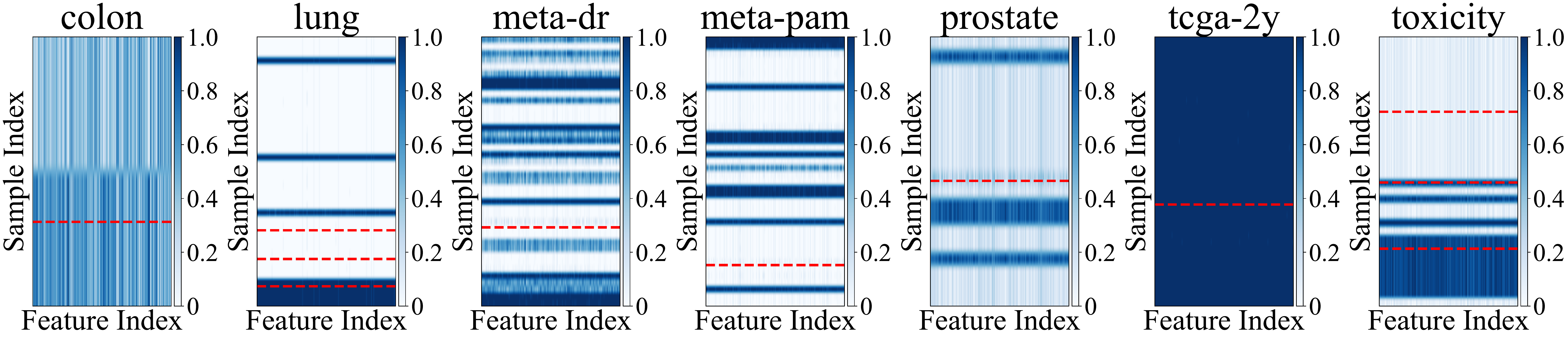}}\\
    \subfloat[Heatmaps of mask values from ProtoGate]
    {\includegraphics[width=0.95\columnwidth]{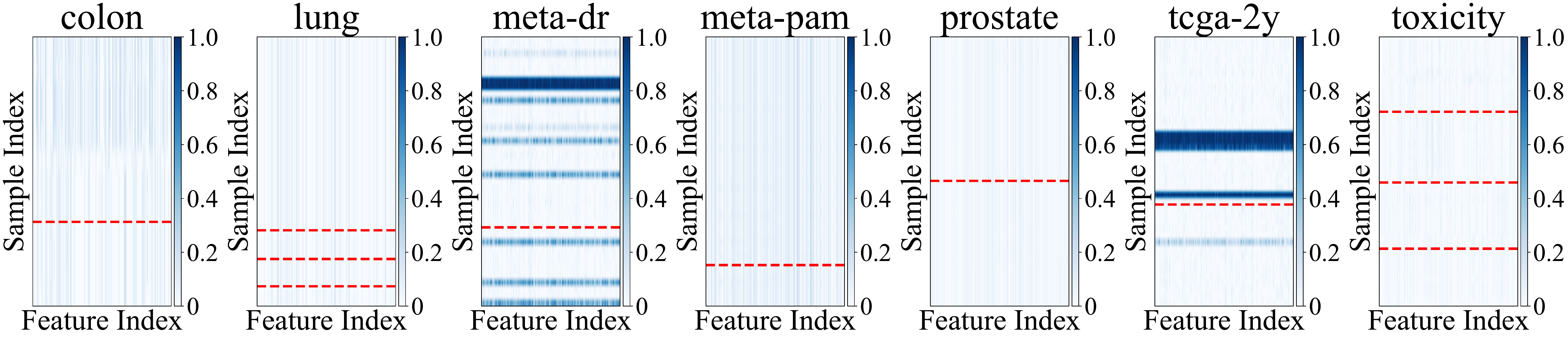}}
    \caption{\textbf{Visualisaton of the mask values on real-world datasets.} We plot the heatmaps of predicted mask values $\vs_{\text{local}}^{(i)}$ of test samples, where the x-axis refers to the indices of features, and the y-axis refers to the indices of samples (this is different to \cref{fig:fs_illustration}). Note that different datasets can have different numbers of samples and features, and the number of features should be more than the number of samples in the HDLSS regime. For visualisation purposes, we align them by adjusting the aspect ratio of heatmaps. The samples are sorted according to their ground truth labels, and the red dash lines separate samples of different classes.}
    \label{fig:gate_val}
\end{figure*}

\newpage
\subsubsection{Composition of Selected Features}
\label{appendix:feat_composition}
We further analyse the composition of the final feature selection results (i.e., the proportions of ``both selected'' and ``locally recovered'' features in $\textbf{s} _\text{local}^{(i)}$). \cref{tab:fs_behaviour} shows that the majority proportion of $\textbf{s} _\text{local}^{(i)}$ is generally included in $\textbf{s} _\text{global}$ on most datasets (i.e., the proportion of ``both selected'' features is generally greater than ``locally recovered'' features). For instance, on the ``tcga-2y'' dataset, 96.61\% of $\textbf{s} _\text{local}^{(i)}$ is included in $\textbf{s} _\text{global}$.
\cref{tab:fs_behaviour} also shows that ProtoGate can adaptively adjust its selection behaviours for different datasets. On some datasets, such as the ``toxicity'' dataset, ProtoGate locally recovers more features than reusing the ones from global selection.
The above results show that ProtoGate always includes some globally selected features in the subsequent local selection results. Therefore, we introduce ProtoGate by ``refining the global mask into the local mask''. In other words, the global mask $\textbf{s} _\text{global}$ empirically represents a lower-dimensional feature set to effectively determine the local mask $\textbf{s} _\text{local}^{(i)}$.

\begin{table}[htbp]
    \centering
    \caption{\textbf{Proportions (\%) of ``both selected'' and ``locally recovered'' features in ProtoGate’s local selection results (i.e., $\textbf{s}_\text{local}^{(i)}$).} ProtoGate adaptively adjust its selection behaviours to be more global or more local for different datasets.}
    \label{tab:fs_behaviour}
    \begin{tabular}{llllllll}
    \toprule
    \textbf{Selection behaviours} & \textbf{colon} & \textbf{lung} & \textbf{meta-dr} & \textbf{meta-pam} & \textbf{prostate} & \textbf{tcga-2y} & \textbf{toxicity} \\

    \midrule
    
    Both selected                 & 56.99          & 79.54         & 33.33            & 69.64             & 84.56             & 96.61            & 10.80             \\
    Locally recovered             & 43.01          & 20.46         & 66.67            & 30.36             & 15.44             & 3.39             & 89.20  \\         
    \bottomrule
    \end{tabular}%
\end{table}

\subsection{Results on Computation Efficiency}
\label{appendix:efficiency}
As a supplement to \cref{fig:trade_offs} (Middle), we further provide detailed numerical results on the computation efficiency. \cref{tab:time_train} shows that ProtoGate’s training time per epoch is consistently shorter than average. Moreover, the training time differences per epoch between ProtoGate and the fastest methods are consistently smaller than 0.1s across all datasets, which is practically insignificant.  \cref{tab:time_test} further shows that ProtoGate can efficiently make predictions for new samples. The inference time per sample of ProtoGate is shorter than 1ms, and the time differences between ProtoGate and the fastest methods are consistently smaller than 0.5ms across all datasets. 

\begin{table}[!htbp]
\vspace{-5mm}
    \centering
    \caption{\textbf{Training time per epoch (unit: second).} We \textbf{bold} the average training time and ProtoGate's training time per batch. ProtoGate’s training time per epoch is consistently shorter than average, and the time differences per epoch between ProtoGate and the fastest methods are consistently smaller than 0.1s across all datasets.}
    \label{tab:time_train}
    \begin{tabular}{lrrrrrrr}
    \toprule
    Methods                     & colon          & lung          & meta-dr          & meta-pam          & prostate          & tcga-2y          & toxicity          \\ \midrule
    TabNet                      & 0.03           & 0.08          & 0.11             & 0.11              & 0.06              & 0.11             & 0.13              \\
    L2X                         & 0.03           & 0.16          & 0.19             & 0.18              & 0.08              & 0.19             & 0.16              \\
    INVASE                      & 0.06           & 0.10          & 0.11             & 0.11              & 0.12              & 0.11             & 0.12              \\
    REAL-X                      & 0.03           & 0.30          & 0.37             & 0.40              & 0.19              & 0.40             & 0.33              \\
    LLSPIN                      & 0.02           & 0.05          & 0.06             & 0.06              & 0.04              & 0.05             & 0.06              \\
    LSPIN                       & 0.02           & 0.05          & 0.05             & 0.05              & 0.04              & 0.05             & 0.06              \\
    \textbf{Avg. w/o ProtoGate} & \textbf{0.03}  & \textbf{0.12} & \textbf{0.15}    & \textbf{0.15}     & \textbf{0.09}     & \textbf{0.15}    & \textbf{0.14}     \\ \midrule
    \textbf{ProtoGate (Ours)}          & \textbf{0.02}  & \textbf{0.10} & \textbf{0.12}    & \textbf{0.12}     & \textbf{0.08}     & \textbf{0.12}    & \textbf{0.13}     \\ \bottomrule
    \end{tabular}
\vspace{-5mm}
\end{table}

\begin{table}[!htbp]
\vspace{-3mm}
    \centering
    \caption{\textbf{Inference time per sample (unit: millisecond).} We \textbf{bold} the average inference time and ProtoGate's inference time per sample. The inference time per sample of ProtoGate is shorter than $1$ms, and the time differences per sample between ProtoGate and the fastest methods are consistently smaller than $0.5$ms across all datasets.}
    \label{tab:time_test}
    \begin{tabular}{lrrrrrrr}
    \toprule
    Methods                     & colon          & lung          & meta-dr          & meta-pam          & prostate          & tcga-2y          & toxicity          \\ \midrule
    TabNet                      & 0.53           & 0.39          & 0.54             & 0.55              & 0.57              & 0.53             & 0.75              \\
    L2X                         & 0.41           & 0.81          & 0.93             & 0.90              & 0.77              & 0.93             & 0.91              \\
    INVASE                      & 1.04           & 0.51          & 0.54             & 0.54              & 1.16              & 0.55             & 0.69              \\
    REAL-X                      & 0.41           & 1.55          & 1.86             & 1.98              & 1.88              & 2.00             & 1.95              \\
    LLSPIN                      & 0.31           & 0.26          & 0.30             & 0.28              & 0.35              & 0.27             & 0.34              \\
    LSPIN                       & 0.31           & 0.27          & 0.27             & 0.27              & 0.37              & 0.27             & 0.34              \\
    \textbf{Avg. w/o ProtoGate} & \textbf{0.50}  & \textbf{0.63} & \textbf{0.74}    & \textbf{0.75}     & \textbf{0.85}     & \textbf{0.76}    & \textbf{0.83}     \\ \midrule
    \textbf{ProtoGate (Ours)}          & \textbf{0.17}  & \textbf{0.19} & \textbf{0.28}    & \textbf{0.28}     & \textbf{0.61}     & \textbf{0.44}    & \textbf{0.59}     \\ \bottomrule
    \end{tabular}
\vspace{-3mm}
\end{table}

\cref{tab:param} further shows that ProtoGate can have much fewer trainable parameters than benchmark methods. Note that we focus on the model and exclude the shared training settings such as batch size, training epochs and learning rate when counting the hyperparameters. Because the number of trainable parameters depends on the input dimensionality, we compute the total number of trainable parameters according to the ``prostate'' dataset (102 samples with 5,966 features). Although the number of trainable parameters can change across datasets, the order remains the same (INVASE$>$L2X$>$REAL-X$>$LSPIN$>$LLSPIN$>$ProtoGate$>$TabNet) in our experimental settings.
    
\begin{table}[!htbp]
\vspace{-5mm}
\centering
\caption{\textbf{Number of hyperparameters and trainable parameters in ProtoGate and other local methods.} We \textbf{bold} the number of parameters for ProtoGate. Benchmark methods are sorted according to the number of trainable parameters. ProtoGate can have much fewer trainable parameters than other local methods in the considered experimental settings.}
\label{tab:param}
\begin{tabular}{lrrrrrrr}
\toprule
Methods            & \# hyperparameters& \# Trainable Parameters \\    \midrule
INVASE             & 2                  & 3.7M                   \\
L2X                & 3                  & 2.4M                   \\
REAL-X             & 2                  & 2.4M                   \\
LSPIN              & 2                  & 1.8M                   \\
LLSPIN             & 2                  & 1.8M                   \\
TabNet             & 3                  & 0.4M                   \\ 
\midrule
\textbf{ProtoGate (Ours)} & \textbf{3}         & \textbf{1.2M}          \\
\bottomrule
\end{tabular}
\vspace{-5mm}
\end{table}

\subsection{Results on non-HDLSS Classification Tasks}
\label{appendix:non_HDLSS}
\cref{tab:non_hdlss_dataset} shows that ProtoGate consistently ranks within the top three for all datasets, implying robust performance. Although ProtoGate is primarily designed for HDLSS regimes, we find it can also perform quite well in other challenging tasks, even tied in average rank with LightGBM. Specifically, ProtoGate ranks first on the ``100-plants-texture'' dataset, surpassing other local methods by a clear margin. Note that REAL-X even failed to converge well on the ``100-plants-texture'' dataset, highlighting its limitations in managing datasets with a high number of classes. Moreover, ProtoGate can achieve competitive performance on non-biomedical datasets, such as the ``mfeat-fourier'' dataset. These findings indicate that ProtoGate's applicability may well extend beyond its original scope in HDLSS biomedical applications.

\begin{table}[!htbp]
\centering
\caption{\textbf{Classification accuracy (\%) on four non-HDLSS real-world tabular datasets.} We report the mean $\pm$ std balanced accuracy and average accuracy rank across datasets. A higher rank implies higher accuracy. We highlight the {\color[HTML]{008080} \textbf{First}}, {\color[HTML]{7030A0} \textbf{Second}} and {\color[HTML]{C65911} \textbf{Third}} ranking accuracy for each dataset. ProtoGate consistently ranks Top-3 across all datasets and achieves comparable overall performance as the state-of-the-art methods.}
\label{tab:exp_acc_non_HDLSS}
\begin{tabular}{lrrrr|r}
\toprule
Methods & 100-plants-texture & cnae-9 & mfeat-fourier & vehicle & \textbf{Rank} \\
\midrule
MLP & 19.10$_{\pm\text{2.03}}$ & 64.52$_{\pm\text{3.10}}$ & 79.76$_{\pm\text{1.89}}$ & {\color[HTML]{008080} \textbf{77.87}}$_{\pm\text{3.41}}$ & 4.50$_{\pm\text{2.29}}$ \\
Lasso & 52.89$_{\pm\text{2.46}}$ & 74.92$_{\pm\text{3.67}}$ & 76.86$_{\pm\text{1.54}}$ & 70.62$_{\pm\text{2.68}}$ & 5.00$_{\pm\text{0.71}}$ \\
RF & {\color[HTML]{C65911} \textbf{64.81}}$_{\pm\text{2.34}}$ & {\color[HTML]{C65911} \textbf{87.30}}$_{\pm\text{1.79}}$ & {\color[HTML]{C65911} \textbf{80.31}}$_{\pm\text{1.89}}$ & 72.86$_{\pm\text{2.46}}$ & {\color[HTML]{C65911} \textbf{3.25}}$_{\pm\text{0.43}}$ \\
LightGBM & {\color[HTML]{7030A0} \textbf{75.73}}$_{\pm\text{2.34}}$ & {\color[HTML]{008080} \textbf{92.96}}$_{\pm\text{1.45}}$ & {\color[HTML]{7030A0} \textbf{81.96}}$_{\pm\text{1.76}}$ & {\color[HTML]{7030A0} \textbf{75.87}}$_{\pm\text{2.92}}$ & {\color[HTML]{008080} \textbf{1.75}}$_{\pm\text{0.43}}$ \\
REAL-X & 9.89$_{\pm\text{1.33}}$ & 70.81$_{\pm\text{3.00}}$ & 63.61$_{\pm\text{1.75}}$ & 46.81$_{\pm\text{1.99}}$ & 6.75$_{\pm\text{0.43}}$ \\
LSPIN & 49.33$_{\pm\text{1.84}}$ & 79.54$_{\pm\text{2.47}}$ & 77.49$_{\pm\text{1.76}}$ & 68.98$_{\pm\text{2.12}}$ & 5.00$_{\pm\text{0.71}}$ \\
\midrule
\textbf{ProtoGate (Ours)} & {\color[HTML]{008080} \textbf{77.61}}$_{\pm\text{1.42}}$ & {\color[HTML]{7030A0} \textbf{87.80}}$_{\pm\text{2.59}}$ & {\color[HTML]{008080} \textbf{82.84}}$_{\pm\text{1.60}}$ & {\color[HTML]{C65911} \textbf{73.63}}$_{\pm\text{2.63}}$ & {\color[HTML]{008080} \textbf{1.75}}$_{\pm\text{0.83}}$ \\
\bottomrule
\end{tabular}
\end{table}

\clearpage
\section{Additional Ablation Results on Global-to-local Feature Selection}
\label{appendix:ablation_fs}

\subsection{Results on Features Selection Sparsity}
\label{appendix:ablation_reg_sparsity}
We perform ablation experiments to illustrate the interplay between $\lambda_{\text{global}}$ and $\lambda_{\text{local}}$. Note that we summarise the combinations into three cases: ProtoGate (both regularisations: $\lambda_{\text{global}} \neq 0, \lambda_{\text{local}} \neq 0$), ProtoGate-global (only $\ell_1$-regularisation: $\lambda_{\text{global}} \neq 0, \lambda_{\text{local}} = 0$) and ProtoGate-local (only $\ell_0$-regularisation: $\lambda_{\text{global}} = 0, \lambda_{\text{local}} \neq 0$). Specifically, we perform grid search within ranges of $\lambda_{\text{global}} \in \{0, 1e-4, 3e-4, 5e-4, 7e-4, 1e-3\}$ and $\lambda_{\text{local}} \in \{0, 5e-4, 1e-3, 2e-3\}$. 

\cref{tab:ablation_reg_sparsity} shows that ProtoGate and ProtoGate-local generally selects fewer features than ProtoGate-global. Note that we tried manually increasing $\lambda_{\text{global}}$ in ProtoGate-global to attain a sparsity level similar to ProtoGate, but it substantially degrades its performance. Note that although when ProtoGate and ProtoGate-local achieve similar sparsity, the accuracy gap remains. This demonstrates that ProtoGate selects different features to ProtoGate-local, showing the effects of soft global selection.

\begin{table}[htbp]
\centering
\caption{\textbf{Number of selected features of ProtoGate and its two variants.} We report the mean $\pm$ std of the number of selected features on test samples, averaged over 25 runs. We \textbf{bold} the fewest selected features for each dataset.}
\label{tab:ablation_reg_sparsity}
\resizebox{\textwidth}{!}{
\begin{tabular}{lrrrrrrr}
\toprule
Cases                                                     & colon   & lung                      & meta-dr                     & meta-pam                    & prostate                  & tcga-2y                     & toxicity \\ \midrule
ProtoGate-global & 846.48$_{\pm\text{95.32}}$ & 1673.92$_{\pm\text{250.13}}$ & 2089.16$_{\pm\text{265.44}}$ & 2601.15$_{\pm\text{170.17}}$ & 2542.48$_{\pm\text{108.79}}$ & 1667.31$_{\pm\text{84.74}}$  & 2367.68$_{\pm\text{53.67}}$ \\

ProtoGate-local & 115.96$_{\pm\text{16.65}}$ & \textbf{82.18}$_{\pm\text{14.53}}$  & \textbf{85.22}$_{\pm\text{12.57}}$ & \textbf{69.85}$_{\pm\text{23.20}}$  & 291.13$_{\pm\text{58.48}}$  & \textbf{131.43}$_{\pm\text{15.52}}$ & 154.00$_{\pm\text{25.11}}$ \\

\textbf{ProtoGate} & \textbf{110.70}$_{\pm\text{58.62}}$ & 177.24$_{\pm\text{173.13}}$ & 337.21$_{\pm\text{738.81}}$ & 469.47$_{\pm\text{46.90}}$ & \textbf{91.29}$_{\pm\text{7.20}}$ & 348.79$_{\pm\text{869.23}}$ & \textbf{76.39}$_{\pm\text{17.42}}$ \\ 
\bottomrule
\end{tabular}
}
\end{table}

\subsection{Results on Degree of Local Selection}
\label{appendix:lambda_t_test}
To further distinguish between ``similar number of selected features'' and ``similar selected features'', \textbf{we introduce a new metric: degree of local sparsity $\mathcal{Q}$}, which is computed by:
\begin{equation}
    \mathcal{Q} = \frac{1}{D \cdot N} \sum_{j=1}^{N} \text{card} \left(\bigcup_{i=1}^{N}\text{nonzero}(\vs_{\text{local}}^{(i)}) - \text{nonzero}(\vs_{\text{local}}^{(j)}) \right)
\end{equation}
where $\text{card}(\cdot)$ returns the cardinality of a set and $\text{nonzero}(\cdot)$ returns the indices of non-zero elements in a vector. $\mathcal{Q}$ measures the difference between the union set of selected features for all samples and the selected features for a specific sample. Intuitively, a smaller $\mathcal{Q}$ denotes the method selects features more globally than a bigger $\mathcal{Q}$. We perform ablation experiments with $\lambda_{\text{local}}=1e-3$ and $\lambda_{\text{global}} \in \{0, 1e-4, 3e-4, 5e-4, 7e-4, 1e-3\}$, and report the average $\mathcal{Q}$ and average balanced accuracy over 25 runs. 

\vspace{-5mm}
\begin{figure*}[!htbp]
    \centering
    \subfloat[Degree of local sparsity under different $\lambda_{\text{global}}$]{\includegraphics[width=0.43\columnwidth]{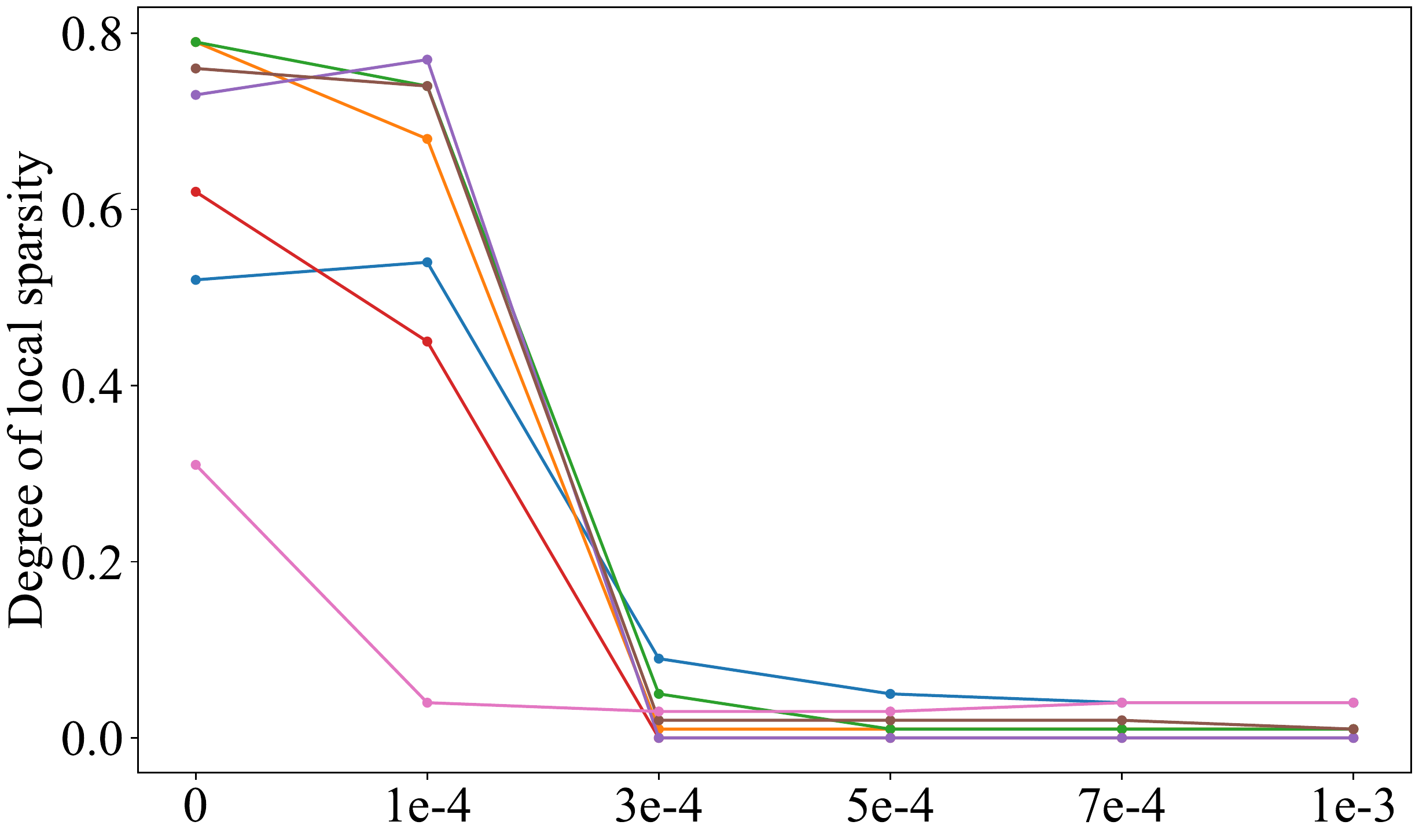}}
    \hspace{5mm}
    \subfloat[Balanced accuracy under under different $\lambda_{\text{global}}$]{\includegraphics[width=0.439\columnwidth]{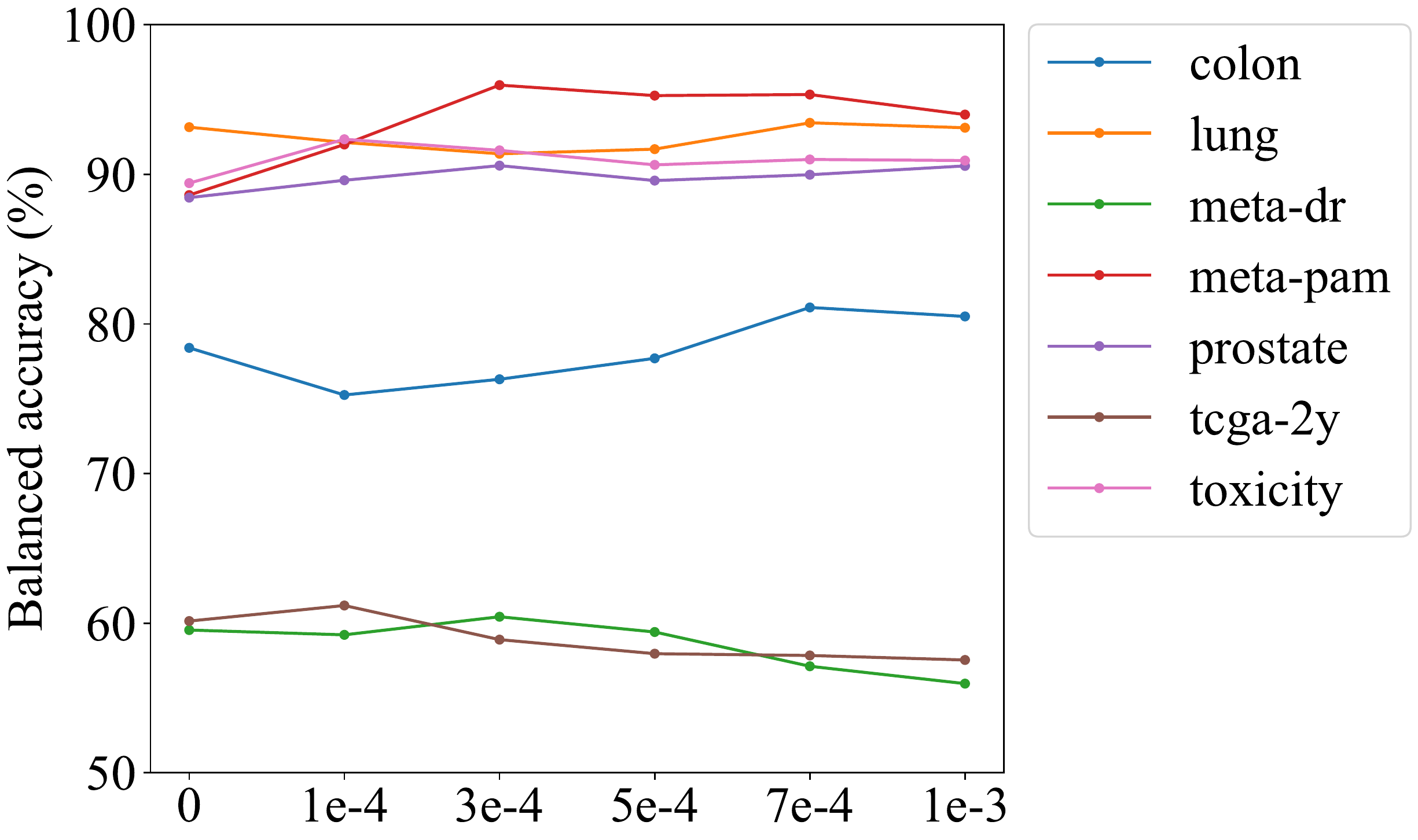}}
    \caption{\textbf{Comparison of different values of global sparsity hyperparameter $\lambda_{\text{global}}$.} (a) Degree of local sparsity averaged over 25 runs. Increasing $\lambda_{\text{global}}$ reduces the diversity of selected features across samples. (b) Balanced accuracy (\%) averaged over 25 runs. Increasing $\lambda_{\text{global}}$ does not guarantee improvement in the prediction accuracy. Note that confidence intervals are omitted for visualisation purposes.}
    \label{fig:l1_ablation}
\end{figure*}

\begin{table}[!htbp]
\vspace{-5mm}
\centering
\caption{Statistical analysis of performances between ProtoGate with optimal $\lambda_{\text{global}}^{*}$ and $\lambda_{\text{global}}=1e-3$.}
\label{tab:lambda_t_test}
\begin{tabular}{lrrrrrrrr}
\toprule
\textbf{} & lung   & meta-dr & meta-pam & prostate & tcga-2y & toxicity & colon  & Wilcoxon test \\ \midrule
P-values  & 0.8856 & 0.0112  & 0.1638   & 0.9919   & 0.1272  & 0.3517   & 0.8252 & 0.0156 \\ \bottomrule       
\end{tabular}
\vspace{-5mm}
\end{table}

\cref{fig:l1_ablation} shows that increasing $\lambda_{\text{global}}$ enables ProtoGate to select features more globally. Namely, increasing $\lambda_{\text{global}}$ can promote the existence of global important features by performing soft global selection via $\ell_1$-regularisation on $\mathbf{W}^{[1]}$. We performed statistical analysis between test accuracy with optimal $\lambda_{\text{global}}^*$ and maximum $\lambda_{\text{global}}=1e-3$. Specifically, we performed a \textit{corrected} 2-tailed t-test as well as Wilcoxon signed-rank test of the mean accuracies~\cite{nadeau1999inference, bouckaert2004evaluating}. We find that while models with optimal $\lambda_{\text{global}}^*$ generally tend to perform better than their counterparts, this difference is statistically significant (at $\alpha=0.05$) \textbf{only} in the case of ``meta-dr'' dataset, which highlights the benefits of leveraging the locally important features across samples. On the other datasets, leveraging the locally important features is less effective, which helps to explain why global methods (e.g., Lasso) can generally outperform some local methods in \cref{tab:acc}. The results show that ProtoGate can easily adapt and effectively balance global and local feature selection for a particular dataset, rather than exclusively regularising for either of them.

\clearpage
\section{Additional Ablation Results on Non-parametric Prototype-based Prediction}
\label{appendix:ablation_pred}

\vspace{2mm}
\subsection{Results on Different Predictors}
To show the efficacy of clustering assumption in prediction tasks, we replace the differentiable KNN predictor with (i) a linear head and (ii) an MLP, and then tune the hyperparameter for global sparsity $\lambda_{\text{global}}$ by searching within $\{1e-4, 2e-4, 3e-4\}$ for fair comparison between different predictors. \cref{tab:ablation_predictor} shows that differentiable KNN consistently outperforms other predictors across datasets, suggesting that the clustering assumption is beneficial for feature selection.

\begin{table*}[htbp]
\centering
\caption{\textbf{Classification accuracy (\%) for different predictors on real-world datasets.} We \textbf{bold} the highest accuracy for each dataset. The prototype-based classifier consistently outperforms linear and MLP predictors on all datasets.}
\label{tab:ablation_predictor}
\resizebox{\textwidth}{!}{
\begin{tabular}{lrrrrrrrrr}
\toprule
Predictors                & colon                    & lung                      & meta-dr                  & meta-pam                 & prostate                 & tcga-2y                  & toxicity                 \\ 
\midrule

MLP                       & 80.95$_{\pm\text{7.77}}$ & 69.97$_{\pm\text{9.17}}$  & 56.00$_{\pm\text{6.37}}$ & 93.62$_{\pm\text{6.04}}$ & 89.13$_{\pm\text{6.36}}$ & 54.74$_{\pm\text{8.11}}$ & 90.36$_{\pm\text{5.61}}$ \\

Linear Head               & 79.45$_{\pm\text{6.23}}$ & 66.51$_{\pm\text{12.45}}$ & 56.10$_{\pm\text{8.95}}$ & 93.20$_{\pm\text{6.18}}$ & 89.87$_{\pm\text{5.80}}$ & 56.60$_{\pm\text{8.20}}$ & 90.29$_{\pm\text{5.93}}$ \\

\textbf{Differentiable KNN (Ours)} & \textbf{83.95}$_{\pm\text{9.82}}$ & \textbf{93.56}$_{\pm\text{6.29}}$  & \textbf{60.43}$_{\pm\text{7.62}}$ & \textbf{95.96}$_{\pm\text{3.93}}$ & \textbf{90.58}$_{\pm\text{5.72}}$ & \textbf{61.18}$_{\pm\text{6.47}}$ & \textbf{92.34}$_{\pm\text{5.67}}$ \\
\bottomrule
\end{tabular}
}
\end{table*}

\vspace{2mm}
\subsection{Results on Different Number of Nearest Neighbours $K$}
\label{appendix:ablation_k}
To evaluate the robustness of the prototype-based predictor, we further conduct experiments using different numbers of nearest neighbours denoted as $K$. Considering the limited sample sizes of the datasets under investigation, we set the maximum number of nearest samples to $K=5$. All experimental settings are kept consistent to ensure 
 fair comparisons. 

\cref{tab:ablation_k} presents the results of the ablation experiments on the number of nearest neighbours, demonstrating that the optimal value of $K$ varies across different datasets. It is observed that using a small value of $K$ can make the predictions more sensitive to noise and outliers, resulting in lower accuracy. Notably, ProtoGate consistently achieves high accuracy across the range of $K \in \{3, 4, 5\}$. This finding supports the validity of the prototype-based prediction for the considered real-world datasets and ProtoGate's robustness in prediction tasks.

\begin{table*}[!htbp]
\centering
\caption{\textbf{Classification accuracy (\%) for different numbers of the nearest neighbours.} We \textbf{bold} the highest accuracy for each dataset. A small $K \in \{1, 2\}$ can lead to sensitivity to noise, and the model performs stably with $K \in \{3, 4, 5\}$.}
\label{tab:ablation_k}
\begin{tabular}{lrrrrrrrrr}
\toprule
\# Prototypes      & colon & lung                      & meta-dr                   & meta-pam                  & prostate                  & tcga-2y                   & toxicity                                       \\ \midrule
$K=1$ & 70.40$_{\pm\text{14.45}}$ & 87.53$_{\pm\text{7.28}}$ & 50.50$_{\pm\text{6.21}}$ & 73.02$_{\pm\text{10.90}}$ & 75.91$_{\pm\text{10.21}}$ & 57.46$_{\pm\text{6.85}}$ & 75.85$_{\pm\text{7.02}}$ \\

$K=2$ & 77.35$_{\pm\text{13.46}}$ & 92.30$_{\pm\text{7.28}}$ & 56.06$_{\pm\text{7.29}}$ & 90.28$_{\pm\text{6.01}}$  & 86.93$_{\pm\text{7.33}}$  & 59.40$_{\pm\text{6.24}}$ & 88.81$_{\pm\text{7.01}}$ \\

$K=3$ & \textbf{83.95}$_{\pm\text{9.82}}$  & \textbf{93.56}$_{\pm\text{6.29}}$ & 57.82$_{\pm\text{8.93}}$ & \textbf{95.96}$_{\pm\text{3.93}}$  & 89.53$_{\pm\text{5.64}}$  & \textbf{61.18}$_{\pm\text{6.47}}$ & 91.14$_{\pm\text{5.19}}$ \\

$K=4$ & 75.25$_{\pm\text{13.34}}$ & 90.34$_{\pm\text{7.01}}$ & \textbf{60.43}$_{\pm\text{7.62}}$ & 95.03$_{\pm\text{4.77}}$  & 88.85$_{\pm\text{5.87}}$  & 60.97$_{\pm\text{5.60}}$ & 91.10$_{\pm\text{4.93}}$ \\

$K=5$ & 77.50$_{\pm\text{8.67}}$  & 91.12$_{\pm\text{6.36}}$ & 59.23$_{\pm\text{6.88}}$ & 95.83$_{\pm\text{5.89}}$  & \textbf{90.58}$_{\pm\text{5.72}}$  & 60.84$_{\pm\text{5.88}}$ & \textbf{92.34}$_{\pm\text{5.67}}$\\
\bottomrule
\end{tabular}
\end{table*}

\vspace{2mm}
\subsection{Results on Hybrid Sorting}
\label{appendix:ablation_sort}
\vspace{1mm}
To gauge the effect of NeuralSort, we run a proof-of-principle analysis comparing our proposed ProtoGate (w/ HybridSort) to the one where we substitute the sorting mechanisms with only classical QuickSort or differentiable NeuralSort. The results and conclusions are below:

\looseness-1
Firstly, \cref{tab:ablation_sort_time} shows that ProtoGate w/ only QuickSort (as expected) leads to faster training than HybridSort. Nevertheless, given the problem settings we are focusing on (high-dimensional and low-sample-size tabular data), these differences are practically insignificant. Secondly, \cref{tab:ablation_sort_acc} shows the predictive performance of using only QuickSort substantially degrades w.r.t our proposed ProtoGate (w/ HybridSort). This large performance gap indicates that a differentiable sorting operator is important for learning a well-performing feature selector since it allows the clustering assumption to be explicitly encoded into the optimisation procedure (via backpropagating gradients from the sorted samples/prototypes). Thirdly, we also find that using QuickSort for training can be less stable on the considered datasets, resulting in different and sometimes larger selected feature sets (up to 40\% of all features), compared to the ones obtained with HybridSort (consistently below 15\%). 

In a nutshell, HybridSort is an effective sorting operator for prototype-based classification, which reduces the inference time by almost half while preserving the identical predictive performance as only using the computationally expensive differentiable sorting operators.

\begin{table}[htbp]
    \centering
    \caption{\textbf{Training time (per epoch) and inference time (per sample) with different sorting operators (unit: millisecond).} ProtoGate w/ only QuickSort leads to faster training and inference.}
    \label{tab:ablation_sort_time}
    \begin{tabular}{lrrrrrrrr}
    \toprule
    Methods    & Stages                            & colon  & lung  & meta-dr & meta-pam & prostate & tcga-2y & toxicity \\ \midrule
    \multirow{2}{*}{QuickSort}  & training         & 10.66  & 36.78 & 56.72   & 56.09    & 61.76    & 88.91   & 100.06   \\
               & inference                         & 0.17   & 0.19  & 0.28    & 0.28     & 0.61     & 0.44    & 0.59     \\ \midrule
    \multirow{2}{*}{NeuralSort} & training         & 17.63  & 99.42 & 118.44  & 117.50   & 80.96    & 120.31  & 133.33   \\
               & inference                         & 0.28   & 0.50  & 0.59    & 0.59     & 0.79     & 0.60    & 0.78     \\ \midrule
    \multirow{2}{*}{\textbf{HybridSort (Ours)}} & training  & 17.63  & 99.42 & 118.44  & 117.50   & 80.96    & 120.31  & 133.33   \\
               & inference                         & 0.17   & 0.19  & 0.28    & 0.28     & 0.61     & 0.44    & 0.59     \\
    \bottomrule 
    \end{tabular}
\end{table}

\begin{table}[htbp]
\centering
\caption{\textbf{Classification accuracy (\%) of ProtoGate with different sorting operators.} We \textbf{bold} the highest accuracy for each dataset. HybridSort and NeuralSort outperform QuickSort in predictive accuracy by a clear margin.}
\label{tab:ablation_sort_acc}
\begin{tabular}{lrrrrrrr}
\toprule
Methods             & colon & lung                      & meta-dr                   & meta-pam                  & prostate                  & tcga-2y                   & toxicity                                        \\ \midrule
QuickSort  & 67.25$_{\pm\text{8.17}}$  & 59.46$_{\pm\text{12.51}}$ & 49.11$_{\pm\text{1.79}}$    & 65.78$_{\pm\text{14.63}}$    & 81.70$_{\pm\text{11.80}}$    & 55.00$_{\pm\text{2.89}}$    & 67.57$_{\pm\text{7.34}}$     \\
NeuralSort &\textbf{83.95}$_{\pm\text{9.82}}$  & \textbf{93.56}$_{\pm\text{6.29}}$  & \textbf{60.43}$_{\pm\text{7.62}}$    & \textbf{95.96}$_{\pm\text{3.93}}$     & \textbf{90.58}$_{\pm\text{5.72}}$     & \textbf{61.18}$_{\pm\text{6.47}}$    & \textbf{92.34}$_{\pm\text{5.67}}$ \\
\textbf{HybridSort (Ours)} & \textbf{83.95}$_{\pm\text{9.82}}$  & \textbf{93.56}$_{\pm\text{6.29}}$  & \textbf{60.43}$_{\pm\text{7.62}}$    & \textbf{95.96}$_{\pm\text{3.93}}$     & \textbf{90.58}$_{\pm\text{5.72}}$     & \textbf{61.18}$_{\pm\text{6.47}}$    & \textbf{92.34}$_{\pm\text{5.67}}$ \\ \bottomrule
\end{tabular}
\end{table}

\clearpage
\section{Additional Results on Interpretability Evaluation}
\label{appendix:inter_eval}

\subsection{Alternative Metrics for Interpretability Evaluation}
\label{appendix:metric_interpretability}

We carefully assess the existing metrics before performing the interpretability evaluation, and we find that some of them are debatable, even with potential flaws.

\looseness-1
\textbf{Faithfulness.}
The faithfulness in feature selection promotes that ``all selected features should be significant for predictions''.
Firstly, faithfulness~\cite{yang2022locally, alvarez2018towards} can be computationally impractical for high-dimensional datasets. The real faithfulness should be evaluated in four steps:
(i)~Remove a feature;
(ii)~Retrain the model on the new datasets without the dropped feature;
(iii)~Compute the correlation between the accuracy drop and the feature importance;
(iv)~Repeat the above three steps for all features.
On high-dimensional datasets, Step-(ii) requires retraining the model thousands of times.
Secondly, we understand that some prior work~\cite{yang2022locally} omits Step-(ii) to compute a surrogate of real faithfulness, but omitting Step-(ii) means Step-(iii) evaluates the model with a different distribution to which the model is trained on. Following prior work~\cite{jethani2021have}, our experiments are also based on the common assumption in machine learning: models are trained and evaluated on data with the same distribution (i.e., the training and test data are independent and identically distributed). Evaluating the surrogate of faithfulness can violate this key assumption~\cite{jethani2021have, hooker2019benchmark}.
Thirdly, the transferability of selected features directly demonstrates whether the selected features are overall informative for accurate predictions with downstream predictors~\cite{yang2022locally}. Therefore, we choose to provide an analysis of the transferability of selected features to avoid providing possibly misleading guidance for users.

\textbf{Stability.}
The stability in feature selection promotes ``samples similar in the input space should have similar informative features''~\cite{yang2022locally}.
Firstly, the poor accuracy of vanilla KNN shows that samples of the same class are not always close (i.e., similar) to each other in the original high-dimensional space. In other words, the similarity in the original high-dimensional space can be misleading. Therefore, leveraging the similarity in the original high-dimensional space can harm the model performance.
Secondly, in our experiments, we also find that LSPIN without such stability regularisation generally performs better than LSPIN with it. The highest accuracy of LSPIN (reported in \cref{tab:acc}) is consistently achieved with stability strength of 0 (i.e., no stability regularisation) across datasets. The results further show the negative effects of such stability regularisation.
Thirdly, we further analyse the source of such negative effects. The real clustering assumption promotes ``samples of the same class should have similar representations''~\cite{chapelle2006semi}. However, ``similar masks'' is not a sufficient condition for ``similar representations''. Technically, ``similar representations'' refer to the similarity between masked samples (i.e., $\textbf{x}^{(i)} \odot \textbf{s} _\text{local}^{(i)}$), rather than the masks (i.e., $\textbf{s} _\text{local}^{(i)}$). Therefore, similar samples in the input space are not expected to have similar masks. This is also the reason why ProtoGate can achieve high accuracy by attending to the similarity between representations, instead of masks.

\textbf{Diversity.}
The diversity in feature selection promotes ``samples of different classes should have different informative features''. Firstly, we show with vanilla KNN (\cref{tab:acc}) that the similarity (also dissimilarity) in the original high-dimensional space can be misleading for feature selection and prediction. Therefore, leveraging the similarity in the original high-dimensional space can harm the model performance.
Secondly, the real clustering assumption promotes ``samples of the same class should have similar representations''~\cite{chapelle2006semi}. However, ``similar masks'' is not a sufficient condition for ``similar representations'', and ``dissimilar masks'' is also not a sufficient condition for ``dissimilar representations''. In other words, similar samples are not necessarily expected to have similar masks. Likewise, dissimilar samples are not necessarily expected to have dissimilar masks. Therefore, we chose not to evaluate stability in the considered tasks to avoid possibly misleading guidance for users.

\subsection{Results on Fidelity of Selected Features}
\label{appendix:co_adapt}
In \cref{tab:acc_syn}, ProtoGate achieves better or comparable performance in feature selection and classification than the benchmark methods on Syn1$_{(+)}$ and Syn2$_{(+)}$. On Syn3$_{(-)}$, ProtoGate performs poorly as expected (see \cref{appendix:data_syn} for the reason), verifying the inductive bias of the clustering assumption. We also find the LSPIN exhibits visible misalignment in feature selection and prediction. On Syn1$_{(+)}$, LSPIN achieves the best classification accuracy, but the quality of selected features is much worse, with a rank of six out of ten methods. In other words, LSPIN simply overfits the dataset without correctly identifying the informative features, denoting a severe co-adaptation problem and low-fidelity feature selection. In contrast, ProtoGate has consistently non-positive rank differences between F1$_{\text{selec}}$ and ACC$_{\text{pred}}$, showing high-fidelity feature selection. The results demonstrate that ProtoGate can achieve a well-aligned performance of feature selection and classification, guaranteeing the fidelity of selected features.

\begin{table*}[!htbp]
    \centering
    \caption{\textbf{Evaluation comparison of ProtoGate and nine benchmark methods on three synthetic datasets.} We report the F1 score of selected features (F1$_\text{select}$) and the balanced accuracy for prediction (ACC$_{\text{pred}}$). ``Diff.'' refers to the difference between the ranks of F1$_\text{select}$ and ACC$_{\text{pred}}$, and a positive value indicates a high possibility of co-adaptation. We highlight the {\color[HTML]{008080} \textbf{First}}, {\color[HTML]{7030A0} \textbf{Second}} and {\color[HTML]{C65911} \textbf{Third}} performance for each dataset. ProtoGate achieves well-aligned performance for feature selection and prediction.}
    \label{tab:acc_syn}
    \resizebox{\textwidth}{!}{
    \begin{tabular}{lrrrrrr|rrr}
    \toprule
    \multirow{2}{*}{Methods} & \multicolumn{3}{c}{Syn1$_{(+)}$}           & \multicolumn{3}{c|}{Syn2$_{(+)}$}            & \multicolumn{3}{c}{Syn3$_{(-)}$}           \\ \cmidrule{2-10}
                             & F1$_\text{select} \uparrow$              & ACC$_{\text{pred}} \uparrow$            & \textbf{Diff.}               & F1$_\text{select} \uparrow$              & ACC$_{\text{pred}} \uparrow$            & \textbf{Diff.}               & F1$_\text{select} \uparrow$              & ACC$_{\text{pred}} \uparrow$            & \textbf{Diff.}               \\ \midrule
    
    Lasso                     & 0.09$_{\pm\text{0.02}}$                                 & 54.55$_{\pm\text{6.14}}$                                 & 2              & 0.11$_{\pm\text{0.01}}$                                 & 52.42$_{\pm\text{6.69}}$                                 & 0              & 0.09$_{\pm\text{0.02}}$                                 & 55.30$_{\pm\text{7.44}}$                                 & 2              \\
    
    RF                        & 0.15$_{\pm\text{0.04}}$                                 & 57.08$_{\pm\text{6.48}}$                                 & 3              & 0.19$_{\pm\text{0.02}}$                                 & {\color[HTML]{C65911} \textbf{59.44}}$_{\pm\text{5.24}}$ & 1              & {\color[HTML]{C65911} \textbf{0.22}}$_{\pm\text{0.02}}$ & 56.33$_{\pm\text{9.08}}$                                 & -1             \\

    STG                       & {\color[HTML]{7030A0} \textbf{0.27}}$_{\pm\text{0.04}}$ & {\color[HTML]{C65911} \textbf{58.65}}$_{\pm\text{9.03}}$ & -1             & {\color[HTML]{C65911} \textbf{0.22}}$_{\pm\text{0.09}}$ & 58.28$_{\pm\text{8.36}}$                                 & -2             & {\color[HTML]{008080} \textbf{0.28}}$_{\pm\text{0.18}}$ & 54.00$_{\pm\text{9.09}}$                                 & -7             \\
    
    TabNet                    & 0.08$_{\pm\text{0.02}}$                                 & 48.59$_{\pm\text{6.55}}$                                 & 1              & 0.06$_{\pm\text{0.02}}$                                 & 49.57$_{\pm\text{5.38}}$                                 & 0              & 0.06$_{\pm\text{0.02}}$                                 & 48.45$_{\pm\text{8.31}}$                                 & 0              \\
    
    L2X                       & 0.16$_{\pm\text{0.07}}$                                 & 52.89$_{\pm\text{7.51}}$                                 & -3             & 0.19$_{\pm\text{0.10}}$                                 & 55.78$_{\pm\text{6.97}}$                                 & -1             & 0.10$_{\pm\text{0.09}}$                                 & 55.92$_{\pm\text{7.30}}$                                 & 2              \\
    
    INVASE                    & 0.18$_{\pm\text{0.05}}$                                 & 55.36$_{\pm\text{9.00}}$                                 & -1             & 0.16$_{\pm\text{0.03}}$                                 & {\color[HTML]{7030A0} \textbf{60.28}}$_{\pm\text{8.61}}$ & 6              & 0.13$_{\pm\text{0.03}}$                                 & {\color[HTML]{008080} \textbf{58.75}}$_{\pm\text{8.70}}$ & 5              \\
    
    REAL-X                    & {\color[HTML]{C65911} \textbf{0.19}}$_{\pm\text{0.04}}$ & 47.54$_{\pm\text{9.51}}$                                 & -7             & {\color[HTML]{7030A0} \textbf{0.23}}$_{\pm\text{0.07}}$ & 55.20$_{\pm\text{6.38}}$                                 & -6             & {\color[HTML]{7030A0} \textbf{0.26}}$_{\pm\text{0.06}}$ & {\color[HTML]{C65911} \textbf{56.48}}$_{\pm\text{9.34}}$ & -1             \\

    LSPIN                     & 0.15$_{\pm\text{0.04}}$                                 & {\color[HTML]{008080} \textbf{59.04}}$_{\pm\text{9.24}}$ & 5              & 0.19$_{\pm\text{0.04}}$                                 & 59.40$_{\pm\text{8.07}}$                                 & 1              & 0.19$_{\pm\text{0.06}}$                                 & {\color[HTML]{7030A0} \textbf{58.09}}$_{\pm\text{6.41}}$ & 2              \\
    
    LLSPIN                    & 0.11$_{\pm\text{0.02}}$                                 & 54.96$_{\pm\text{9.49}}$                                 & 2              & 0.17$_{\pm\text{0.08}}$                                 & 56.18$_{\pm\text{5.80}}$                                 & 1              & 0.10$_{\pm\text{0.06}}$                                 & 52.35$_{\pm\text{8.32}}$                                 & -2             \\
    
    \midrule
    \textbf{ProtoGate (Ours)}                 & {\color[HTML]{008080} \textbf{0.29}}$_{\pm\text{0.07}}$ & {\color[HTML]{7030A0} \textbf{58.68}}$_{\pm\text{6.28}}$ & -1             & {\color[HTML]{008080} \textbf{0.29}}$_{\pm\text{0.09}}$ & {\color[HTML]{008080} \textbf{60.67}}$_{\pm\text{8.21}}$ & 0              & 0.17$_{\pm\text{0.06}}$                                 & 56.16$_{\pm\text{6.82}}$                                 & 0             
    \\
    \bottomrule
    \end{tabular}
    }
\end{table*}

\subsection{Results on Transferability of Selected Features}
\label{appendix:transferability}
\cref{fig:interpretability_transferability} shows that 
ProtoGate consistently improves the performance of KNN across datasets, while ProtoGate-global and ProtoGate-local can cause performance degradation on some datasets. Although SVM does not have the same inductive bias as ProtoGate (i.e., the clustering assumption), the selected features by ProtoGate do \textit{not} cause performance degradation in SVM, even bringing notable improvements on some datasets (e.g., 25.8\% increase in balanced accuracy on the ``meta-pam'' dataset). In contrast, ProtoGate-global and ProtoGate-local can cause considerable performance degradation (e.g., over 10\% drop in balanced accuracy on the ``toxicity'' dataset), further showing the complementary effects of soft global selection and local selection. In a nutshell, both global and local feature selection contribute to the selection behaviours of ProtoGate, and the features selected by ProtoGate are generally transferable and beneficial for the performance of simple models.

\vspace{-5mm}
\begin{table}[htbp]
\centering
\caption{\textbf{Normalised balanced accuracy (\%) of simple models with different feature selectors.} We \textbf{bold} the highest accuracy for each simple model. ProtoGate selects features that generally improve the performance of KNN and SVM.}
\label{tab:Table1}
\begin{tabular}{lrrrr}
\toprule
Methods & None & ProtoGate-global & ProtoGate-local & \textbf{ProtoGate (Global-to-local)} \\
\midrule
KNN &  35.08$_{\pm\text{20.07}}$ &  38.24$_{\pm\text{24.54}}$ &   64.02$_{\pm\text{9.47}}$ &  \textbf{79.45}$_{\pm\text{11.22}}$ \\
SVM &  44.22$_{\pm\text{15.96}}$ &  31.24$_{\pm\text{21.90}}$ &  43.45$_{\pm\text{29.58}}$ &  \textbf{71.38}$_{\pm\text{23.95}}$ \\
\bottomrule
\end{tabular}
\end{table}

\vspace{-5mm}
\begin{figure}[!ht]
    \centering
    \subfloat{\includegraphics[width=0.5\columnwidth]{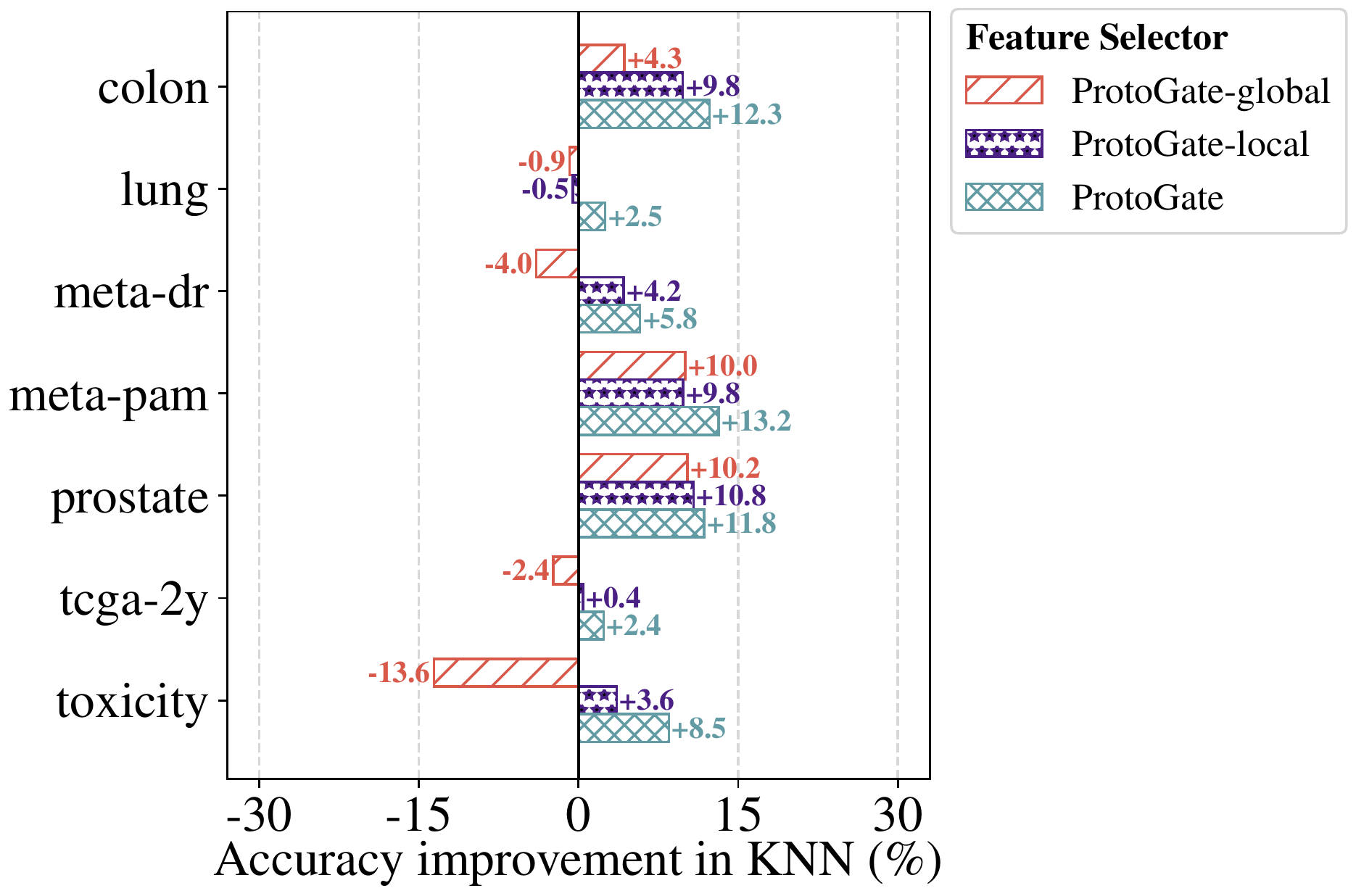}}
    \subfloat{\includegraphics[width=0.5\columnwidth]{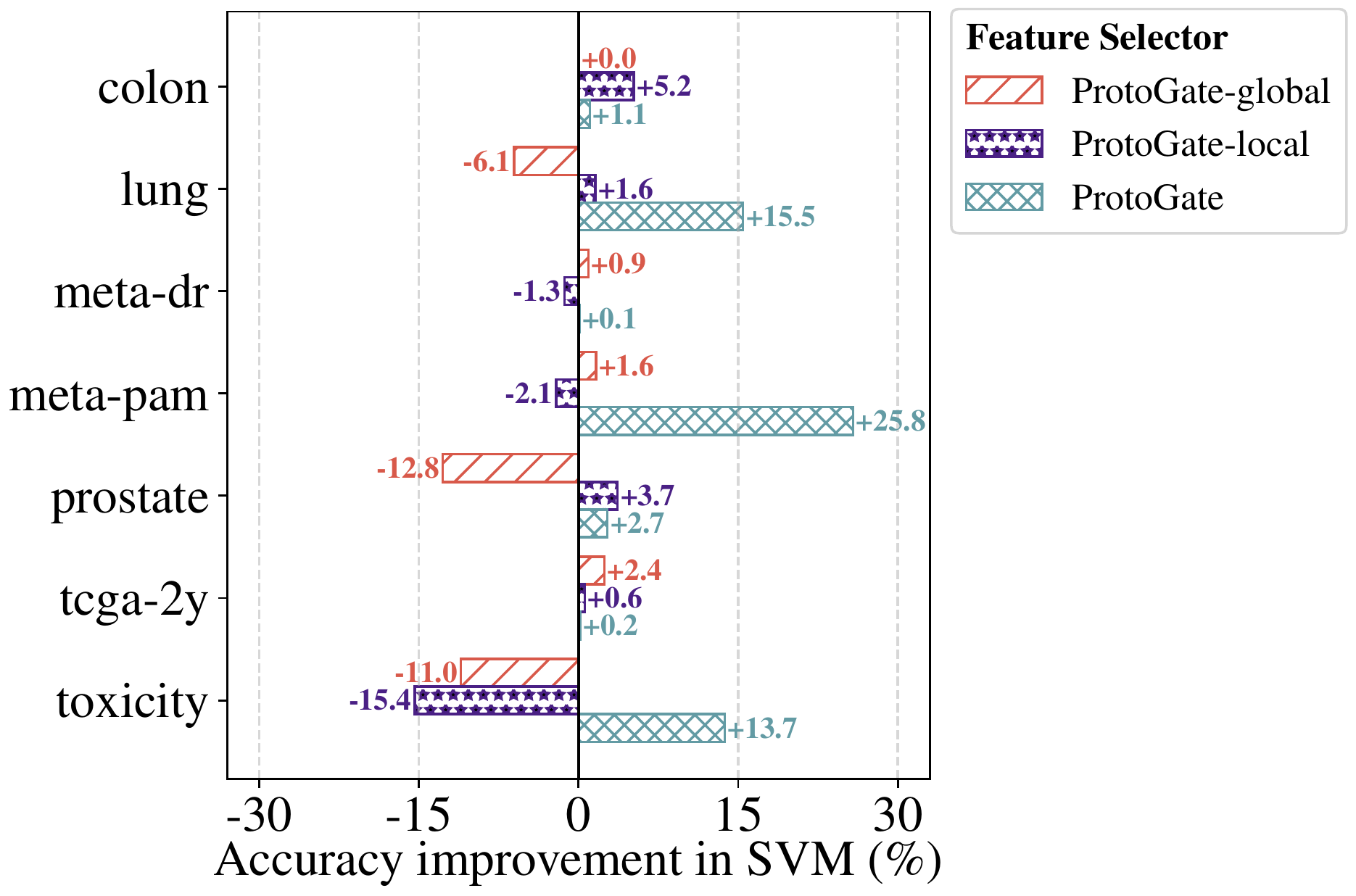}}
    \caption{\textbf{Left:} Accuracy improvement in KNN (\%) with different feature selectors. \textbf{Right:} Accuracy improvement in SVM (\%) with different feature selectors. ProtoGate selects features that generally improve the performance of KNN and SVM, while ProtoGate-global and ProtoGate-local can cause performance degradation on some datasets.}
    \label{fig:interpretability_transferability}
\end{figure}

\clearpage
\section{Future Work}
\looseness-1
We would like to emphasise that no discussed methods, including ProtoGate, dominate the considered datasets. Indeed, ProtoGate could be outperformed by the other top three methods, Lasso and MLP, on some HDLSS datasets, e.g., ``lung'' dataset. This is likely due to the inherently limited expressiveness of prototype-based models in comparison to well-regularized connectionist models, as discussed in \cite{lu2017expressive, margeloiu2022weight}. However, the robustness and prototypical explainability offered by ProtoGate contribute valuable perspectives for advancing reliable HDLSS feature selection methods. A promising direction for future research is the development of more expressive non-parametric models, which can potentially enhance feature selection models with the same learning paradigm as ProtoGate (i.e., disjoint in-model selection). In addition, the clustering characteristics of the samples after feature selection, in conjunction with domain knowledge, could reveal important scientific insights. Moreover, exploring the application of ProtoGate in other high-dimensional fields where explainability is crucial, such as finance \cite{song2019autoint}, and environmental modelling \cite{jarquin2014reaction}, could yield important findings.

\let\clearpage\relax


\end{document}